\documentclass[11pt]{article}
\usepackage{enumerate}
\usepackage{url}
\usepackage{pdflscape} 
\usepackage{rotating} 
\usepackage{booktabs}
\usepackage{threeparttable} 

\usepackage{geometry}
\RequirePackage{amsthm,amsmath,amsfonts,amssymb}
\usepackage{xcolor}
\usepackage{xr-hyper}
\usepackage[pdftex,bookmarks,colorlinks,breaklinks]{hyperref} 

\definecolor{dullmagenta}{rgb}{0.4,0,0.4} 
\definecolor{darkblue}{rgb}{0,0,0.4}
\definecolor{coquelicot}{rgb}{0.20, 0.12, 0.72}
\definecolor{navyblue}{rgb}{0,0,0.5}
\hypersetup{linkcolor=blue,citecolor=blue,filecolor=blue,urlcolor=blue}

\newcommand\independent{\protect\mathpalette{\protect\independenT}{\perp}}
\def\independenT#1#2{\mathrel{\rlap{$#1#2$}\mkern2mu{#1#2}}}

\usepackage{epstopdf}
\usepackage{tablefootnote}

\usepackage{multirow}
\usepackage{mathtools}
\usepackage{bbm}
\usepackage{graphicx} 
\usepackage{flafter} 
\usepackage{subcaption}
\usepackage{graphicx}

\usepackage{tikz}
\usetikzlibrary{positioning}
\usetikzlibrary{arrows.meta,arrows}

\usetikzlibrary{arrows, decorations.markings,shapes,arrows,fit}
\tikzset{box/.style={draw, minimum size=2em, text width=4.5em, text centered},
	bigbox/.style={draw, inner sep=20pt,label={[shift={(-3ex,3ex)}]south east:#1}}
}

\usepackage{bm} 

\usepackage{colortbl}
 \usepackage{arydshln}
\usepackage{booktabs}

\usepackage{algorithm}
\usepackage{algpseudocode}

\definecolor{coquelicot}{rgb}{0.90, 0.42, 0.72}
\definecolor{burntorange}{rgb}{0.8, 0.33, 0.0}

\newcommand{\jelenax}[1]{\textcolor{black}{#1}}

\definecolor{burntblue}{RGB}{0, 114, 206}

\numberwithin{table}{section}
\numberwithin{equation}{section}

\def\bx{\mathbf{x}}
\def\bX{\mathbf{X}}

\def\bU{\mathbf{U}}

\def\bV{\mathbf{V}}

\def\bP{\mathbf{P}}

\def\bC{\mathbf{C}}
\def\bc{\mathbf{c}}
\def\bd{\mathbf{d}}

\def\ba{\mathbf{a}}
\def\bA{\mathbf{A}}

\def\bQ{\mathbf{Q}}
\def\b0{\mathbf{0}}

\def\E{\mathbb{E}}

\def\P{\mathbb{P}}

\def\bbeta{\boldsymbol{\beta}}

\def\bxi{\boldsymbol{\xi}}

\def\bbetahat{\widehat{\bbeta}}

\def\bzero{\boldsymbol{0}}

\def\mhat{\widehat{m}}

\def\independenT#1#2{\mathrel{\rlap{$#1#2$}\mkern4mu{#1#2}}}

\def\Csc{\mathcal{C}}
\def\Isc{\mathcal{I}}
\def\Jsc{\mathcal{J}}

\def\Var{\mbox{Var}}

\def\bzero{\mathbf{0}}
\def\bb{\mathbf{b}}

\def\bT{\mathbf{T}}

\def\bS{\mathbf{S}}

\def\bg{\mathbf{g}}
\def\bG{\mathbf{G}}

\def\S{\mathbb{S}}

\def\bDelta{\boldsymbol{\Delta}}

\def\Vtil{\widetilde{\mathbb{V}}}

\def\bW{\mathbf{W}}

\def\bgamma{\boldsymbol{\gamma}}

\def\Ybar{\overline{Y}}

\def\Asc{\mathcal{A}}
\def\Bsc{\mathcal{B}}
\def\Csc{\mathcal{C}}

\def\R{\mathbb{R}}
\def\Jsc{\mathcal{J}}

\def\S{\mathbb{S}}

\def\cbar{\bar{c}}

\def\bxi{\boldsymbol{\xi}}

\def\bD{\mathbf{D}}

\def \hs2{\hspace{2mm}}

\numberwithin{table}{section}
\numberwithin{equation}{section}

\definecolor{jcolor}{RGB}{041,122,000}
\definecolor{darkred}{RGB}{100,000,000}
\definecolor{purple}{RGB}{200,000,200}

\def\boxit#1{\vbox{\hrule\hbox{\vrule\kern6pt  \vbox{\kern6pt#1\kern6pt}\kern6pt\vrule}\hrule}}

\def\muhat{\widehat{\mu}}

\def\muhat{\widehat{\mu}}

\def\muhat{\widehat{\mu}}

\def\be{\mathbf{e}}

\def\bT{\mathbf{T}}
\def\bS{\mathbf{S}}
\def\bG{\mathbf{G}}

\def\Rbar{\bar{R}}
\def\Rtil{\widetilde{R}}

\def\dbar{\bar{d}}
\def\Vbar{\bar{V}}

\def\bVtil{\widetilde{\bV}}

\def\pihat{\widehat{\pi}}
\def\mhat{\widehat{m}}

\def\Vtil{\widetilde{V}}

\def\bT{\mathbf{T}}

\def\bxi{\boldsymbol{\xi}}

\def\sigmahat{\widehat{\sigma}}

\def\thetahat{\widehat{\theta}}

\def\Vtil{\widetilde{V}}

\DeclareMathOperator*{\argminn}{argmin}

\theoremstyle{plain}
\newtheorem{theorem}{Theorem}[section]

\newtheorem{lemma}[theorem]{Lemma}

\theoremstyle{remark}
\newtheorem{assumption}{Assumption}
\newtheorem{remark}{Remark}

\usepackage{xr}

\begin{document}

\date{}

 
\title{\bf Adaptive Split Balancing for Optimal Random Forest}

\author{Yuqian Zhang\thanks{Institute of Statistics and Big Data, Renmin University of China} \and Weijie Ji\thanks{School of Statistics and Management, Shanghai University of Finance and Economics} \and Jelena Bradic\thanks{Department of Mathematics and Halicioglu Data Science Institute, University of California, San Diego, E-mail: \href{mailto:jbradic@ucsd.edu}{jbradic@ucsd.edu} }}




\maketitle

\begin{abstract}
\jelenax{In this paper, we propose a new random forest algorithm that constructs the trees using a novel adaptive split-balancing method.} Rather than relying on the widely-used random feature selection, we \jelenax{propose} a \jelenax{permutation-based} balanced \jelenax{splitting criterion}. The \emph{adaptive split balancing forest} (ASBF), achiev\jelenax{es} minimax optimality under the Lipschitz class. Its localized version, \jelenax{which fits local regressions at the leaf level}, attains the minimax rate under the \jelenax{broad} H\"older class $\mathcal{H}^{q,\beta}$ \jelenax{of problems} for any $q\in\mathbb{N}$ and $\beta\in(0,1]$. 
\jelenax{We identify} that over-reliance on auxiliary randomness \jelenax{in tree construction} may compromise the approximation power of trees, leading to suboptimal results. Conversely, \jelenax{the proposed} less random, \jelenax{permutation-based} approach demonstrates optimality \jelenax{over a wide range of models}. \jelenax{Although random forests are known to perform well empirically, their theoretical convergence rates are slow. Simplified versions that construct trees without data dependence offer faster rates but lack adaptability during tree growth.} Our proposed method achieves optimality in simple, smooth scenarios while adaptively learning the tree structure from the data. \jelenax{Additionally, we establish uniform upper bounds and demonstrate that ASBF improves dimensionality dependence in average treatment effect estimation problems.} Simulation studies and real-world applications demonstrate \jelenax{our methods' superior performance} over existing random forests.\end{abstract}

\section{Introduction}\label{sec:intro}

\jelenax{Renowned for their empirical success, random forests are the preferred method in numerous applied scientific domains. Their versatility has greatly contributed to their popularity, extending well beyond conditional mean problems to include} quantile estimation \cite{meinshausen2006quantile}, survival analysis \cite{ishwaran2008random, ishwaran2010consistency}, and feature selection \cite{goldstein2011random, mentch2014ensemble, louppe2013understanding, li2019debiased, behr2022provable}. Despite its widespread use, the theoretical analysis remains \jelenax{largely} incomplete, even \jelenax{in terms of minimax optimality}.

Consider estimat\jelenax{ing} the conditional mean function $m(\bx) := \E[Y \mid \bX = \bx]$ for any $\bx \in [0,1]^d$\jelenax{, where $Y \in \R$ is the response and $\bX \in [0,1]^d$ is the covariate vector.} Let $\S_N := (Y_i, \bX_i)_{i=1}^N$ be independent and identically distributed (i.i.d.) samples\jelenax{, with $(Y_i, \bX_i)_{i=1}^N \overset{d}{\sim} (Y, \bX)$. Let }$\mhat(\cdot)$ \jelenax{be} the random forest \jelenax{estimate} constructed \jelenax{from} $\S_N$. \jelenax{This paper} focus\jelenax{es} on the integrated mean squared error (IMSE) $\E_{\bx} [\mhat(\bx) - m(\bx)]^2$, where the expectation is taken with respect to the new observation $\bx$.

Breiman's original algorithm \cite{breiman2001random} \jelenax{is} based on \jelenax{Classification And Regression Trees} (CART) \cite{breiman1984classification}. \jelenax{Its essential components,} introduce auxiliary randomness \jelenax{through bagging and the \emph{random feature selection} and play a crucial role in the algorithm's effectiveness.} \jelenax{At each node of every} tree\jelenax{, the algorithm} optimize\jelenax{s Ginny index (for classification) or prediction squared error (for regression) } by selecting \jelenax{axes-aligned} splitting directions and locations \jelenax{from a subset of randomly chosen set of features, using} a subset of samples.
 \jelenax{T}he overall forest is obtained by averaging over an ensemble of trees. \cite{scornet2015consistency} \jelenax{established} consistency of Breiman's original algorithm under additive models \jelenax{ with continuous components, but} did not specify a rate. Recently, \cite{chi2022asymptotic} established the rate under a ``sufficient impurity decrease'' (SID) condition \jelenax{for} high dimensional \jelenax{setting}, and \cite{klusowski2023large} extended this to \jelenax{cases} where covariates' dimension grows sup-exponentially with the sample size.
 These results suggest that Breiman's original algorithm retains consistency \jelenax{with} discontinuous conditional mean functions and high-dimensional covariates. However, the established consistency rates for smooth functions were observed to be slow. For instance, for simple linear functions, the rates are no faster than \(N^{-1/(16d^2)}\) and \(1/\log(N)\), respectively; see Table \ref{table:rate}.
\jelenax{Given} the theoretical \jelenax{complexities} associated with \jelenax{CART-split} criterion, \cite{biau2012analysis}\footnote{The results presented by \cite{biau2012analysis} necessitate prior knowledge of the active features, information that is typically unknown in practice. As an alternative, the author also suggested the honest technique albeit without providing any theoretical guarantees.}, \cite{arlot2014analysis}, and \cite{klusowski2021sharp} investigated a simplified version named the ``centered forest,'' \jelenax{where the} splitting directions are chosen randomly, and splitting points are selected as midpoints of parent nodes. 
 A \jelenax{more advanced} variant, the ``median forest,'' \cite{klusowski2021sharp, duroux2018impact}, \jelenax{selects} sample medians as splitting points.
\jelenax{Importantly}, both centered and median forests\jelenax{' consistency rates} are slow, with minimax rates attained only when $d=1$; see Table \ref{table:rate} and Figure \ref{fig:rate}.

\renewcommand{\arraystretch}{1.5}
\begin{table}[h!]
\caption{IMSE rates. 
Splitting criteria that use only $\bX_i$ are labeled unsupervised, while those using both $\bX_i$ and $Y_i$ are labeled supervised.
Note: \cite{cattaneo2023inference} refers to point-wise MSE at interior points, \cite{cai2023extrapolated} to in-sample excess risk, and \cite{athey2019generalized,friedberg2020local} to normality at a specific point.}\label{table:rate}
\scalebox{0.78}{
\begin{tabular}{|c|c|c|c|c|}
\hline
Methods & Consistency rate 
& Functional class & Random Forest & Splitting criterion\\
\hline
\cite{genuer2012variance}&$N^{-2/3}$&$\mathcal{H}^{0,1}$, $d=1$& Purely uniform &Data-independent\\
\hline
\cite{biau2012analysis}&$N^{-\frac{3/4}{q\log (2)+3/4}}$&$q$-sparse
 $\mathcal{H}^{0,1}$&Centered &Data-independent\\
\hline
\multirow{2}{*}{\cite{arlot2014analysis}}&$N^{\frac{-2\log(1-1/(2d))}{2\log(1-1/(2d))-\log (2)}}$ &$\mathcal{H}^{1,1}$, $d\leq3$&\multirow{2}{*}{Centered }&\multirow{2}{*}{Data-independent}\\
\cline{2-3}
&$N^{\delta+2\log(\frac{2d-1}{2d})}$, $\delta>0$&$\mathcal{H}^{1,1}$, $d\geq4$& &\\
\hline
\cite{mourtada2020minimax}&$N^{-\frac{2(q+\beta)}{d+2(q+\beta)}}$&$\mathcal{H}^{q,\beta}$, $q+\beta\leq1.5$ &Mondrian&Data-independent\\
\cline{2-3}
\hline
\cite{o2024minimax}&$N^{-\frac{2(q+\beta)}{d+2(q+\beta)}}$&$\mathcal{H}^{q,\beta}$, $q\in\{0,1\}$, $\beta\in(0,1]$&Tessellation &Data-independent\\
\hline
\cite{cattaneo2023inference}&$N^{-\frac{2(q+\beta)}{d+2(q+\beta)}}$&$\mathcal{H}^{q,\beta}$, $q\in \mathbb{N}$, $\beta\in(0,1]$&Debiased Mondrian &Data-independent\\
\hline
\cite{cai2023extrapolated}&$(N/\log(N))^{-\frac{2(q+\beta)}{d+2(q+\beta)}}$&$\mathcal{H}^{q,\beta}$, $q\in \mathbb{N}$, $\beta\in(0,1]$&Extrapolated tree&Data-independent\\
\hline
\multirow{3}{*}{\cite{klusowski2021sharp}}&$(N\log^{(d-1)/2}(N))^{-r}$,&\multirow{2}{*}{$\mathcal{H}^{0,1}$}&\multirow{2}{*}{Centered }&\multirow{2}{*}{Data-independent}\\
&$r=\frac{2\log(1-1/(2d))}{2\log(1-1/(2d))-\log (2)}$&&&\\
\cline{2-5}
&$N^{-\frac{2\log(1-1/(2d))}{2\log(1-1/(2d))-\log (2)}}$&$\mathcal{H}^{0,1}$&Median &Unsupervised\\
\hline
\cite{duroux2018impact}&$N^{-\frac{\log(1-3/(4d))}{\log(1-3/(4d))-\log (2)}}$&$\mathcal{H}^{0,1}$&Median &Unsupervised\\
\hline
\cite{scornet2015consistency}&Only $o_p(1)$&Additive model &Breiman&Supervised\\
\hline
\cite{klusowski2023large}&$O_p(1/\log(N))$&Additive model &Breiman&Supervised\\
\hline
\cite{chi2022asymptotic}&$N^{-\frac{c}{\alpha}\land\eta}$, $c<1/4$, $\eta<1/8$ &SID($\alpha)$, $\alpha\geq1$ 
&Breiman&Supervised
\\
\hline
\multirow{2}{*}{\cite{athey2019generalized}}&$N^{\delta-\frac{\log(1-\alpha)}{d\log(\alpha)+\log(1-\alpha)}}$,&\multirow{2}{*}{$\mathcal{H}^{0,1}$}&\multirow{2}{*}{Honest }&Unsupervised if $\alpha=0.5$,\\
&$\alpha\leq1$, $\delta>0$&&&supervised if $\alpha<0.5$\\
\hline
\multirow{2}{*}{\cite{friedberg2020local}}&$N^{\delta-\frac{1.3\log(1-\alpha)}{d\log(\alpha)+1.3(1-\alpha)}}$,&\multirow{2}{*}{$\mathcal{H}^{1,1}$}&\multirow{2}{*}{Local linear honest }&\multirow{2}{*}{Supervised}\\
&$\alpha\leq0.2$, $\delta>0$&&&\\
\hline
\multirow{2}{*}{This paper}&$N^{\frac{-2\log(1-\alpha)}{d\log(\alpha)+2\log(1-\alpha)}}$&$\mathcal{H}^{0,1}$&ASBF&Unsupervised if $\alpha=0.5$,\\
\cline{2-4}
&$N^{\frac{-2(q+\beta)\log(1-\alpha)}{d\log(\alpha)+2(q+\beta)\log(1-\alpha)}}$ &$\mathcal{H}^{q,\beta}$, $q\in \mathbb{N}$, $\beta\in(0,1]$&L-ASBF&supervised if $\alpha<0.5$\\
\hline
\end{tabular}}
\end{table}

Centered forests \jelenax{and other} ``purely random forests'' \cite{mourtada2020minimax, o2024minimax, biau2008consistency, arlot2014analysis, klusowski2021sharp}, grow \jelenax{trees} independently from all the samples. \jelenax{Few studies among these} have achieved minimax rates \jelenax{for} smooth functions. 
 \cite{gao2022towards}  achieved near-minimax rates exclusively for Lipschitz functions by using an ``early stopping'' technique, though their splitting criterion remains data-independent. Mondrian forests \cite{mourtada2020minimax} attain \jelenax{IMSE} minimax optimal rate for the H\"older class $\mathcal{H}^{q,\beta}$ when $s=q+\beta\leq1.5$. 
 \cite{o2024minimax} introduced Tessellation forests, \jelenax{extending} minimax optimality to $s \leq 2$. \cite{cattaneo2023inference} proposed debiased Mondrian forests, establishing minimax optimal rates \jelenax{any $q \in \mathbb{N}$ and $\beta\in(0,1]$} in the point-wise mean squared error (MSE) $\E \left[\mhat(\bx) - m(\bx)\right]^2$ for fixed interior points. However, this debiasing procedure only corrects for interior bias and not boundary bias\jelenax{, preventing it from achieving minimax optimality in terms of the IMSE.} \cite{cai2023extrapolated} allows arbitrary $q \in \mathbb{N}$ and established nearly optimal in-sample excess risk; however, they did not provide upper bounds for the out-of-sample IMSE. All \jelenax{aforementioned} works use data-independent \jelenax{splits}, \jelenax{restricting the use of data during tree growth}. As a result, purely random forests lose the unique advantages of tree-based methods.

\jelenax{Recently}, \cite{athey2019generalized, wager2015adaptive, wager2018estimation, friedberg2020local} explored ``honest forest'' variant that \jelenax{differs from} Breiman's forest in two key aspects: (a) the splitting point \jelenax{ensures} that child nodes contain at least $\alpha \leq 0.5$ fraction of parent' samples, and (b) the forest is ``honest'' \jelenax{using two independent sub-samples per tree}.
One sub-sample\jelenax{'s outcomes and} all the covariates, are used for splitting, while \jelenax{the other sub-sample's} outcomes
 are used for local averaging. \jelenax{Now,} the splits depend on both the covariates and outcomes as long as $\alpha < 0.5$. For $\alpha = 0.5$, their method degenerates to the median forest. \jelenax{Without the $\alpha$-fraction} constraint, splits tend to concentrate \jelenax{at the parent node} endpoints, resulting in inaccurate local averaging; see \cite{ishwaran2015effect, breiman1984classification, cattaneo2022pointwise}. However, the consistency rates for honest forests remain sub-optimal; see Table \ref{table:rate} and Figure \ref{fig:rate}.

\begin{figure}[h!]
\captionsetup[subfloat]{labelformat=empty}
\subfloat[(a) $d=2$]{\includegraphics[height=0.3\linewidth,width=0.45\linewidth]{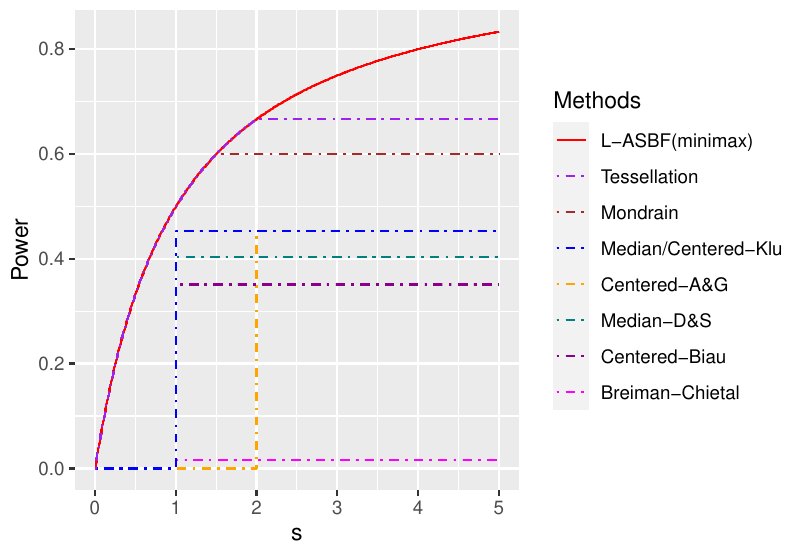}}
\hfill
\subfloat[(b) $d=4$]{
\includegraphics[height=0.3\linewidth,width=0.45\linewidth]{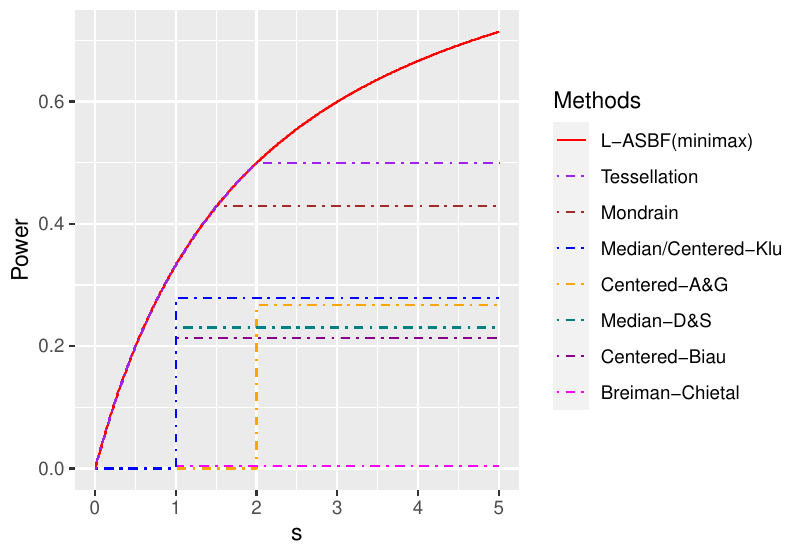}}
\caption{IMSE for the H\"older class $\mathcal{H}^{q,\beta}$ with $s=q+\beta$ and $d\in\{2,4\}$. Here, $y$ denotes IMSE as $O_p(N^{-y})$ excluding log terms; see Table \ref{table:rate}. Klu, A$\&$G, D$\&$S, Biau, and Chietal refer to \cite{klusowski2021sharp}, \cite{arlot2014analysis}, \cite{duroux2018impact}, \cite{biau2012analysis}, and \cite{chi2022asymptotic}, which provided rates for certain integer $s$. L-ASBF is the proposed local ASBF with $\alpha=0.5$, achieving the minimax optimal rate for any $s>0$.}\label{fig:rate}
\end{figure}

 \jelenax{We aim to improve upon the existing methods, which either fail to achieve minimax rates or do not fully utilize the data's information. To address these limitations, we propose a new splitting criterion that eliminates the auxiliary randomness of the \jelenax{random feature selection}. Instead, our approach employs a permutation-based splitting criterion that balances the depth of the tree branches, resulting in minimax optimal rates in simple, smooth situations.} Detailed construction is deferred to Algorithm \ref{alg:balance_RF}. 
A special case of the proposed method is \emph{balanced} median forest \jelenax{which} attains the minimax optimal rate for Lipschitz continuous functions. 
We propose \jelenax{a local ASBF (L-ASBF)} \jelenax{for the H\"older class $\mathcal{H}^{q,\beta}$}, 
 that performs $q$-th order local polynomial regression within the terminal leaves \jelenax{for} any $q \geq 1$. \jelenax{A specific case is the local linear approach of \cite{friedberg2020local}, where now the addition of balanced splitting ensures a faster convergence rate; see} Table \ref{table:rate}. 
 \jelenax{We show that}
 the L-ASBF reaches minimax optimality over the \jelenax{broad} H\"older class $\mathcal H^{q,\beta}$ for any $q\in\mathbb N$ and $\beta\in(0,1]$. To the best of our knowledge, this marks the first random forest that attain\jelenax{s} minimax optimal rates in terms of the IMSE \jelenax{for} the H\"older smoothness \jelenax{class} with any $q\geq2$ and $\beta\in(0,1]$, as seen in Figure \ref{fig:rate}. 
 We also derive uniform rates and demonstrate minimax optimality for any $q\in\mathbb N$; see in Section \ref{sec:uniform}.

Random forests have \jelenax{been} extensive\jelenax{ly} appli\jelenax{ed to} causal inference, \jelenax{with much of the literature} focus\jelenax{ing} on estimating the conditional average treatment effect (CATE); see \cite{wager2018estimation,athey2019generalized}. Estimating the average treatment effect (ATE), however, poses unique challenges; see Remark \ref{remark:ATE}. We \jelenax{use} ASBF with augmented inverse propensity weighting (AIPW) to estimate ATE. 
\jelenax{In contrast to existing random forest methods for which ATE inferential guarantees would confine feature dimensions to one, our methods advance the field}. They support multi-dimensional features, \jelenax{therefore} significantly broaden\jelenax{ing} the practical application\jelenax{s} for forest-based ATE inference.

\section{Adaptive Split Balancing Forest}\label{sec:balance}

The regression tree models function $m(\cdot)$ by recursively partitioning the feature space $[0,1]^{d}$ into non-overlapping rectangles, generally called leaves or nodes. For any given point $\bx \in [0,1]^d$, a regression tree estimates $m(\bx)$ using the average of responses for those samples in the same leaf as $\bx$: $\mathrm{T}(\bx,\xi) = \sum_{i\in\Isc}\mathbbm{1}_{\{\bX_i \in L(\bx,\xi)\}}Y_i/\sum_{i\in\Isc}\mathbbm{1}_{\{\bX_i \in L(\bx,\xi)\}}$, where $\xi$ denotes all the auxiliary randomness in the tree-growing process and is independent of the samples, $\Xi$ denotes the support of $\xi$, $\Isc\subseteq\{1,\dots,N\}$ is the indices of training samples used for local averaging and possibly depends on $\xi$, and $L(\bx,\xi)$ represents the terminal leaf containing the point $\bx$.

Random forests consider ensembles of regression trees, where the forests' predictions are the average of all the tree predictions. Let $\{\mathrm{T}(\bx,\xi_j), j=1, \dots, B\}$ denote the collection of regression trees in a forest, where $B$ is the number of trees and $\xi_1, \dots, \xi_B\in\Xi$ are i.i.d. auxiliary variables. For any $B\geq1$, random forests estimate the conditional mean as $\widehat{m}(\bx):=B^{-1}\sum_{j=1}^B\mathrm{T}(\bx,\xi_j)=\E_{\xi}[\mathrm{T}(\bx,\xi)],$ where for any function $f(\cdot)$, $\E_{\xi}[f(\bx)]=B^{-1}\sum_{j=1}^Bf(\xi_j)$ denotes the empirical average over the auxiliary variables, and we omit the dependence of such an expectation on $B$ for the sake of notation simplicity. Using the introduced notations, random forests can also be represented as a weighted average of the outcomes:
\begin{align}\label{weight}
 \widehat m (\bx) = \E_{\xi}\left[\sum_{i\in\Isc}\omega_i(\bx,\xi)Y_i\right],\;\;\mbox{where}\;\;\omega_i(\bx,\xi):=\frac{\mathbbm{1}_{\left \{\bX_i\in L(\bx,\xi)\right\}}}{\sum_{l\in\Isc}\mathbbm{1}_{\left\{\bX_l \in L(\bx,\xi)\right\}}}.
\end{align}

To study the estimation \jelenax{rates} of random forests, we consider the following decomposition of IMSE: $\E_{\bx}\left[\mhat(\bx)-m(\bx)\right]^2 \leq 2R_1+2R_2,$ where $R_1:=\E_{\bx}\left[\E_{\xi}\left[\sum_{i\in\Isc}\omega_i(\bx,\xi)\varepsilon_i\right]\right]^2$ is the stochastic error originating from the random noise $\varepsilon_i = Y_i-m(\bX_i)$, and $R_2:=\E_{\bx}\left[\E_{\xi}\left[\sum_{i\in\Isc}\omega_i(\bx,\xi)(m(\bX_i)-m(\bx))\right]\right]^2$ is the approximation error. Let $k$ be the minimum leaf size. Standard techniques lead to $R_1=O_p(1/k)$; see \eqref{thm:balance_eq2} of the Supplement. The control of the remaining approximation error is the key to reaching an optimal overall IMSE.

\subsection{Auxiliary randomness and approximation error}\label{sec:aux}

\jelenax{We focus on the Lipschitz class here, with the H\"older class explored in Section \ref{LLCF_consistency}.} In the following, we illustrate how the auxiliary randomness introduced by the widely-used random feature selection affects the approximation error of tree models.

\begin{assumption}[Lipschitz continuous]\label{cond:lip}
Assume that $m(\cdot)$ satisfies $|m(\bx)-m(\bx')|\leq L_0\|\bx-\bx'\|$ for all $\bx,\bx' \in [0,1]^d$ with some constant $L_0>0$. 
\end{assumption}

For any leaf $L\subseteq[0,1]^d$, \jelenax{let} $\mathrm{diam}(L):=\sup_{\bx,\bx'\in L}\|\bx-\bx'\|$ \jelenax{be} its diameter. Under the Lipschitz condition, the approximation error can be controlled by the leaves' diameters: $R_2\leq L_0^2\E_{\bx}[\E_{\xi}[\mathrm{diam}^2(L(\bx,\xi))]].$ Therefore, it suffices to obtain an upper bound for the diameters. 
Although the Breiman's algorithm \jelenax{selects up to $\mbox{mtry} \leq d$ features at random}, we find it worthwhile to study the special case with $\mbox{mtry}=1$, as seen in centered and median forests. \jelenax{Recall that} in the centered forest, a splitting direction is randomly selected (with probability $1/d$) for each split, and the splitting location is chosen as the center point of the parent node. For the moment, let $B=\infty$, meaning the forest is the ensemble of infinitely many trees. Suppose that each terminal leaf has been split for $\mathcal M$ times, and consider a sequence of $\mathcal M$ consecutive leaves containing $\bx$, with the smallest (terminal) leaf denoted as $L(\bx,\xi)$.

For each $m\leq\mathcal M$ and $j\leq d$, let $\delta_{j,m}(\bx,\xi)=1$ if the $m$-th split is performed along the $j$-th coordinate, and $\delta_{j,m}(\bx,\xi)=0$ otherwise. Here, $\P_\xi$ and $\E_\xi$ represent the corresponding probability measure and the expectation taken with respect to $\xi$, respectively. For any given $j\leq d$, the sequence $(\delta_{j,m}(\bx,\xi))_{m=1}^{\mathcal M}$ is i.i.d., with $\P_\xi(\delta_{j,m}(\bx,\xi)=1)=1/d$ \jelenax{for centered forest}. Let $\mathrm{diam}_j(L(\bx,\xi))$ be the length of the longest segment parallel to the $j$-th axis that is a subset of $L(\bx,\xi)$. Then,
\begin{align}
&\E_{\xi}[\mathrm{diam}^2(L(\bx,\xi))]=\sum_{j=1}^d\E_{\xi}[\mathrm{diam}_j^2(L(\bx,\xi))]=\sum_{j=1}^d\E_{\xi}\left[\prod_{m=1}^{\mathcal M}2^{-2\delta_{j,m}(\bx,\xi)}\right]\nonumber\\
&\qquad=\sum_{j=1}^d\E_{\xi}\left[2^{-2c_j(\bx,\xi)}\right]\overset{(i)}{>}\sum_{j=1}^d2^{-2\E_{\xi}[c_j(\bx,\xi)]}=\sum_{j=1}^d2^{-2\sum_{m=1}^{\mathcal M}1/d}=d2^{-2\mathcal M/d}.\label{bound:diam_lower}
\end{align}
The discrepancy introduced by the strict inequality (i),``Jensen gap'', stems from the variation in the quantity $c_j(\bx,\xi):=\sum_{m=1}^{\mathcal M}\delta_{j,m}(\bx,\xi)$, representing the count of splits along the $j$-th direction induced by the auxiliary randomness $\xi$. This discrepancy results in a relatively large approximation error $R_2\asymp (1-1/(2d))^{2\mathcal M}$ (up to logarithmic terms). By selecting an optimal $\mathcal M$ (or $k$) that strikes a balance between stochastic and approximation errors, the centered forest yields an overall IMSE \jelenax{of the size} $N^\frac{-2\log(1-1/(2d))}{2\log(1-1/(2d))-\log (2)}$ -- which is \emph{not} minimax optimal for Lipschitz functions as long as $d>1$, as depicted in Figure \ref{fig:rate}.

The sub-optimality stems from the forests' excessive reliance on auxiliary randomness, leading to a significant number of redundant and inefficient splits. When splitting directions are chosen randomly, there is a non-negligible probability that some directions are overly selected while others are scarcely chosen. Consequently, terminal leaves tend to be excessively wide in some directions and overly narrow in others. For centered and median forests, this long and narrow leaf structure is solely due to auxiliary randomness and is unrelated to the data. This prevalence of long and narrow leaves increases the expected leaf diameter, resulting in a significant approximation error that deviates from the optimal one.
Averaging across the trees stabilizes the forest, but to attain minimax optimal rates, controlling the number of splits at the \emph{tree level} is essential to mitigate the Jensen gap, as shown in \eqref{bound:diam_lower}.

Instead of selecting splitting directions randomly, we adopt a \jelenax{permutation-based} and more balanced approach. By ensuring a sufficiently large $c_j(\bx,\xi)$ for each $j \leq d$ and reducing dependence on auxiliary randomness $\xi$, we achieve an approximation error of $O((N/k)^{-2/d})$, as shown in Lemma \ref{lem:balance diam}. This upper bound mirrors the right-hand side of \eqref{bound:diam_lower} (whereas centered or median forests cannot improve the left-hand side) when $\mathcal{M} \approx \log_2(N/k)$, resulting in a minimax optimal rate for the overall IMSE when $\mathcal{M}$ (or $k$) is chosen to balance the approximation and the stochastic error. 

\subsection{Adaptive Split Balancing}\label{sec:cyclic}

In order to reduce the large approximation error caused by auxiliary randomness, we propose a \jelenax{permutation-based splitting method}. Each time a leaf is split, we randomly select a direction from one of the sides that has been \jelenax{\it split the least number of times}. \jelenax{Moreover}, the splitting directions are chosen in a balanced fashion -- we have to split once in each direction \jelenax{on any given tree path} before proceeding to the next round.

\begin{algorithm}[h!] \caption{Adaptive split balancing forests}\label{alg:balance_RF}
\begin{algorithmic}[1]
\Require Observations $\S_N=(\bX_i,Y_i)_{i=1}^N$, with $B\geq1$, $\alpha\in(0,0.5]$, $w\in(0,1]$, and $k\leq\lfloor wN\rfloor$.
\For{$b=1,\dots,B$}
\State Divide $\S_N$ into disjoint $\S_\Isc^{(b)}$ and $\S_\Jsc^{(b)}$ with $|\Isc^{(b)}|=\lfloor wN\rfloor$ and $|\Jsc^{(b)}|=N-\lfloor wN\rfloor$.
\Repeat { For each current node $L\subseteq[0,1]^d$:}
 \State Select direction $j$ along which the node has been split the least number of times. 
\State Partition along $j$-th direction to minimize the mean squared error (MSE) on $\S_\Jsc^{(b)}$:
\begin{align}\label{rule:cyclic}
\sum_{i\in\Jsc^{(b)}}(Y_i-\Ybar_1)^2\mathbbm1\{\bX_i\in L_1\}+\sum_{i\in\Jsc^{(b)}}(Y_i-\Ybar_2)^2\mathbbm1\{\bX_i\in L_2\},\\ \label{rule:alpha}
\mbox{ensuring } \#\{i\in\Isc^{(b)}:\bX_i\in L_l\}\geq\alpha\#\{i\in\Isc^{(b)}:\bX_i\in L\},\; l=1,2.
\end{align}
 \hskip 34pt Here, $\Ybar_1$ and $\Ybar_2$ are the average responses within the child nodes $L_1$ and $L_2$. 
\Until{each node contains $k$ to $2k-1$ samples $\S_\Isc^{(b)}$.}
\State Estimate $m(\mathbf{x})$ with the $b$-th adaptive split balancing tree:
\begin{equation}\label{eq:ASBT}
T(\bx,\xi_b):=\frac{\sum_{i\in\Isc^{(b)}}\mathbbm{1}_{\left\{\bX_i \in L(\bx,\xi_b)\right\}}Y_i}{\sum_{i\in\Isc^{(b)}}\mathbbm{1}_{\left\{\bX_i \in L(\bx,\xi_b)\right\}}}.
\end{equation}
\EndFor\\
\Return Adaptive split balancing forest estimate $\mhat(\bx):=B^{-1}\sum_{b=1}^BT(\bx,\xi_b)$.
\end{algorithmic}
\end{algorithm}

 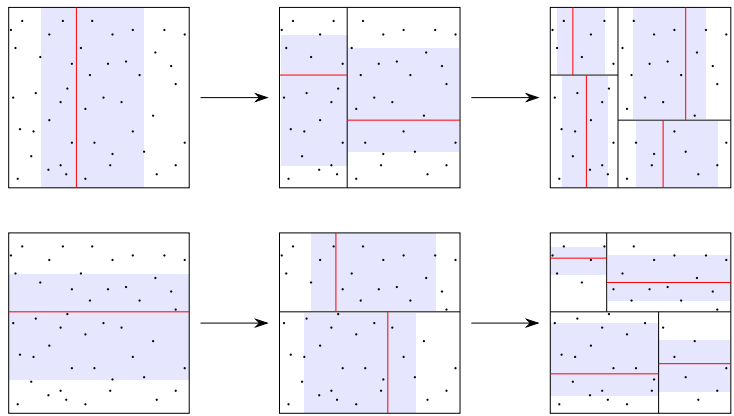
\begin{figure}[h!]
 \scriptsize
 \centering
 \scalebox{0.6}{
 \begin{tikzpicture}
 		\path
		(6,-1.08) node[rectangle,fill=blue!10, minimum width=39.7mm,minimum height=23.5mm](21) {}
		(6,-1) node[rectangle,draw,minimum width=40mm,minimum height=40mm](2) {}
		(5.85,4) node[rectangle,fill=blue!10,minimum width=22.8mm,minimum height=39.7mm](31) {}
		(6,4) node[rectangle,draw,minimum width=40mm,minimum height=40mm](3) {}
		(12.08,0.125) node[rectangle,fill=blue!10,minimum width=27.5mm,minimum height=17.1mm](41) {}
 (11.78,-1.885) node[rectangle,fill=blue!10,minimum width=24.7mm,minimum height=22.1mm](42) {}		
		(12,-1) node[rectangle,draw,minimum width=40mm,minimum height=40mm](4) {}
	 (10.77,3.94) node[rectangle,fill=blue!10,minimum width=14.75mm,minimum height=29mm](51) {}
 (12.74,3.94) node[rectangle,fill=blue!10,minimum width=24.68mm,minimum height=23mm](52) {}
		(12, 4) node[rectangle,draw,minimum width=40mm,minimum height=40mm](5) {}
		(16.77,3.25) node[rectangle,fill=blue!10,minimum width=10mm,minimum height=24.5mm](61) {}
(16.68,5.25) node[rectangle,fill=blue!10,minimum width=10.5mm,minimum height=14.5mm](62) {}
(18.65,4.75) node[rectangle,fill=blue!10,minimum width=16mm,minimum height=24.5mm](63) {}
(18.8,2.75) node[rectangle,fill=blue!10,minimum width=18mm,minimum height=14.5mm](64) {}
		(18,4) node[rectangle,draw,minimum width=40mm,minimum height=40mm](6) {}
		(16.63,0.38) node[rectangle,fill=blue!10,minimum width=12mm,minimum height=6mm](71) {}
(18.625,0) node[rectangle,fill=blue!10,minimum width=27mm,minimum height=10mm](72) {}
(19.18,-1.95) node[rectangle,fill=blue!10,minimum width=16mm,minimum height=11.5mm](73) {}
(17.21,-1.8) node[rectangle,fill=blue!10,minimum width=24mm,minimum height=16mm](74) {}
		(18,-1) node[rectangle,draw,minimum width=40mm,minimum height=40mm](7) {};
		
		\path (16.1,4) node[circle, fill, inner sep=0.5]{}
		(16.2,2.2) node[circle, fill, inner sep=0.5]{}
		(16.25,3.3) node[circle, fill, inner sep=0.5]{}
		(16.5,2.7) node[circle, fill, inner sep=0.5]{}
		(16.55,3.25) node[circle, fill, inner sep=0.5]{}
		(16.6,4.1) node[circle, fill, inner sep=0.5]{}
		(16.88,2.4) node[circle, fill, inner sep=0.5]{}
		(16.9,3.5) node[circle, fill, inner sep=0.5]{}
		(17.15,3.9) node[circle, fill, inner sep=0.5]{}
		(17.275,2.3) node[circle, fill, inner sep=0.5]{}
 (17.3,4.2) node[circle, fill, inner sep=0.5]{}
		(17.4,3) node[circle, fill, inner sep=0.5]{}
		(16.15,5.1) node[circle, fill, inner sep=0.5]{}
		(16.3,5.7) node[circle, fill, inner sep=0.5]{}
		(17.4,4.75) node[circle, fill, inner sep=0.5]{}
		(17.2,5.7) node[circle, fill, inner sep=0.5]{}
		(16.7,4.9) node[circle, fill, inner sep=0.5]{}
		(16.9,5.4) node[circle, fill, inner sep=0.5]{}
		(16.05,5.5) node[circle, fill, inner sep=0.5]{}
		(17.7,3.75) node[circle, fill, inner sep=0.5]{}
		(17.6,5.1) node[circle, fill, inner sep=0.5]{}
		(17.8,4.5) node[circle, fill, inner sep=0.5]{}
		(17.85,5.7) node[circle, fill, inner sep=0.5]{}
	 	(18.1,4) node[circle, fill, inner sep=0.5]{}
	 	(18.2,4.75) node[circle, fill, inner sep=0.5]{}
	 (18.3,5.4) node[circle, fill, inner sep=0.5]{}
	 (18.5,3.9) node[circle, fill, inner sep=0.5]{}
	 (18.6,4.8) node[circle, fill, inner sep=0.5]{}
	 (18.75,5.5) node[circle, fill, inner sep=0.5]{}
	 (18.9,4.5) node[circle, fill, inner sep=0.5]{}
	 (19.25,5) node[circle, fill, inner sep=0.5]{}
	 (19.2,3.6) node[circle, fill, inner sep=0.5]{}
 	(19.45,5.5) node[circle, fill, inner sep=0.5]{}
		(19.6,4.7) node[circle, fill, inner sep=0.5]{}
		(19.7,4.3) node[circle, fill, inner sep=0.5]{}
		(19.9,5.4) node[circle, fill, inner sep=0.5]{}
	 (17.7,2.2) node[circle, fill, inner sep=0.5]{}
		(17.9,3) node[circle, fill, inner sep=0.5]{}
		(17.15,2.5) node[circle, fill, inner sep=0.5]{}
		(18.25,2.5) node[circle, fill, inner sep=0.5]{}
		(18.3,2.75) node[circle, fill, inner sep=0.5]{}
		(18.75,3.25) node[circle, fill, inner sep=0.5]{}
		(19,2.8) node[circle, fill, inner sep=0.5]{}
		(19.28,2.3) node[circle, fill, inner sep=0.5]{}
		(19.7,2.5) node[circle, fill, inner sep=0.5]{}
		(19.9,3) node[circle, fill, inner sep=0.5]{}; 
		
		\path (16.1,-1) node[circle, fill, inner sep=0.5]{}
		(16.2,-2.8) node[circle, fill, inner sep=0.5]{}
		(16.25,-1.7) node[circle, fill, inner sep=0.5]{}
		(16.5,-2.3) node[circle, fill, inner sep=0.5]{}
		(16.55,-1.75) node[circle, fill, inner sep=0.5]{}
		(16.6,-0.9) node[circle, fill, inner sep=0.5]{}
		(16.88,-2.6) node[circle, fill, inner sep=0.5]{}
		(16.9,-1.5) node[circle, fill, inner sep=0.5]{}
		(17.15,-1.1) node[circle, fill, inner sep=0.5]{}
		(17.275,-2.7) node[circle, fill, inner sep=0.5]{}
		(17.3,-0.8) node[circle, fill, inner sep=0.5]{}
		(17.4,-2) node[circle, fill, inner sep=0.5]{}
		(16.15,0.1) node[circle, fill, inner sep=0.5]{}
		(16.3,0.7) node[circle, fill, inner sep=0.5]{}
		(17.4,-0.25) node[circle, fill, inner sep=0.5]{}
		(17.2,0.7) node[circle, fill, inner sep=0.5]{}
		(16.7,-0.1) node[circle, fill, inner sep=0.5]{}
		(16.9,0.4) node[circle, fill, inner sep=0.5]{}
		(16.05,0.5) node[circle, fill, inner sep=0.5]{}
		(17.7,-1.25) node[circle, fill, inner sep=0.5]{}
		(17.6,0.1) node[circle, fill, inner sep=0.5]{}
		(17.8,-0.5) node[circle, fill, inner sep=0.5]{}
		(17.85,0.7) node[circle, fill, inner sep=0.5]{}
		(18.1,-1) node[circle, fill, inner sep=0.5]{}
		(18.2,-0.25) node[circle, fill, inner sep=0.5]{}
		(18.3,0.4) node[circle, fill, inner sep=0.5]{}
		(18.5,-1.1) node[circle, fill, inner sep=0.5]{}
		(18.6,-0.2) node[circle, fill, inner sep=0.5]{}
		(18.75,0.5) node[circle, fill, inner sep=0.5]{}
		(18.9,-0.5) node[circle, fill, inner sep=0.5]{}
		(19.25,0) node[circle, fill, inner sep=0.5]{}
		(19.2,-1.4) node[circle, fill, inner sep=0.5]{}
		(19.45,0.5) node[circle, fill, inner sep=0.5]{}
		(19.6,-0.3) node[circle, fill, inner sep=0.5]{}
		(19.7,-0.7) node[circle, fill, inner sep=0.5]{}
		(19.9,0.4) node[circle, fill, inner sep=0.5]{}
		(17.7,-2.8) node[circle, fill, inner sep=0.5]{}
		(17.9,-2) node[circle, fill, inner sep=0.5]{}
		(17.15,-2.5) node[circle, fill, inner sep=0.5]{}
		(18.25,-2.5) node[circle, fill, inner sep=0.5]{}
		(18.3,-2.25) node[circle, fill, inner sep=0.5]{}
		(18.75,-1.75) node[circle, fill, inner sep=0.5]{}
		(19,-2.2) node[circle, fill, inner sep=0.5]{}
		(19.28,-2.7) node[circle, fill, inner sep=0.5]{}
		(19.7,-2.5) node[circle, fill, inner sep=0.5]{}
		(19.9,-2) node[circle, fill, inner sep=0.5]{};

		\path (10.1,4) node[circle, fill, inner sep=0.5]{}
		(10.2,2.2) node[circle, fill, inner sep=0.5]{}
		(10.25,3.3) node[circle, fill, inner sep=0.5]{}
		(10.5,2.7) node[circle, fill, inner sep=0.5]{}
		(10.55,3.25) node[circle, fill, inner sep=0.5]{}
		(10.6,4.1) node[circle, fill, inner sep=0.5]{}
		(10.88,2.4) node[circle, fill, inner sep=0.5]{}
		(10.9,3.5) node[circle, fill, inner sep=0.5]{}
		(11.15,3.9) node[circle, fill, inner sep=0.5]{}
		(11.275,2.3) node[circle, fill, inner sep=0.5]{}
		(11.3,4.2) node[circle, fill, inner sep=0.5]{}
		(11.4,3) node[circle, fill, inner sep=0.5]{}
		(10.15,5.1) node[circle, fill, inner sep=0.5]{}
		(10.3,5.7) node[circle, fill, inner sep=0.5]{}
		(11.4,4.75) node[circle, fill, inner sep=0.5]{}
		(11.2,5.7) node[circle, fill, inner sep=0.5]{}
		(10.7,4.9) node[circle, fill, inner sep=0.5]{}
		(10.9,5.4) node[circle, fill, inner sep=0.5]{}
		(10.05,5.5) node[circle, fill, inner sep=0.5]{}
		(11.7,3.75) node[circle, fill, inner sep=0.5]{}
		(11.6,5.1) node[circle, fill, inner sep=0.5]{}
		(11.8,4.5) node[circle, fill, inner sep=0.5]{}
		(11.85,5.7) node[circle, fill, inner sep=0.5]{}
		(12.1,4) node[circle, fill, inner sep=0.5]{}
		(12.2,4.75) node[circle, fill, inner sep=0.5]{}
		(12.3,5.4) node[circle, fill, inner sep=0.5]{}
		(12.5,3.9) node[circle, fill, inner sep=0.5]{}
		(12.6,4.8) node[circle, fill, inner sep=0.5]{}
		(12.75,5.5) node[circle, fill, inner sep=0.5]{}
		(12.9,4.5) node[circle, fill, inner sep=0.5]{}
		(13.25,5) node[circle, fill, inner sep=0.5]{}
		(13.2,3.6) node[circle, fill, inner sep=0.5]{}
		(13.45,5.5) node[circle, fill, inner sep=0.5]{}
		(13.6,4.7) node[circle, fill, inner sep=0.5]{}
		(13.7,4.3) node[circle, fill, inner sep=0.5]{}
		(13.9,5.4) node[circle, fill, inner sep=0.5]{}
		(11.7,2.2) node[circle, fill, inner sep=0.5]{}
		(11.9,3) node[circle, fill, inner sep=0.5]{}
		(11.15,2.5) node[circle, fill, inner sep=0.5]{}
		(12.25,2.5) node[circle, fill, inner sep=0.5]{}
		(12.3,2.75) node[circle, fill, inner sep=0.5]{}
		(12.75,3.25) node[circle, fill, inner sep=0.5]{}
		(13,2.8) node[circle, fill, inner sep=0.5]{}
		(13.28,2.3) node[circle, fill, inner sep=0.5]{}
		(13.7,2.5) node[circle, fill, inner sep=0.5]{}
		(13.9,3) node[circle, fill, inner sep=0.5]{};
		
		\path (10.1,-1) node[circle, fill, inner sep=0.5]{}
		(10.2,-2.8) node[circle, fill, inner sep=0.5]{}
		(10.25,-1.7) node[circle, fill, inner sep=0.5]{}
		(10.5,-2.3) node[circle, fill, inner sep=0.5]{}
		(10.55,-1.75) node[circle, fill, inner sep=0.5]{}
		(10.6,-0.9) node[circle, fill, inner sep=0.5]{}
		(10.88,-2.6) node[circle, fill, inner sep=0.5]{}
		(10.9,-1.5) node[circle, fill, inner sep=0.5]{}
		(11.15,-1.1) node[circle, fill, inner sep=0.5]{}
		(11.275,-2.7) node[circle, fill, inner sep=0.5]{}
		(11.3,-0.8) node[circle, fill, inner sep=0.5]{}
		(11.4,-2) node[circle, fill, inner sep=0.5]{}
		(10.15,0.1) node[circle, fill, inner sep=0.5]{}
		(10.3,0.7) node[circle, fill, inner sep=0.5]{}
		(11.4,-0.25) node[circle, fill, inner sep=0.5]{}
		(11.2,0.7) node[circle, fill, inner sep=0.5]{}
		(10.7,-0.1) node[circle, fill, inner sep=0.5]{}
		(10.9,0.4) node[circle, fill, inner sep=0.5]{}
		(10.05,0.5) node[circle, fill, inner sep=0.5]{}
		(11.7,-1.25) node[circle, fill, inner sep=0.5]{}
		(11.6,0.1) node[circle, fill, inner sep=0.5]{}
		(11.8,-0.5) node[circle, fill, inner sep=0.5]{}
		(11.85,0.7) node[circle, fill, inner sep=0.5]{}
		(12.1,-1) node[circle, fill, inner sep=0.5]{}
		(12.2,-0.25) node[circle, fill, inner sep=0.5]{}
		(12.3,0.4) node[circle, fill, inner sep=0.5]{}
		(12.5,-1.1) node[circle, fill, inner sep=0.5]{}
		(12.6,-0.2) node[circle, fill, inner sep=0.5]{}
		(12.75,0.5) node[circle, fill, inner sep=0.5]{}
		(12.9,-0.5) node[circle, fill, inner sep=0.5]{}
		(13.25,0) node[circle, fill, inner sep=0.5]{}
		(13.2,-1.4) node[circle, fill, inner sep=0.5]{}
		(13.45,0.5) node[circle, fill, inner sep=0.5]{}
		(13.6,-0.3) node[circle, fill, inner sep=0.5]{}
		(13.7,-0.7) node[circle, fill, inner sep=0.5]{}
		(13.9,0.4) node[circle, fill, inner sep=0.5]{}
		(11.7,-2.8) node[circle, fill, inner sep=0.5]{}
		(11.9,-2) node[circle, fill, inner sep=0.5]{}
		(11.15,-2.5) node[circle, fill, inner sep=0.5]{}
		(12.25,-2.5) node[circle, fill, inner sep=0.5]{}
		(12.3,-2.25) node[circle, fill, inner sep=0.5]{}
		(12.75,-1.75) node[circle, fill, inner sep=0.5]{}
		(13,-2.2) node[circle, fill, inner sep=0.5]{}
		(13.28,-2.7) node[circle, fill, inner sep=0.5]{}
		(13.7,-2.5) node[circle, fill, inner sep=0.5]{}
		(13.9,-2) node[circle, fill, inner sep=0.5]{};

		\path (4.1,4) node[circle, fill, inner sep=0.5]{}
		(4.2,2.2) node[circle, fill, inner sep=0.5]{}
		(4.25,3.3) node[circle, fill, inner sep=0.5]{}
		(4.5,2.7) node[circle, fill, inner sep=0.5]{}
		(4.55,3.25) node[circle, fill, inner sep=0.5]{}
		(4.6,4.1) node[circle, fill, inner sep=0.5]{}
		(4.88,2.4) node[circle, fill, inner sep=0.5]{}
		(4.9,3.5) node[circle, fill, inner sep=0.5]{}
		(5.15,3.9) node[circle, fill, inner sep=0.5]{}
		(5.275,2.3) node[circle, fill, inner sep=0.5]{}
		(5.3,4.2) node[circle, fill, inner sep=0.5]{}
		(5.4,3) node[circle, fill, inner sep=0.5]{}
		(4.15,5.1) node[circle, fill, inner sep=0.5]{}
		(4.3,5.7) node[circle, fill, inner sep=0.5]{}
		(5.4,4.75) node[circle, fill, inner sep=0.5]{}
		(5.2,5.7) node[circle, fill, inner sep=0.5]{}
		(4.7,4.9) node[circle, fill, inner sep=0.5]{}
		(4.9,5.4) node[circle, fill, inner sep=0.5]{}
		(4.05,5.5) node[circle, fill, inner sep=0.5]{}
		(5.7,3.75) node[circle, fill, inner sep=0.5]{}
		(5.6,5.1) node[circle, fill, inner sep=0.5]{}
		(5.8,4.5) node[circle, fill, inner sep=0.5]{}
		(5.85,5.7) node[circle, fill, inner sep=0.5]{}
		(6.1,4) node[circle, fill, inner sep=0.5]{}
		(6.2,4.75) node[circle, fill, inner sep=0.5]{}
		(6.3,5.4) node[circle, fill, inner sep=0.5]{}
		(6.5,3.9) node[circle, fill, inner sep=0.5]{}
		(6.6,4.8) node[circle, fill, inner sep=0.5]{}
		(6.75,5.5) node[circle, fill, inner sep=0.5]{}
		(6.9,4.5) node[circle, fill, inner sep=0.5]{}
		(7.25,5) node[circle, fill, inner sep=0.5]{}
		(7.2,3.6) node[circle, fill, inner sep=0.5]{}
		(7.45,5.5) node[circle, fill, inner sep=0.5]{}
		(7.6,4.7) node[circle, fill, inner sep=0.5]{}
		(7.7,4.3) node[circle, fill, inner sep=0.5]{}
		(7.9,5.4) node[circle, fill, inner sep=0.5]{}
		(5.7,2.2) node[circle, fill, inner sep=0.5]{}
		(5.9,3) node[circle, fill, inner sep=0.5]{}
		(5.15,2.5) node[circle, fill, inner sep=0.5]{}
		(6.25,2.5) node[circle, fill, inner sep=0.5]{}
		(6.3,2.75) node[circle, fill, inner sep=0.5]{}
		(6.75,3.25) node[circle, fill, inner sep=0.5]{}
		(7,2.8) node[circle, fill, inner sep=0.5]{}
		(7.28,2.3) node[circle, fill, inner sep=0.5]{}
		(7.7,2.5) node[circle, fill, inner sep=0.5]{}
		(7.9,3) node[circle, fill, inner sep=0.5]{};
		
		\path (4.1,-1) node[circle, fill, inner sep=0.5]{}
		(4.2,-2.8) node[circle, fill, inner sep=0.5]{}
		(4.25,-1.7) node[circle, fill, inner sep=0.5]{}
		(4.5,-2.3) node[circle, fill, inner sep=0.5]{}
		(4.55,-1.75) node[circle, fill, inner sep=0.5]{}
		(4.6,-0.9) node[circle, fill, inner sep=0.5]{}
		(4.88,-2.6) node[circle, fill, inner sep=0.5]{}
		(4.9,-1.5) node[circle, fill, inner sep=0.5]{}
		(5.15,-1.1) node[circle, fill, inner sep=0.5]{}
		(5.275,-2.7) node[circle, fill, inner sep=0.5]{}
		(5.3,-0.8) node[circle, fill, inner sep=0.5]{}
		(5.4,-2) node[circle, fill, inner sep=0.5]{}
		(4.15,0.1) node[circle, fill, inner sep=0.5]{}
		(4.3,0.7) node[circle, fill, inner sep=0.5]{}
		(5.4,-0.25) node[circle, fill, inner sep=0.5]{}
		(5.2,0.7) node[circle, fill, inner sep=0.5]{}
		(4.7,-0.1) node[circle, fill, inner sep=0.5]{}
		(4.9,0.4) node[circle, fill, inner sep=0.5]{}
		(4.05,0.5) node[circle, fill, inner sep=0.5]{}
		(5.7,-1.25) node[circle, fill, inner sep=0.5]{}
		(5.6,0.1) node[circle, fill, inner sep=0.5]{}
		(5.8,-0.5) node[circle, fill, inner sep=0.5]{}
		(5.85,0.7) node[circle, fill, inner sep=0.5]{}
		(6.1,-1) node[circle, fill, inner sep=0.5]{}
		(6.2,-0.25) node[circle, fill, inner sep=0.5]{}
		(6.3,0.4) node[circle, fill, inner sep=0.5]{}
		(6.5,-1.1) node[circle, fill, inner sep=0.5]{}
		(6.6,-0.2) node[circle, fill, inner sep=0.5]{}
		(6.75,0.5) node[circle, fill, inner sep=0.5]{}
		(6.9,-0.5) node[circle, fill, inner sep=0.5]{}
		(7.25,0) node[circle, fill, inner sep=0.5]{}
		(7.2,-1.4) node[circle, fill, inner sep=0.5]{}
		(7.45,0.5) node[circle, fill, inner sep=0.5]{}
		(7.6,-0.3) node[circle, fill, inner sep=0.5]{}
		(7.7,-0.7) node[circle, fill, inner sep=0.5]{}
		(7.9,0.4) node[circle, fill, inner sep=0.5]{}
		(5.7,-2.8) node[circle, fill, inner sep=0.5]{}
		(5.9,-2) node[circle, fill, inner sep=0.5]{}
		(5.15,-2.5) node[circle, fill, inner sep=0.5]{}
		(6.25,-2.5) node[circle, fill, inner sep=0.5]{}
		(6.3,-2.25) node[circle, fill, inner sep=0.5]{}
		(6.75,-1.75) node[circle, fill, inner sep=0.5]{}
		(7,-2.2) node[circle, fill, inner sep=0.5]{}
		(7.28,-2.7) node[circle, fill, inner sep=0.5]{}
		(7.7,-2.5) node[circle, fill, inner sep=0.5]{}
		(7.9,-2) node[circle, fill, inner sep=0.5]{};

		\draw[-, red] (5.5,6) -- (5.5, 2);
		\draw[-, red] (4,-0.75) -- (8, -0.75);
		\draw[-, black] (11.5,6) -- (11.5, 2);
		\draw[-, black] (10,-0.75) -- (14, -0.75);
		\draw[-, red] (10,4.5) -- (11.5, 4.5);
		\draw[-, red] (11.5,3.5) -- (14, 3.5);
		\draw[-, red] (11.25,1) -- (11.25, -0.75);
		\draw[-, red] (12.4,-0.75) -- (12.4, -3);
		
		\draw[-, black] (17.5,6) -- (17.5, 2);
		\draw[-, black] (16,-0.75) -- (20, -0.75);
		\draw[-, black] (16,4.5) -- (17.5, 4.5);
		\draw[-, black] (17.5,3.5) -- (20, 3.5);
		\draw[-, black] (17.25,1) -- (17.25, -0.75);
		\draw[-, black] (18.4,-0.75) -- (18.4, -3);
	 \draw[-, red] (16.8,2) -- (16.8, 4.5);
	 \draw[-, red] (16.5,4.5) -- (16.5, 6);
	 \draw[-, red] (18.5,2) -- (18.5, 3.5);
	 \draw[-, red] (19,3.5) -- (19, 6);
	 \draw[-, red] (16,0.44) -- (17.25, 0.44);
	 \draw[-, red] (17.25,-0.1) -- (20, -0.1);
	 \draw[-, red] (16,-2.125) -- (18.4, -2.125);
	 \draw[-, red] (18.4,-1.9) -- (20, -1.9);
		
		\draw[-{Stealth[length=3mm, width=2mm]}] (8.25,4) -- (9.75, 4);
		\draw[-{Stealth[length=3mm, width=2mm]}] (8.25,-1) -- (9.75, -1);
		\draw[-{Stealth[length=3mm, width=2mm]}] (14.25,4) -- (15.75, 4);
 \draw[-{Stealth[length=3mm, width=2mm]}] (14.25,-1) -- (15.75, -1);
	\end{tikzpicture}
	}
\caption{Illustrations of the balanced tree-growing process with $d = 2$. The purple shading represents the splitting range that satisfies the $\alpha$-fraction constraint. The red lines indicate the chosen splits.}\label{fig:eg}
\end{figure}

To further enhance the practical performance, \jelenax{we employ} data-dependent splitting rules \jelenax{ of the ``honest forests'' approach that are dependent} on both $\bX_i$ and $Y_i$. This flexibility is particularly valuable when the local smoothness level varies in different directions and locations. \jelenax{F}or each tree, \jelenax{we} partition the samples into two sets, denoted by $\Isc$ and $\Jsc$. Additionally, we impose constraints on the child node $\alpha$-fraction and terminal leaf size $k$. \jelenax{W}ith $\alpha \in (0, 0.5]$ and $k \in \mathbb{N}$, we require the following conditions to hold: (a) each child node contains at least an $\alpha$-fraction of observations within the parent node, and (b) the number of observations within terminal leaves is between $k$ and $2k-1$. The splitting locations are then determined to minimize the empirical mean squared error (or maximize the impurity gain) within each parent node, selected from the set of points satisfying the above conditions. For more details, refer to Algorithm \ref{alg:balance_RF} and an illustration in Figure \ref{fig:eg}.

\begin{algorithm} [h!] \caption{Sparse adaptive split balancing forests}\label{alg:balance_GRF}
\begin{algorithmic}[1]
\Require Observations $\S_N=(\bX_i,Y_i)_{i=1}^N$, with $B\geq1$, $\alpha\in(0,0.5]$, $w\in(0,1]$, $k\leq\lfloor wN\rfloor$, and $\mbox{mtry}\in\{1,\dots,d\}$.
\For{$b=1,\dots,B$}
 \State Divide $\S_N$ into disjoint $\S_\Isc^{(b)}$ and $\S_\Jsc^{(b)}$ with $|\Isc^{(b)}|=\lfloor wN\rfloor$ and $|\Jsc^{(b)}|=N-\lfloor wN\rfloor$.
\State Create a collection of index sets, $\mathcal Q([0,1]^d)=\{Q_1, \dots, Q_d\}$, with each $Q_j \subset \{1, \dots, d\}$ containing \textit{mtry} directions while ensuring that each direction $1, \dots, d$ appears in exactly \textit{mtry} of the sets $Q_1, \dots, Q_d$.

\Repeat { For each current node $L\subseteq[0,1]^d$:}

\State Randomly select a set $Q\in\mathcal Q(L)$.

\State Partition along $j \in Q$-th direction to minimize the MSE as in \eqref{rule:cyclic} and \eqref{rule:alpha}.

\If{$\#\mathcal{Q}(L) > 1$}
 \State $\mathcal{Q}(L_1) = \mathcal{Q}(L_2) = \mathcal{Q}(L) \setminus \{Q\}$.
\Else
 \State Randomly reinitialize $\mathcal{Q}(L_1)$ and $\mathcal{Q}(L_2)$ as in Steps 3.
\EndIf
\Until{each current node contains $k$ to $2k-1$ samples $\S_\Isc^{(b)}$.}
\State The $b$-th sparse adaptive split balancing tree estimates $m(\bx)$ as in \eqref{eq:ASBT}.
\EndFor\\
\Return Sparse adaptive split balancing forest as the average of $B$ trees.
\end{algorithmic}
\end{algorithm}

Any $\alpha$ controls the balance in leaf lengths, and any $\alpha > 0$ prevents greedy splits near the endpoints of the parent node. For any $\alpha < 0.5$, \jelenax{while} the terminal leaves may \jelenax{still vary in width}, \jelenax{they will be more balanced than before}, hence the term ``adaptive split-balancing.'' Unlike centered and median forests, this time, the long and narrow leaf structure depends on the data. This distinction sets it apart from methods with data-independent splitting rules (e.g., \cite{mourtada2020minimax, o2024minimax, gao2022towards}), \jelenax{allowing} us to improve empirical performance, \jelenax{particularly} when covariates have distinct local effects on the outcome.

While Algorithm \ref{alg:balance_RF} selects splitting points data-dependently and can adaptively learn local smoothness levels for different directions, it still encounters challenges in the presence of certain sparse structures. Notably, even if some covariates are entirely independent of the outcome, Algorithm \ref{alg:balance_RF} may still make splits along such directions. These splits are redundant, as they do not contribute to reducing the approximation error but increase the stochastic error, \jelenax{due to the smaller sample sizes in the resulting child nodes}.

To address sparse situations, we introduce a tuning parameter, $\text{mtry}\in\{1,\dots,d\}$, representing the number of candidate directions for each split. Unlike existing methods, we select candidate directions in a balanced fashion instead of randomly. In each round, we initiate a collection of $d$ index sets, each containing $\text{mtry}$ distinct directions, ensuring {\it each direction appears in exactly $\text{mtry}$ sets}. At each split, we {\it randomly select one unused index set for the current round}. Alternatively, we can achieve the same result by shuffling the directions before each round and setting the index sets as $\{1,2,\dots,\text{mtry}\}$, $\{2,3,\dots,\text{mtry}+1\}$, etc., in a random order; see Figure \ref{fig:balance_GRF}. The selected index set provides the candidate splitting directions for the current node. Details are provided in Algorithm \ref{alg:balance_GRF}.

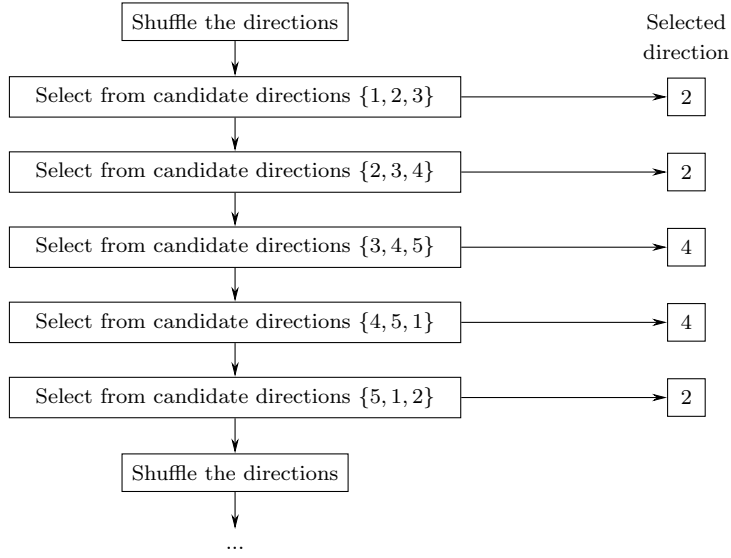
\begin{figure}[h!]
\vspace{1em}
	\centering
	\begin{tikzpicture}
		\path
		(0,0) node[rectangle,draw,minimum width=30mm,minimum height=5mm](A) {\scriptsize Shuffle the directions}
		(0,-1) node[rectangle,draw,minimum width=60mm,minimum height=5mm](A1) {\scriptsize Select from candidate directions $\{1,2,3\}$}
		(6,-1) node[rectangle,draw,minimum width=5mm,minimum height=5mm](A2) {\scriptsize $2$}
		(0,-2) node[rectangle,draw,minimum width=60mm,minimum height=5mm](B1) {\scriptsize Select from candidate directions $\{2,3,4\}$}
		(6,-2) node[rectangle,draw,minimum width=5mm,minimum height=5mm](B2) {\scriptsize $2$}
		(0,-3) node[rectangle,draw,minimum width=60mm,minimum height=5mm](C1) {\scriptsize Select from candidate directions $\{3,4,5\}$}
		(6,-3) node[rectangle,draw,minimum width=5mm,minimum height=5mm](C2) {\scriptsize $4$}
		(0,-4) node[rectangle,draw,minimum width=60mm,minimum height=5mm](D1) {\scriptsize Select from candidate directions $\{4,5,1\}$}
		(6,-4) node[rectangle,draw,minimum width=5mm,minimum height=5mm](D2) {\scriptsize $4$}
		(0,-5) node[rectangle,draw,minimum width=60mm,minimum height=5mm](E1) {\scriptsize Select from candidate directions $\{5,1,2\}$}
		(6,-5) node[rectangle,draw,minimum width=5mm,minimum height=5mm](E2) {\scriptsize $2$}
		(0,-6) node[rectangle,draw,minimum width=30mm,minimum height=5mm](F) {\scriptsize Shuffle the directions}
		(0,-7) node{\scriptsize $...$}
		(6,0) node {\scriptsize Selected }
		(6,-0.4) node {\scriptsize direction}
		;
		\draw[-{Stealth[length=2mm, width=1mm]}] (A) -- (A1);
		\draw[-{Stealth[length=2mm, width=1mm]}] (A1) -- (B1);
		\draw[-{Stealth[length=2mm, width=1mm]}] (B1) -- (C1);
		\draw[-{Stealth[length=2mm, width=1mm]}] (C1) -- (D1);
		\draw[-{Stealth[length=2mm, width=1mm]}] (D1) -- (E1);
		\draw[-{Stealth[length=2mm, width=1mm]}] (E1) -- (F);
		\draw[-{Stealth[length=2mm, width=1mm]}] (F) -- (0,-6.75);
		\draw[-{Stealth[length=2mm, width=1mm]}] (A1) -- (A2);
		\draw[-{Stealth[length=2mm, width=1mm]}] (B1) -- (B2);
		\draw[-{Stealth[length=2mm, width=1mm]}] (C1) -- (C2);
		\draw[-{Stealth[length=2mm, width=1mm]}] (D1) -- (D2);
		\draw[-{Stealth[length=2mm, width=1mm]}] (E1) -- (E2);
 \end{tikzpicture}
\caption{An illustration of splitting directions undergone by a leaf in Algorithm \ref{alg:balance_GRF}. }\label{fig:balance_GRF}
\end{figure}

When the true conditional mean function is sparse, setting a relatively large $\text{mtry}$ avoids splits on redundant directions. Conversely, when all directions contribute similarly, a smaller $\text{mtry}$ prevents inefficient splits due to sample randomness. However, existing methods select candidate sets randomly, causing some directions to be frequently considered as candidate directions. This randomness can reduce approximation power by neglecting important features. 
In contrast, our balanced approach ensures all features are equally considered \jelenax{for selection} at the {\it tree level},   making all the trees perform similarly and leading to a smaller estimation error for the forest.
 Due to the Jensen gap, the error of the averaged forest mainly stems from the poorer-performing trees, as the approximation power non-linearly depends on the number of splits in important features. The tuning parameters \(\alpha\) (location) and \(\text{mtry}\) (direction) control the ``greediness'' of each split during the tree-growing procedure. Greedy splits, characterized by a small \(\alpha\) and a large \(\text{mtry}\), lead to better adaptability but are sensitive to noise, making them sub-optimal under the Lipschitz or H\"older class. Conversely, non-greedy splits, which use a larger \(\alpha\) and a smaller \(\text{mtry}\), achieve minimax optimal results under smooth functional classes but are non-adaptive to complex (e.g., sparse) scenarios.

\subsection{Theoretical results}\label{sec:theory}

For simplicity, we consider uniformly distributed $\bX$ with support $[0,1]^d$, as seen in \cite{arlot2014analysis, genuer2012variance, biau2012analysis, klusowski2021sharp, cai2023extrapolated, duroux2018impact, scornet2015consistency, friedberg2020local, cattaneo2022pointwise, meinshausen2006quantile, lu2021unified, wager2018estimation}. We first highlight the advantages of the balanced approach.

\begin{lemma}\label{lem:balance diam}
For any $r>0 $ and $\alpha\in(0,0.5]$, the leaves constructed by Algorithm \ref{alg:balance_RF} satisfy 
$$\sup_{\bx\in[0,1]^d,\xi\in\Xi}\E_{\S_\Isc}[\mathrm{diam}^r(L(\bx,\xi))]< 2^r d^{\max\{r/2,1\}}\exp(r^2+r+1)(\lfloor wN\rfloor/(2k-1))^{-\frac{r\log (1-\alpha)}{d \log (\alpha)}}.$$
\end{lemma}
 \jelenax{The right most term of the above expression dominates the rest and determines the expected size of the diameter}
  $(\lfloor wN\rfloor/(2k-1))^{-\frac{r\log (1-\alpha)}{d \log (\alpha)}}\asymp(w/2)^{-\frac{r\log (1-\alpha)}{d \log (\alpha)}}\cdot(N/k)^{-\frac{r\log (1-\alpha)}{d \log (\alpha)}}\asymp(k/N)^{\frac{r\log (1-\alpha)}{d \log (\alpha)}}$.  
  \jelenax{Recall that the sample-splitting fraction $\omega \in (0,1]$ remains constant and that the minimum node size, $k$ is such that $k \leq \lfloor \omega N \rfloor$.} According to Lemma \ref{lem:balance diam}, the proposed forests' approximation error
 \jelenax{is at the order of} \( (k/N)^{\frac{2\log (1-\alpha)}{d \log (\alpha)}}\). \jelenax{Our} balanced median forest, \jelenax{with} \(\alpha=0.5\) yields a rate of \((k/N)^{\frac{2}{d}}\), \jelenax{whereas, a standard median forest with random feature selection has a strictly larger  rate of \((k/N)^{2\log_2\left(\frac{2d}{2d-1}\right)}\)} \cite{klusowski2021sharp}, \jelenax{whenever} \(d>1\); for \(d=1\), the rates are the same (no need to choose a splitting direction). \jelenax{Centered forests achieve comparable upper bounds, which suggests the sub-optimality of auxiliary randomness.}

We further assume the following standard condition for the noise variable.
\begin{assumption}\label{cond:noise}
Assume that $\E[\varepsilon^2\mid\bX]\leq M$ almost surely with some constant $M>0$.
\end{assumption} 

The following theorem characterizes the IMSE of the proposed forest in Algorithm \ref{alg:balance_RF}.

\begin{theorem}\label{thm:balance_consistency}
	Let Assumptions \ref{cond:lip} and \ref{cond:noise} hold. Suppose that $w\in(0,1]$ and $\alpha\in(0,0.5]$ are both constants. Choose any $B\in\mathbb N$ and $k\leq \lfloor wN\rfloor$. Then, as $N\to\infty$, the adaptive split balancing forest proposed in Algorithm \ref{alg:balance_RF} satisfies
$\E_{\bx}[\mhat(\bx)-m(\bx)]^2=O_p(1/k+ (k/N)^{\frac{2\log(1-\alpha)}{d \log(\alpha)}}).$ Moreover, let $k\asymp N^{\frac{2\log (1-\alpha)}{d\log (\alpha)+2\log (1-\alpha)}}$, we have 
$$\E_{\bx}[\mhat(\bx)-m(\bx)]^2=O_p( N^{-\frac{2\log (1-\alpha)}{d\log (\alpha)+2\log (1-\alpha)}} ).$$
\end{theorem}

\begin{figure}[htbp]
\centering
\includegraphics[width=1\textwidth]{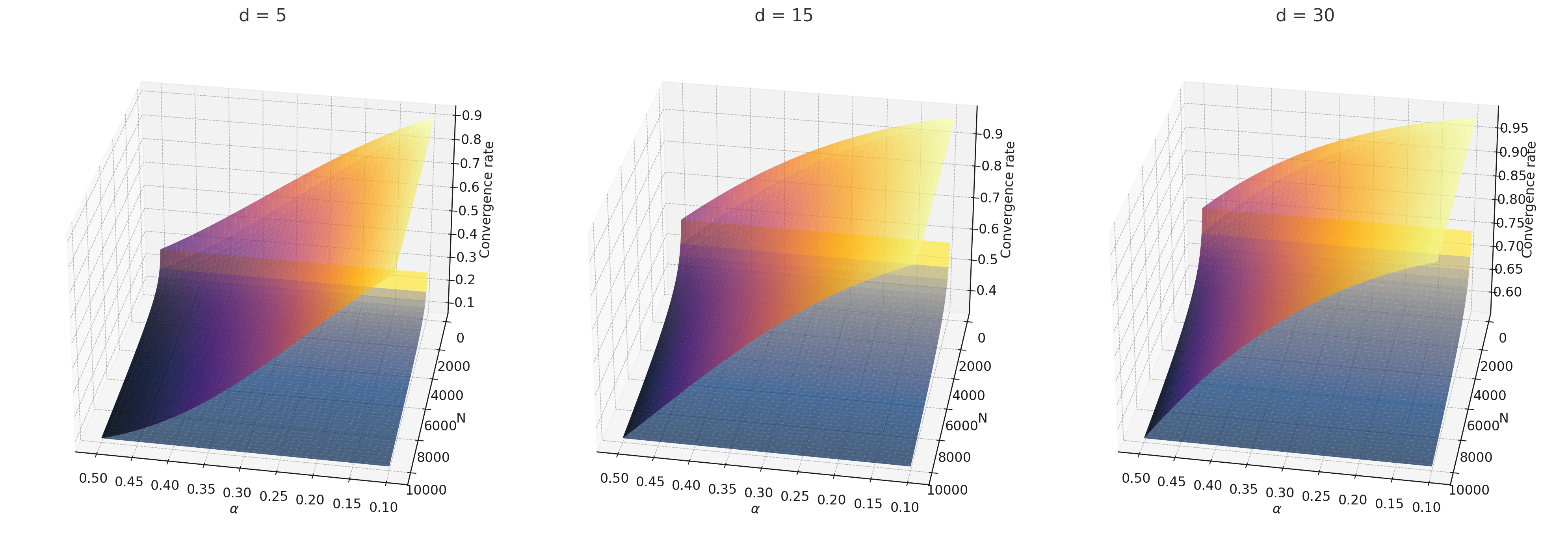}
\caption{IMSE convergence rate from Theorem \ref{thm:balance_consistency} sections for different dimension of features: $N^{-\frac{2\log (1-\alpha)}{d\log (\alpha)+2\log (1-\alpha)}} $ (top) and \(N^{-\frac{2}{d+2}}\) (bottom). Darker plot represents minimax optimal rate. The darker the cover of the overlap, the closer the rates. }
\label{fix:rates1}
\end{figure}

\jelenax{The issue of excess auxiliary randomness is inherent at the tree level, meaning that merely collecting more samples or further averaging does not address the optimality of the approximation error.  Our approach to resolving this involves a new splitting criterion that has yielded minimax optimal rates.}

\jelenax{ For instance, split-balancing median forest is the first minimax optimal median forest. When setting \(\alpha = 0.5\), split-balancing random forest becomes new, {\it balanced} median forest, with the optimal minimax rate of  \(N^{-\frac{2}{d+2}}\). } In contrast, existing median and center forests  \cite{klusowski2021sharp, duroux2018impact, biau2012analysis, arlot2014analysis}  do not achieve the minimax optimal rate; see Table \ref{table:rate} and Figure \ref{fig:rate}.
\jelenax{The conclusions of Theorem \ref{thm:balance_consistency} remain valid for  any constant \(\alpha\) in the range \((0, 0.5]\).} 
Greedy splits (small \(\alpha\)) enhance adaptability by leveraging information from both $\bX_i$ and $Y_i$, whereas non-greedy splits attain minimax optimality, necessitating a balance in practice. \jelenax{The rates are represented in Figure \ref{fix:rates1}.}

\section{Local Adaptive Split Balancing Forest}\label{LLCF_consistency}
In this section, we extend our focus to more general H{\"o}lder smooth functions and introduce \jelenax{local adaptive split-balancing} forests capable of exploiting higher-order smoothness levels.

\subsection{Local polynomial forests}

\begin{algorithm}[h!] \caption{Local adaptive split balancing forests}\label{alg:local RF}
\begin{algorithmic}[1]
\Require Observations $\S_N=(\bX_i,Y_i)_{i=1}^N$, with $B\geq1$, $\alpha\in(0,0.5]$, $w\in(0,1]$, $k\leq\lfloor wN\rfloor$, and $q\in\mathbb N$.
\State Calculate the polynomial basis $G(\bX_i)\in\R^{\dbar}$ for each $i\leq N$.
\For{$b=1,\dots,B$}
\State Divide $\S_N$ into disjoint $\S_\Isc^{(b)}$ and $\S_\Jsc^{(b)}$ with $|\Isc^{(b)}|=\lfloor wN\rfloor$ and $|\Jsc^{(b)}|=N-\lfloor wN\rfloor$.

\Repeat{ For each node $L\subseteq[0,1]^d$:}
 \State Select direction $j$ along which the node has been split the least number of times. 
\State Partition along $j$-th direction to minimize the MSE on $\S_\Jsc^{(b)}$ as in \eqref{rule:poly2} and \eqref{rule:alpha}.

\Until{each current node contains $k$ to $2k-1$ samples $\S_\Isc^{(b)}$.}
\State Estimate $\bbetahat(\bx,\xi_b)$ as defined in \eqref{def:gammahat} using observations $\S_\Isc^{(b)}$.
\State The $b$-th local adaptive split balancing tree: $T_\mathrm{L}(\bx,\xi_b):=G(\bx)^\top\bbetahat(\bx,\xi_b)$.

\EndFor\\
\Return The local adaptive split balancing forest $\mhat_\mathrm{L}(\bx):=B^{-1}\sum_{b=1}^BT_\mathrm{L}(\bx,\xi_b)$.
\end{algorithmic}
\end{algorithm}

To capture the higher-order smoothness of the conditional mean function $m(\cdot)$, we propose to fit a \jelenax{$q$-th} local polynomial regression within the leaves of the split-adaptive balanced forest of Section \ref{sec:cyclic}.
We first introduce the polynomial basis \jelenax{of} order $q\in\mathbb N$. For any $\bx\in[0,1]^d$ and $j\in\{0,1,\dots,q\}$, let $\bg_j(\bx):=(\bx^{\bgamma})_{\bgamma\in\mathcal A_j}\in\R^{d^j}$, where $\mathcal A_j:=\{\bgamma=(\bgamma_1,\dots,\bgamma_d)\in\mathbb N^d:|\bgamma|=j\}$.  For instance, $\bg_0(\bx)=1$, $\bg_1(\bx)=(\bx_1,\bx_2,\dots,\bx_d)^\top$, and $\bg_2(\bx)=(\bx_1^2,\bx_1\bx_2,\dots,\bx_d^2)^{\top}$. Denote $\bG(\bx):=(\bg_0(\bx),\bg_1(\bx)^\top,\dots,\bg_q(\bx)^\top)^\top\in\R^{\dbar}$ as the $q$-th order polynomial basis, where $\dbar:= \sum_{j=0}^q d^j$. 
For any $\xi\in\Xi$, define the weights $\omega_i(\bx,\xi)$ as in \eqref{weight}. Using the training samples indexed by $\Isc$, \jelenax{we} consider the weighted polynomial regression:
\begin{align}\label{def:gammahat}
\bbetahat(\bx,\xi):=\argminn_{\bbeta\in\R^{\dbar}}\sum_{i\in\Isc}\omega_i(\bx,\xi)(Y_i-\bG(\bX_i)^{\top}\bbeta)^2.
\end{align}
The $q$-th order local adaptive split balancing forest is proposed as 
$$\mhat_\mathrm{L}(\bx):=\E_{\xi}[\bG(\bx)^\top\bbetahat(\bx,\xi)].$$
Unlike Algorithm \ref{alg:balance_RF}, where local averages are used as tree predictions, our approach here involves conducting polynomial regressions within the terminal leaves. To optimize the behavior of the final polynomial regressions, we minimize the following during the leaf construction:
\begin{equation}\label{rule:poly2}
\sum_{i\in\Jsc^{(b)}}(\widehat{Y}_i-\widetilde{Y}_1)^2\mathbbm1\{\bX_i\in L_1\}+\sum_{i\in\Jsc^{(b)}}(\widehat{Y}_i-\widetilde{Y}_1)^2\mathbbm1\{\bX_i\in L_2\},
\end{equation}
where $\widehat{Y}_i:=Y_i-G(\bX_i)^\top\bbetahat$ with $\bbetahat$ denoting the least squares estimate within the parent node $L$, and $\widetilde{Y}_j$ is the average of $\widehat{Y}_i$ within the child node $L_j$ for each $j \in \{1,2\}$. As in Section \ref{sec:cyclic}, we also require that both child nodes contain at least an $\alpha$-fraction of samples from the parent node. 
Additional specifics are outlined in Algorithm \ref{alg:local RF}. Note that Algorithm \ref{alg:balance_RF} is a special case of Algorithm \ref{alg:local RF} with $q=0$.
To further improve the forests' empirical behavior in the presence of certain sparse structures, we also introduce a generalized sparse version of the local adaptive split balancing forest 
following; see Algorithm \ref{alg:local GRF} \jelenax{in} the Supplement for more details.

\subsection{Theoretical results with higher-order smoothness levels}

\begin{assumption}[H{\"o}lder smooth]\label{cond:holder}
Assume that $m(\cdot)\in\mathcal{H}^{q,\beta}$ with $q\in \mathbb{N}$ and $\beta\in(0,1]$. The H{\"o}lder class $\mathcal{H}^{q,\beta}$ contains all functions $f:[0,1]^d \to \R$ that are $q$ times continuously differentiable, with (a) $|D^{\bgamma} f(\bx)|\leq L_0$ for all $\bx\in[0,1]^d$ and multi-index $\bgamma$ satisfying $|\bgamma|\leq q$, and (b) $|D^{\bgamma} f(\bx)-D^{\bgamma} f(\bx')|\leq L_0\|\bx-\bx' \|^\beta$ for all $\bx,\bx'\in [0,1]^{d}$ and $\bgamma$ satisfying $|\bgamma|=q$, where $L_0>0$ is a constant.
\end{assumption}

The following theorem characterizes the convergence rate of the local adaptive split balancing forest proposed in Algorithm \ref{alg:local RF}.

\begin{theorem}\label{thm:local_consistency}
Let Assumptions \ref{cond:noise} and \ref{cond:holder} hold. Suppose that $w\in(0,1]$ and $\alpha\in(0,0.5]$ are both constants. Choose any $B\in\mathbb N$ and $k\leq \lfloor wN\rfloor$ satisfying $k\gg\log(N)$. Then, as $N\to \infty$, the local adaptive split balancing forest proposed in Algorithm \ref{alg:local RF} satisfies 
$$\E_{\bx}[\mhat_\mathrm{L}(\bx)-m(\bx)]^2=O_p\left(1/k+(k/N)^{\frac{2(q+\beta)\log (1-\alpha)}{d \log (\alpha)}}\right).$$
Moreover, let $k\asymp N^{\frac{2(q+\beta)\log (1-\alpha)}{d\log (\alpha)+2(q+\beta)\log (1-\alpha)}}$, we have
\begin{align}\label{local optimal rate}
\E_{\bx}\left[\mhat_\mathrm{L}(\bx)-m(\bx)\right]^2=O_p\left(N^{-\frac{2(q+\beta)\log (1-\alpha)}{d\log (\alpha)+2(q+\beta)\log (1-\alpha)}} \right).
\end{align}
\end{theorem}

When $\alpha=0.5$, the rate given by \eqref{local optimal rate} is $N^{-\frac{2(q+\beta)}{d+2(q+\beta)}}$, which is minimax optimal for the H\"{o}lder class $\mathcal{H}^{q,\beta}$; see, e.g., \cite{stone1982optimal}.  This accomplishment represents the first random forest with IMSE reaching minimax optimality when $q>1$; refer to Table \ref{table:rate} and Figure \ref{fig:rate}.

Theorem \ref{thm:local_consistency} \jelenax{is} built upon several key technical components. To mitigate the issue of multicollinearity during local polynomial regression within terminal leaves, we first demonstrate that \(\sup_{\bx\in[0,1]^d,\xi\in\Xi}\kappa_N(\bx)=O_p(1)\), where \(\kappa_N(\bx):=[1-\bd(\bx)^\top\bS^{-1}(\bx)\bd(\bx)]^{-1}\) assesses the degree of multicollinearity. Here, \(\bd(\bx):=\sum_{i\in\mathcal I}\omega_i(\bx,\xi)\bG_{-1}(\bX_i-\bx)\), \(\bS(\bx):=\sum_{i\in\mathcal I}\omega_i(\bx,\xi)\bG_{-1}(\bX_i-\bx)\bG_{-1}(\bX_i-\bx)^\top\), and \(\bG(\bx)=(1,\bG_{-1}(\bx)^\top)^\top\). Unlike \cite{friedberg2020local}, which imposes an upper bound on the random quantity \(\kappa_N\) by assumption, we rigorously prove that this quantity is uniformly bounded above with high probability; see the equivalent form in Lemma \ref{lem:sample eigen}. Additionally, by employing localized polynomial regression, we reduce the approximation error to \(O(\E_{\bx}[\mathrm{diam}^{2(q+\beta)}(L(\bx,\xi))])\) under the H\"older class \(\mathcal{H}^{q,\beta}\), as detailed in \eqref{consistency:R1} of the Supplement. Together with our proposed permutation-based splitting rule, previous results on leaf diameter (Lemma \ref{lem:balance diam}) ensure a sufficiently small rate.
  
\jelenax{\cite{mourtada2020minimax} achieved minimax optimal IMSE for \(q=1\) and \(\beta \in (0, 1/2]\) by leveraging forest structure and concentration across trees. However, for \(\beta \in (1/2, 1]\), their optimality is limited to interior points where \(\E_{\xi}[\E_{\bx'}[\bx' \mid \bx' \in L(\bx,\xi)]] \approx \bx\). Our approach broadens the range of H\"older smooth functions addressed, using polynomial approximation and split-balancing to eliminate the approximation error {\it at the tree level}.}

Tessellation forest \jelenax{of} \cite{o2024minimax}  addressed this boundary issue, achieving minimax optimal IMSE for \(q=1\) and any \(\beta \in (0,1]\). 
 Theorem \ref{thm:local_consistency} \jelenax{further highlights} that the \emph{local \jelenax{split-balancing}} median forest (with $\alpha=0.5$) \jelenax{is} minimax optimal not only \jelenax{for} $q=1$ but also \jelenax{for} any arbitrary $q>1$. This \jelenax{signifies a major breakthrough} in \jelenax{attaining} minimax optimality \jelenax{for} higher\jelenax{-order} smoothness levels. 

 \jelenax{D}ebiased Mondrian forests \jelenax{by} \cite{cattaneo2023inference} achieve minimax optimality for any \(q \in \mathbb{N}\) and \(\beta \in (0,1]\), in \jelenax{terms of} point-wise MSE at interior point\jelenax{s}. However, their results do not yield an optimal rate for the IMSE, as their \jelenax{debiasing technique does not extend} for boundary points. \jelenax{Additionally}, \cite{cai2023extrapolated} achieved a nearly optimal rates \jelenax{for any} 
 \(q \in \mathbb{N}\) and \(\beta \in (0,1]\) \jelenax{ but their findings are limited to} in-sample excess risk.
 
 \jelenax{It's important to note that while all previous works grow trees completely independently of the samples, our approach employs supervised splitting rules with $\alpha<0.5$ after tuning, significantly enhancing forest performance.} Only \cite{bloniarz2016supervised, friedberg2020local} considered data-dependent splitting rules and studied local linear forests under the special case $q=1$. However, \cite{bloniarz2016supervised} only demonstrated the consistency of their method, without providing any explicit rate of convergence. \cite{friedberg2020local} provided asymptotic normal results at a given $\bx\in[0,1]^d$. However, their established upper bound for the asymptotic variance is no faster than $N^{-(1+\frac{d\log(0.2)}{1.3\log(0.8)})^{-1}}\approx N^{-(1+5.55d)^{-1}}$, which is slow as the splitting directions are chosen randomly.

\subsection{Uniform \jelenax{consistency rate} }\label{sec:uniform}

In the following, we extend our analysis to include uniform-type results for the estimation error of the forests. These are useful for our particular applications to causal inference and more broadly are of independent interest. As Algorithm \ref{alg:balance_RF} is a specific instance of Algorithm \ref{alg:local RF}, we focus on presenting results for the latter.

To begin, we establish a uniform bound on the diameter of the leaves as follows.

\begin{lemma}\label{lem:balance diam_point-wise}
Suppose that $r\geq 1$, $w\in(0,1]$, and $\alpha\in(0,0.5]$ are constants. Choose any $k\leq \lfloor wN\rfloor$ satisfying $k\gg\log^3(N)$. Then, 
$$\sup_{\bx\in[0,1]^d,\xi\in\Xi}\mathrm{diam}^r(L(\bx,\xi))\leq C(N/k)^{-\frac{r\log (1-\alpha)}{d \log (\alpha)}}$$
with probability at least $1-\log(\lfloor wN\rfloor/k)/(\sqrt n\log\left((1-\alpha)^{-1}\right))$ and some constant $C>0$.
\end{lemma}

The result in Lemma \ref{lem:balance diam_point-wise} holds for a single tree and also for an ensemble forest. In fact, obtaining uniform results for a single tree is relatively simple -- it suffices to control the diameter of each terminal leaf with high probability and take the uniform bound over all the leaves, as the number of terminal leaves is at most $N/k$. However, such an approach is invalid for forests -- the number of all terminal leaves in a forest grows with the number of trees $B$. Hence, such a method can be used only when $B$ is relatively small; however, in practice, we would like to set $B$ as large as possible unless constrained by computational limits. To obtain a uniform result for forests, we approximate all the possible leaves through a collection of rectangles.

Subsequently, we present a uniform upper bound for the estimation error of the forests.

\begin{theorem}\label{thm:local_consistency_uniform}
Let Assumption \ref{cond:holder} hold. Suppose that $|Y|\leq M$. Let $M>0$, $w\in(0,1]$, and $\alpha\in(0,0.5]$ be constants. Choose any $B\in\mathbb N$ and $k\leq \lfloor wN\rfloor$ satisfying $k\gg\log^3(N)$. Then, as $N\to \infty$, 
$$\sup_{\bx\in[0,1]^d}|\mhat_\mathrm{L}(\bx)-m(\bx)|=O_p\left(\sqrt{\log(N)/k}+ (k/N)^{\frac{(q+\beta)\log (1-\alpha)}{d\log (\alpha)}}\right).$$
Moreover, let $k\asymp N^{\frac{2(q+\beta)\log (1-\alpha)}{d\log (\alpha)+2(q+\beta)\log (1-\alpha)}}\left(\log(N)\right)^{\frac{d\log (\alpha)}{d\log (\alpha)+2(q+\beta)\log (1-\alpha)}}$, we have 
$$\sup_{\bx\in[0,1]^d}|\mhat_\mathrm{L}(\bx)-m(\bx)|=O_p\left((\log(N)/N)^{\frac{(q+\beta)\log (1-\alpha)}{d\log (\alpha)+2(q+\beta)\log (1-\alpha)}}\right).$$
\end{theorem}

Comparing with the results in Theorem \ref{thm:local_consistency}, the rates above consist of additional logarithm terms. This is due to the cost of seeking uniform bounds. When $\alpha=0.5$, an optimally tuned $k$ leads to the rate $(\log(N)/N)^{\frac{q+\beta}{d+2(q+\beta)}}$, which is minimax optimal for sup-norms; see, e.g., \cite{stone1982optimal}. To the best of our knowledge, we are the first to establish minimax optimal uniform bounds for forests over the H{\"o}lder class $\mathcal H^{q,\beta}$ for any $q\in\mathbb N$.

\subsection{Application: causal inference for ATE}\label{sec:ATE}

In this section, we apply the proposed forests to estimate the average treatment effect (ATE) in the context of causal inference. We consider i.i.d. samples $(\bW_i)_{i=1}^N := (Y_i, \bX_i, A_i)_{i=1}^N$, and denote $\bW=(Y,\bX,A)$ as its independent copy. Here, $Y \in \mathbb{R}$ denotes the outcome of interest, $A \in \{0,1\}$ is a binary treatment variable, and $\bX \in \mathbb{R}^d$ represents a vector of covariates uniformly distributed in $[0,1]^d$. We operate within the potential outcome framework, assuming the existence of potential outcomes $Y(1)$ and $Y(0)$, where $Y(a)$ represents the outcome that would be observed if an individual receives treatment $a\in\{0,1\}$. The ATE is defined as $\theta := \mathbb{E}[Y(1) - Y(0)]$, representing the average effect of the treatment $A$ on the outcome $Y$. To identify causal effects, we make the standard assumptions seen in \cite{rosenbaum1983central, crump2009dealing, imbens2015causal}.

\begin{assumption}\label{identification}
(a) Unconfoundedness: $\{Y(0),Y(1)\} \independent A \mid \bX$. (b) Consistency: $Y = Y(A)$. (c) Overlap: $\P(c_0<\pi^*(\bX)<1-c_0)=1$, where $c_0\in (0,1/2)$ is a constant and the propensity score (PS) function is defined as $\pi^*(\bx):=\P(A=1\mid\bX=\bx)$ for any $\bx\in[0,1]^d$. 
\end{assumption}

Define the true outcome regression function $\mu_a^*(\bx):=\E[Y(a)\mid \bX=\bx]$ for $a\in\{0,1\}$ and consider the doubly robust score function: for any $\eta=(\mu_1, \mu_0, \pi)$,
\begin{equation*}
 \psi(\bW;\eta):= \mu_1(\bX)-\mu_0(\bX)+\frac{A(Y-\mu_1(\bX))}{\pi(\bX)}-\frac{(1-A)(Y-\mu_0(\bX))}{1-\pi(\bX)}.
\end{equation*}

As the ATE parameter can be represented as $\theta=\E[\psi(\bW;\eta^*)]$, it can be estimated as the empirical average of the score functions as long as we plug in appropriate estimates of the nuisance functions $\eta^*=(\mu_1^*, \mu_0^*, \pi^*)$. In the following, we introduce the forests-based ATE estimator using the double machine-learning framework of \cite{chernozhukov2017double}. 

For any fixed integer $K \geq 2$, split the samples into $K$ equal-sized parts, indexed by $(\Isc_k)_{k=1}^K$. For the sake of simplicity, we assume $n := \#\Isc_k = N/K \in \mathbb{N}$. For each $k \leq K$, denote $\Isc_{-k} = \Isc \setminus \Isc_{k}$. Under Assumption \ref{identification}, we can identify the outcome regression function as $\mu_a^*(\bx) = \E(Y \mid \bX = \bx, A = a)$ for each $a \in \{0,1\}$. Hence, we construct $\muhat_a^{-k}(\cdot)$ using Algorithm \ref{alg:local RF}, based on samples $(Y_i,\bX_i)_{i \in \{i \in \Isc_{-k} : A_i = a\}}$. Additionally, we also construct $\pihat^{-k}(\cdot)$ using Algorithm \ref{alg:local RF}, based on samples $(A_i, \bX_i)_{i \in \Isc_{-k}}$. For the sake of simplicity, we denote $\mu_2(\cdot) := \pi(\cdot)$. The number of trees $B$ and the orders of polynomial forests are chosen in advance, where we use $q_j$ to denote the polynomial orders considered in the estimation of $\mu_j(\cdot)$ for each $j \in \{0,1,2\}$. Further denote $h_j := (\alpha_j, w_j, k_j)$ as the hyperparameters for estimating $\mu_j(\cdot)$. To appropriately select $h_j$, we further split the samples indexed by $\Isc_{-k}$ into training and validation sets. After obtaining the nuisance estimates $\widehat\eta^{-k}:=(\muhat_1^{-k},\muhat_0^{-k},\pihat^{-k})$ for each $k\leq K$, we define the ATE estimator as $\thetahat:=N^{-1}\sum_{k=1}^K\sum_{i\in\Isc_k}\psi(\bW_i;\widehat\eta^{-k}).$

While the double machine-learning framework offers a flexible and robust approach for estimating causal parameters like the ATE, the validity of statistical inference based on these methods requires careful consideration, as asymptotic normality depends on sufficiently fast convergence rates in nuisance estimation. For instance, even when minimax optimal nuisance estimation is achieved within the Lipschitz class, inference can still be inaccurate under worst-case scenarios when $d>1$. To address this, we focus on the widely used random forest method and develop techniques that enhance convergence rates. By integrating the split-balancing technique with local polynomial regression, we enhance the method's convergence rate, thereby providing solid theoretical guarantees for forest-based ATE estimation and delivering more reliable inference in practice. Additional discussion on the technical challenges can be found in Remark \ref{remark:ATE}.

Now, we introduce theoretical properties of the ATE estimator.

\begin{theorem}\label{thm:ATE}
Let Assumption \ref{identification} hold, $|Y|\leq M$, and $\E[\mathbbm{1}_{\{A=a\}}(Y(a)-\mu_a^*)]^2\geq C_0$ for each $a\in\{0,1\}$, with some positive constants $M$ and $C_0$. Suppose that $\mu_0^*\in \mathcal{H}^{q_0,\beta_0}$, $\mu_1^*\in \mathcal{H}^{q_1,\beta_1}$, and $\pi^*\in \mathcal{H}^{q_2,\beta_2}$, where $q_j\in \mathbb{N}$ and $\beta_j\in(0,1]$ for each $j\in\{0,1,2\}$. Let $w_j\in(0,1]$ and $\alpha_j\in(0,0.5]$ be constants. Choose any $B\geq1$ and $k_j\asymp N^{\frac{2(q_j+\beta_j)\log (1-\alpha_j)}{d\log (\alpha_j)+2(q_j+\beta_j)\log (1-\alpha_j)}}.$ Moreover, let 
$$d^2\log (\alpha_a)\log (\alpha_2)<4(q_a+\beta_a)(q_2+\beta_2)\log (1-\alpha_a)\log (1-\alpha_2)$$ for each $a\in\{0,1\}$. Then, as $N\to\infty$, 
$$\sigma^{-1}\sqrt{N}(\thetahat-\theta)\leadsto N(0,1)$$ and $\sigmahat^{-1}\sqrt{N}(\thetahat-\theta)\leadsto N(0,1)$, where $\widehat{\sigma}^2:= {N}^{-1}\sum_{k=1}^K \sum_{i\in\mathcal I_k}[\psi(W_i;\widehat\eta^{-k})-\widehat{\theta}]^2$. 
\end{theorem}

\begin{remark}[Technical challenges of forest-based ATE estimation]\label{remark:ATE}
It is worth emphasizing that the following aspects are the main challenges in our analysis:

(a) Establish convergence rates for the integrated mean squared error (IMSE) of the nuisance estimates. As the ATE is a parameter defined through integration over the entire population, we require nuisance convergence results in the sense of IMSE; point-wise mean squared error results are insufficient. This distinguishes our work from \cite{wager2018estimation,athey2019generalized}, which focused on the estimation and inference for the conditional average treatment effect (CATE).

(b) Develop sufficiently fast convergence rates through higher-order smoothness. The asymptotic normality of the double machine-learning method \cite{chernozhukov2017double} requires a product-rate condition for the nuisance estimation errors. If we only utilize the Lipschitz continuity of the nuisance functions, root-$N$ inference is ensured only when $d=1$. In other words, we need to establish methods that can exploit the higher-order smoothness of nuisance functions as long as $d>1$. As shown in Theorem \ref{thm:ATE}, the higher the smoothness levels are, the larger dimension $d$ we allow for.

(c) Construct stable propensity score (PS) estimates. As demonstrated in Lemma \ref{lem:pihat} of the Supplement, as long as we ensure a sufficiently large minimum leaf size $k_2\gg\log^3(N)$ for the forests used in PS estimation, we can guarantee that each terminal leaf contains a non-negligible fraction of samples from both treatment groups, provided the overlap condition holds for the true PS function as in Assumption \ref{identification}. Consequently, we can stabilize the PS estimates, avoiding values close to zero.
\end{remark}

\section{Numerical Experiments}

In this section, we assess the numerical performance of the proposed methods in non-parametric regression through both simulation studies and real-data analysis. Additional results for the ATE estimation problem are provided in Section \ref{sec:sim-ATE} of the Supplement.

\subsection{Simulations for the conditional mean estimation}\label{sec:sim}

We first focus on the estimation of conditional mean function $m(x)=\E[Y\mid\bX=\bx]$. Generate i.i.d. covariates $\bX_i\sim\mathrm{Uniform}[0,1]^d$ and noise $\varepsilon_i\sim N(0,1)$ for each $i\leq N$. 

Consider the following models: (a) $Y_i=10\sin(\pi \bX_{i1}\bX_{i2})+20(\bX_{i3}-5)^2+10\bX_{i4}+5\bX_{i5}+\varepsilon_i$, (b) $Y_i=20\exp((\sum_{j=1}^s \bX_{ij}-0.5s)/\sqrt s) +\varepsilon_i$. In Setting (a), we utilize the well-known Friedman function proposed by \cite{friedman1991multivariate}, which serves as a commonly used benchmark for assessing non-parametric regression methods \cite{zhang2012bias,hothorn2021predictive,lu2021unified}. We set the covariates' dimension to $d=5$ and consider sample sizes $N\in\{500,1000\}$. In Setting (b), we investigate the performance of the forests under various sparsity levels, keeping $d=10$, $N=1000$, and choosing $s\in\{2,6,10\}$.

We implement the proposed adaptive split balancing forest (ASBF, Algorithm \ref{alg:balance_RF}), local linear adaptive split balancing forest (LL-ASBF, Algorithm \ref{alg:local RF} with $q=1$), and local quadratic adaptive split balancing forest (LQ-ASBF, Algorithm \ref{alg:local RF} with $q=2$). In Setting (b) where various sparsity levels are considered, we further evaluate the numerical performance of the sparse adaptive split balancing forest (S-ASBF, Algorithm \ref{alg:balance_GRF}), which is more suitable for scenarios with sparse structures. We choose $B=200$ and utilize $80\%$ of samples for training purposes, reserving the remaining $20\%$ for validation to determine the optimal tuning parameters $(\alpha, k)$, as well as $\mbox{mtry}$ for the sparse versions. For the sake of simplicity, we fix the honest fraction $w=0.5$ and do not perform additional subsampling.

\begin{figure*}[h!]
	\centering
\captionsetup[subfloat]{labelformat=empty}
\subfloat[(a) Sample size $N=500$]{\includegraphics[scale=0.57]{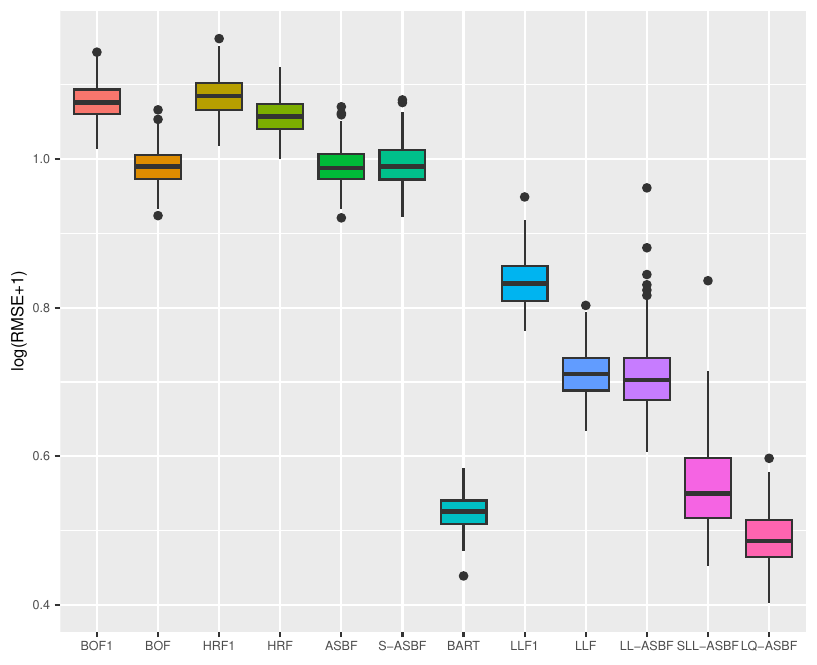}}
\subfloat[(b) Sample size $N=1000$]{
\includegraphics[scale=0.5]{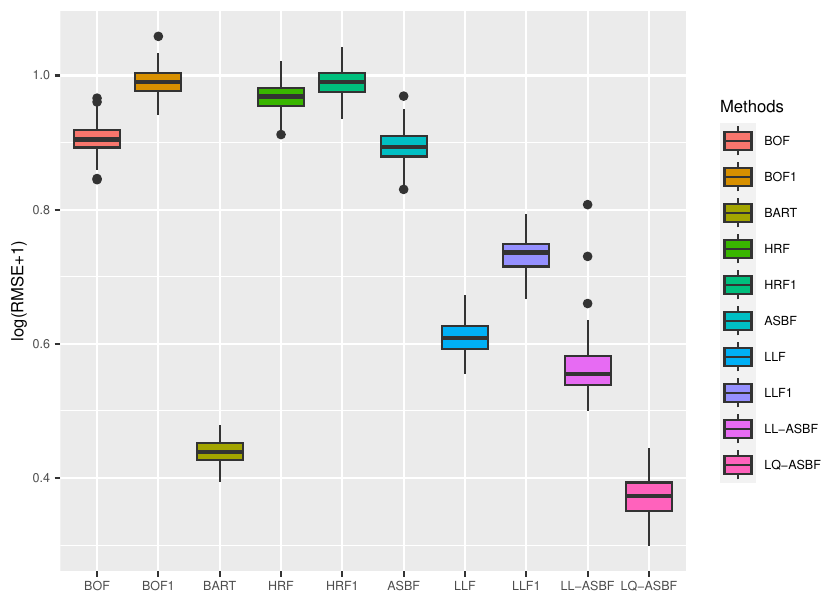}}
\caption{Boxplots of $\log(\text{RMSE}+1)$ under Setting (a) with a varying sample size.}\label{fig:a}
\end{figure*}

We also consider Breiman's original forest (BOF), honest random forest (HRF), local linear forest (LLF), and Bayesian additive regression trees (BART). BOF is implemented using the R package \texttt{ranger} \cite{wright2015ranger}, HRF and LLF are implemented using the R package \texttt{grf} \cite{tibshirani2023package}, and BART is implemented by the \texttt{BART} package \cite{sparapani2021nonparametric}. HRF and LLF methods involve the tuning parameter $\mathrm{mtry}$, denoting the number of directions tried for each split. For comparison purposes, we also consider modified versions with fixed $\mathrm{mtry}=1$. This corresponds to the case where splitting directions are randomly chosen and is the only case that has been thoroughly studied theoretically \cite{wager2018estimation,friedberg2020local}. We denote the modified versions of HRF and LLF as HRF1 and LLF1, respectively. The only difference between HRF1 and the proposed ASBF is that ASBF considers a balanced splitting approach for the selection of splitting directions, instead of a fully random way; a parallel difference exists between LLF1 and LL-ASBF. Additionally, we also introduce a modified version of BOF with the splitting direction decided through random selection, denoted as BOF1.

We evaluate the root mean square error (RMSE) within 1000 test points and repeat the procedure 200 times. Figures \ref{fig:a} and \ref{fig:b} depict boxplots comparing the log-transformed RMSE of all the considered methods across various settings introduced above.

\begin{figure*}[h!]
	\centering
\captionsetup[subfloat]{labelformat=empty}
\subfloat[(a) Sparsity level $s=2$]{\includegraphics[scale=0.55]{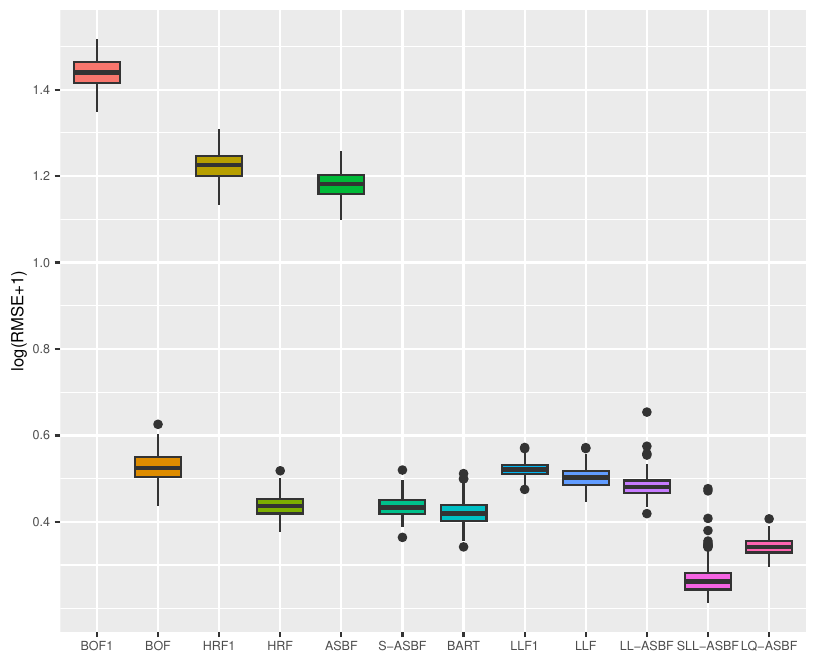}}
\hspace{-0.02\linewidth}
\subfloat[(b) Sparsity level $s=6$]{
\includegraphics[scale=0.55]{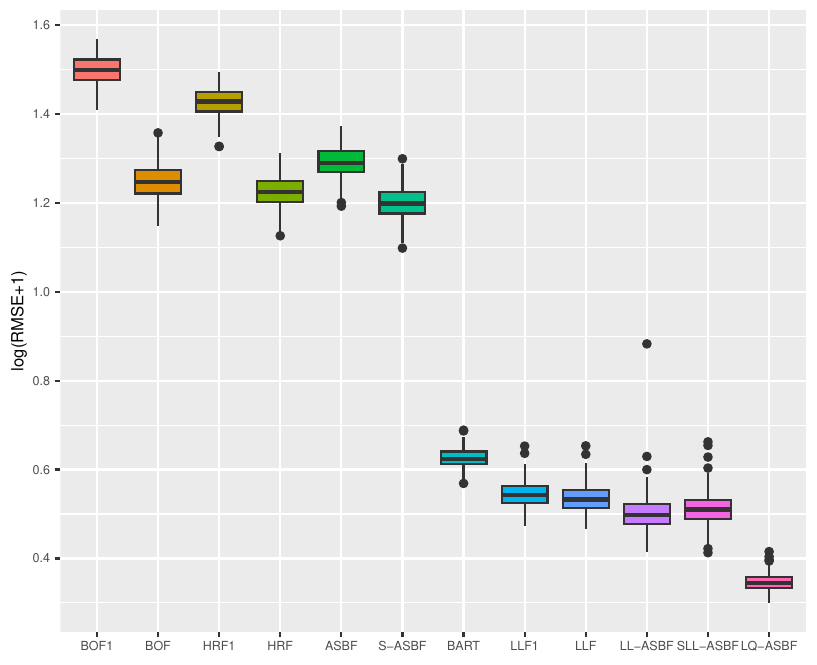}}
\hspace{-0.02\linewidth}
\subfloat[(c) Sparsity level $s=10$]{
\includegraphics[scale=0.5]{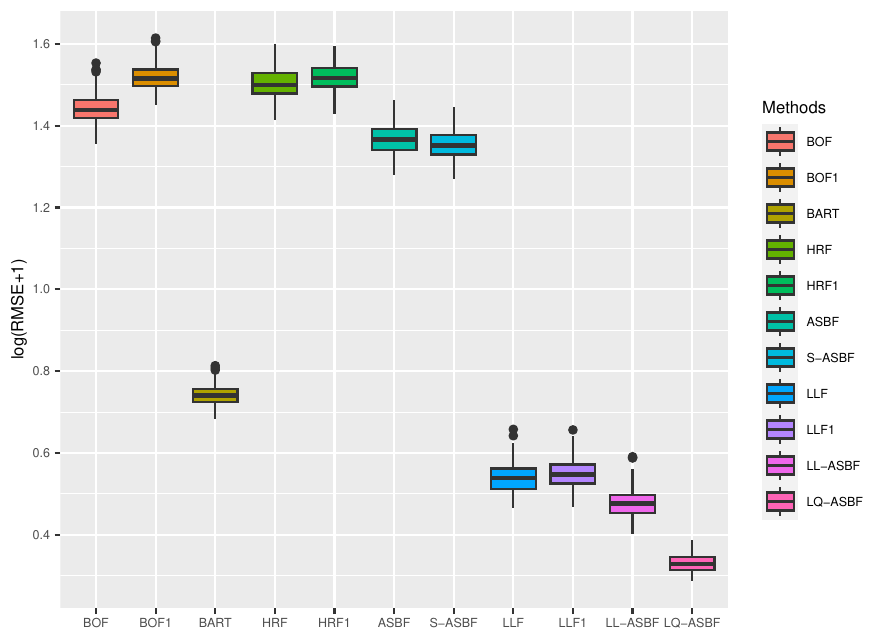}}
\caption{Boxplots of $\log(\text{RMSE}+1)$ under Setting (b) with a varying sparsity level.}\label{fig:b}
\end{figure*}

As shown in Figures \ref{fig:a} and \ref{fig:b}, LQ-ASBF consistently exhibits the best performance across all the considered settings. When we focus on the choice of $\mbox{mtry}=1$, the proposed ASBF method consistently outperforms the other local averaging methods BOF1 and HRF1, highlighting the distinct advantages offered by our balanced method in contrast to random feature selection. In addition, when the true model is dense, as shown in Figures \ref{fig:a} and \ref{fig:b}(c), the ASBF method (with a fixed $\mbox{mtry}=1$) outperforms the general BOF and HRF methods, even when their mtry parameters are appropriately tuned. The only exception is when $N=500$ under Setting (a), where ASBF and BOF show similar performance. In sparse scenarios, as demonstrated in Figures \ref{fig:b}(a)-(b), the proposed generalized sparse version S-ASBF, with an appropriately tuned mtry, clearly outperforms ASBF, especially when the sparsity level is small. Overall, the S-ASBF method consistently leads to a smaller RMSE than BOF and HRF methods in Figure \ref{fig:b} for all considered sparsity levels. The only exception occurs when $s=2$, where S-ASBF and HRF exhibit similar behaviors. This similarity arises because, in scenarios with a small true sparsity level, the optimal mtry parameter is close to the dimension $d$; otherwise, it is likely that all candidate directions are redundant for certain splits. Meanwhile, when $\mbox{mtry}=d$, there is no difference between the balanced and random approaches, as we always need to consider all directions as candidate directions for each split. Lastly, for forest-based local linear methods, we observe that the proposed balanced method LL-ASBF consistently outperforms both LLF1 and LLF under all considered scenarios. 

\subsection{Application to wine quality and abalone datasets}

We further assess the performance of the considered methods in Section \ref{sec:sim} using the wine quality and abalone datasets, both available from the UCI repository \cite{asuncion2007uci}.

The wine quality dataset comprises red and white variants of the Portuguese ``Vinho Verde'' wine, with 4898 observations for white and 1599 observations for red. The quality variable serves as the response, measured on a scale from 0 (indicating the worst quality) to 10 (representing the highest quality). Additionally, the dataset includes 11 continuous features. For detailed information about the data, refer to \cite{cortez2009modeling}.

The abalone dataset consists of 1 categorical feature and 7 continuous features, along with the age of abalones determined by cutting through the shell cone and counting the number of rings. This age is treated as the response variable. The categorical feature pertains to sex, classifying the entire dataset into three categories: male (1528 observations), female (1307 observations), and infant (1342 observations). For more details, refer to \cite{nash1994population}.

\begin{table}[h!]
	\centering
	\caption{Root mean square error across methods for wine quality and age of abalone} \label{table:real}
	\tabcolsep=0.14cm
	\scalebox{0.75}{
	\begin{tabular}{ccccccc|cccccc}
		\toprule
		Method&BOF1&BOF&HRF1&HRF&ASBF&S-ASBF&BART&LLF1&LLF&LL-ASBF&SLL-ASBF&LQ-ASBF\\
		\hline
		Wine (overall)&0.830&0.826&0.833&0.828&\bf0.808&\bf0.804&0.824&0.773&0.767&\bf0.728&\bf0.725&\bf{0.715}\\
		\hdashline
		Red wine&0.809&0.796&0.819&0.809&\bf0.794&\bf0.791&0.799&0.733&0.735&\bf0.729&\bf0.719&\bf{0.719}\\
		White wine&0.837&0.835&0.837&0.834&\bf0.812&\bf0.809&0.832&0.786&0.778&\bf0.728&\bf0.727&\bf{0.714}\\
		\hline
		Abalone (overall)&2.619&2.629&2.608&2.600&\bf2.551&\bf2.550&2.611&2.557& 2.556&\bf2.539&\bf2.537&\bf{2.497}\\
		\hdashline			
		Male abalone&2.685&2.697&2.687&2.682&\bf2.679&\bf2.677&2.717&2.635&2.633&\bf2.622&\bf2.622&\bf{2.570}\\
		Female abalone&3.063&3.066&3.019&3.005&\bf2.990&\bf2.989&3.004&2.955&2.957&\bf2.936&\bf2.931&\bf{2.889}\\
		Infant abalone&2.006&2.025&2.025&2.018&\bf1.839&\bf1.838&2.016&1.993&1.988&\bf1.962&\bf1.961&\bf{1.944}\\
		\bottomrule
	\end{tabular}}
\end{table}

Based on the categorical features, we initially divide the wine quality dataset into two groups (red and white) and the abalone dataset into three groups (male, female, and infant). Random forests are then constructed based on samples within each of the sub-groups. We standardize the continuous features using min-max scaling, ensuring that all features fall within the range $[0,1]$. Each group of the data is randomly partitioned into three parts. With a total group size of $N$, $\lceil 3N/5 \rceil$ observations are used for training, $\lceil N/5 \rceil$ observations for validation to determine optimal tuning parameters, and the prediction performance of the considered methods is reported based on the remaining testing observations. The tree size and subsampling ratio for honesty are chosen as $B=200$ and $w=0.5$ in advance.

\begin{figure*}[h!]
	\centering
\captionsetup[subfloat]{labelformat=empty}
\subfloat[Wine quality]{\includegraphics[height=0.3\linewidth,width=0.4\linewidth]{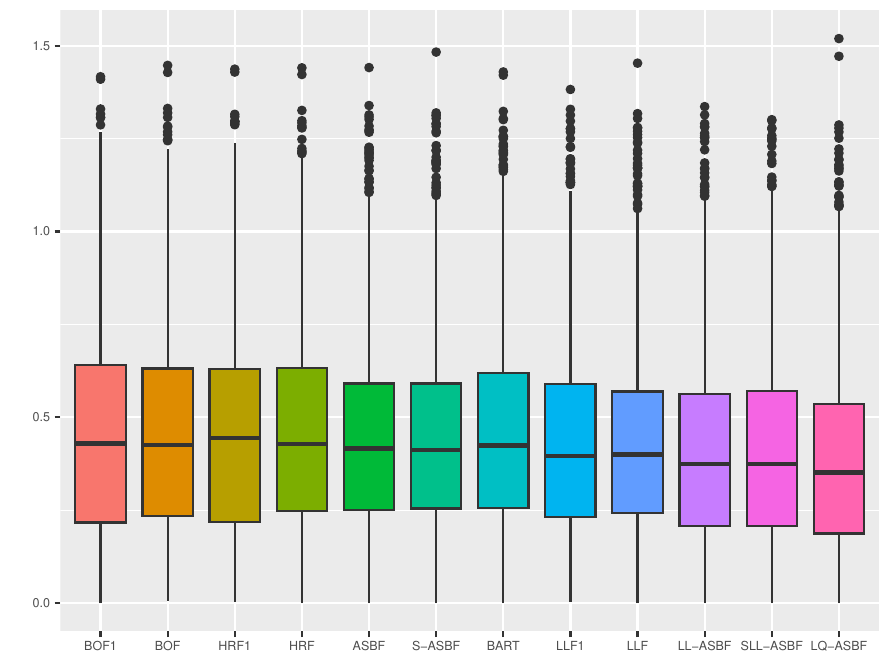}}
\subfloat[Age of abalone]{\includegraphics[height=0.3\linewidth,width=0.45\linewidth]{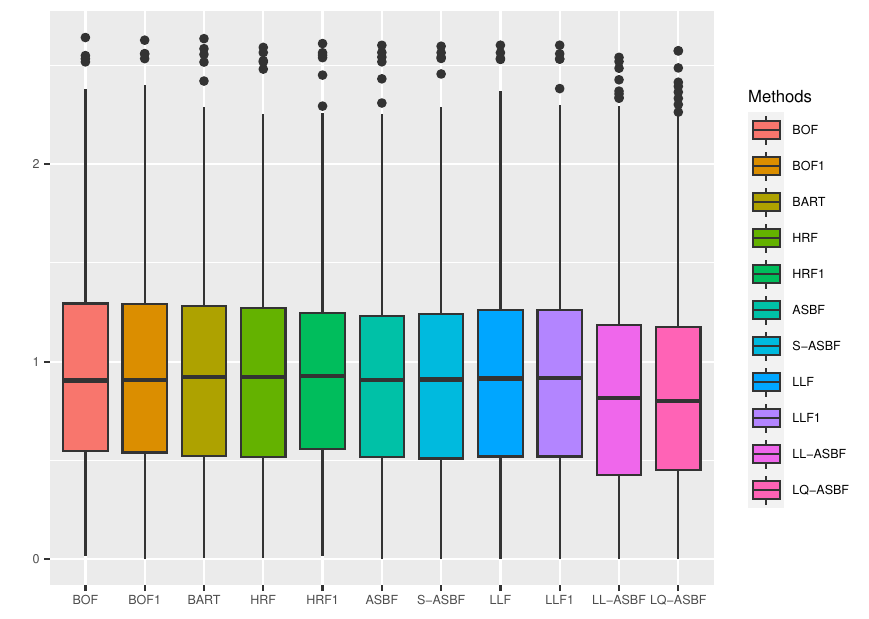}}
\caption{Boxplots of the $\log(\text{absolute errors}+1)$ for wine quality and age of abalone.}\label{fig:real}
\end{figure*}

Table \ref{table:real} reports the prediction performance of the considered random forest methods within each of the sub-groups. The proposed ASBF method and the sparse version S-ASBF outperform existing local averaging methods (including BOF, BOF1, HRF, and HRF), as well as BART, across all the sub-groups. The LL-ASBF method further outperforms existing local linear methods (LLF and LLF1), while the LQ-ASBF provides the most accurate prediction overall. Boxplots of the log-transformed absolute errors are also included in Figure \ref{fig:real}, illustrating the overall performance for the wine quality dataset (including white and red) and the age of abalone dataset (including male, female, and infant).

\section{Discussion}\label{sec:dis}

Since the introduction of random forest methods by \cite{breiman2001random}, the infusion of randomness has played a pivotal role in mitigating overfitting and reducing the variance associated with individual greedy trees. However, this work raises pertinent concerns and queries regarding the over-reliance on such auxiliary randomness. Even for a simple median forest, opting for completely random splitting directions does not yield optimal results. Conversely, when we choose directions in a less random, or more balanced manner, we can achieve minimax results for smooth functions. Notably, as auxiliary randomness lacks information about the conditional distribution $\P_{Y\mid\bX}$ of interest, overemphasizing its role in constructing regression methods does not necessarily improve results; rather, it can compromise the approximation power of tree models. Our theoretical and numerical findings suggest that, especially for low-dimensional problems, adopting a more balanced approach in constructing trees and forests leads to more efficient outcomes. While our numerical results also indicate the efficacy of the proposed balanced method in complex scenarios, such as those involving sparse structures, further in-depth investigation is needed to understand its performance comprehensively in intricate situations.

\appendix

\bibliographystyle{abbrv}
\bibliography{ref}


\renewcommand{\thetheorem}{S.\arabic{theorem}}
\renewcommand{\thelemma}{S.\arabic{lemma}}


\clearpage\newpage 
\begin{center}
\textbf{\uppercase{Supplementary Materials for ``Adaptive Split Balancing for Optimal Random Forest''}}
\end{center}

\par\medskip

\paragraph*{Notation} 
We denote rectangles $L\in [0,1]^d$ by $R=\bigotimes_{j=1}^{d}[a_j,b_j]$, where $0\leq a_j<b_j \leq 1$ for all $j=1,\dots,d$, writing the Lebesgue measure of $L$ as $\lambda(L)=\prod_{j=1}^{d}(b_j-a_j)$.
The indicator function of a subset $A$ of a set $X$ is a function $\mathbbm{1}_A$ defined as $\mathbbm{1}_A=1$ if $x\in A$, and $\mathbbm{1}_A=0$ if $x\notin A$.
For any rectangle $L\in[0,1]^d$, we denote $\mu(L):=\E[\mathbbm{1}_{\{ \bX \in L\}}]$ as the expected fraction of training examples falling within $L$. Denote $\# L:=\sum_{i\in\mathcal I}\mathbbm{1}_{\{ \bX_i\in L\}}$ as the number of training samples $\bX_i$ falling within $L$. For any $n\times n$ matrix $\bA$, let $\Lambda_{\min}(\bA)$ and $\Lambda_{\max}(\bA)$ denote the smallest and largest eigenvalues of the matrix $\bA$, respectively. A $d$-dimensional vector of all ones is denoted with $\mathbf{1}_{d}$. A tree grown by recursive partitioning is called $(\alpha,k)$-regular for some $\alpha\in(0,0.5]$ and $k \in \mathbb{N}$ if the following conditions to hold for the $\Isc$ sample: (a) each child node contains at least an $\alpha$-fraction of observations within the parent node, and (b) the number of observations within terminal leaves is between $k$ and $2k-1$.

\section{The sparse local adaptive split balancing forests}\label{sec:SLASBF}

In the following, we provide a generalized sparse version of the local adaptive split balancing forests proposed in Algorithm \ref{alg:local RF}.

\begin{algorithm} [h!] \caption{Sparse local adaptive split balancing forests}\label{alg:local GRF}
\begin{algorithmic}[1]
\Require Observations $\S_N=(\bX_i,Y_i)_{i=1}^N$, with parameters $B\geq1$, $\alpha\in(0,0.5]$, $w\in(0,1]$, $k\leq\lfloor wN\rfloor$, $\mbox{mtry}\in\{1,\dots,d\}$, and $q\in\mathbb N$.
\State Calculate the polynomial basis $G(\bX_i)\in\R^{\dbar}$ for each $i\leq N$.
\For{$b=1,\dots,B$}
\State Divide $\S_N$ into disjoint $\S_\Isc^{(b)}$ and $\S_\Jsc^{(b)}$ with $|\Isc^{(b)}|=\lfloor wN\rfloor$ and $|\Jsc^{(b)}|=N-\lfloor wN\rfloor$.
\State Create a collection of index sets, $\mathcal Q([0,1]^d)=\{Q_1, \dots, Q_d\}$, with each $Q_j \subset \{1, \dots, d\}$ containing \textit{mtry} directions while ensuring that each direction $1, \dots, d$ appears in exactly \textit{mtry} of the sets $Q_1, \dots, Q_d$.\Repeat { For each current node $L\subseteq[0,1]^d$:}
\State Randomly select a set $Q\in\mathcal Q(L)$.
\State Partition along $j \in Q$-th direction to minimize 
\begin{align*}
\sum_{i\in\Jsc^{(b)}}(\widehat{Y}_i-\widetilde{Y}_1)^2\mathbbm1\{\bX_i\in L_1\}+\sum_{i\in\Jsc^{(b)}}(\widehat{Y}_i-\widetilde{Y}_1)^2\mathbbm1\{\bX_i\in L_2\},\\
\mbox{ensuring } \#\{i\in\Isc^{(b)}:\bX_i\in L_l\}\geq\alpha\#\{i\in\Isc^{(b)}:\bX_i\in L\},\; l=1,2.
\end{align*}
\If{$\#\mathcal{Q}(L) > 1$}
 \State $\mathcal{Q}(L_1) = \mathcal{Q}(L_2) = \mathcal{Q}(L) \setminus \{Q\}$.
\Else
 \State Randomly reinitialize $\mathcal{Q}(L_1)$ and $\mathcal{Q}(L_2)$ as in Steps 4.
\EndIf
\Until{each current node contains $k$ to $2k-1$ samples $\S_\Isc^{(b)}$.}
\State Estimate $\bbetahat(\bx,\xi_b)$ as defined in \eqref{def:gammahat} using observations $\S_\Isc^{(b)}$.
\State The $b$-th sparse local adaptive split balancing tree: $T_\mathrm{L}(\bx,\xi_b):=G(\bx)^\top\bbetahat(\bx,\xi_b)$.
\EndFor\\
\Return The sparse local adaptive split balancing forest $\mhat_\mathrm{L}(\bx):=B^{-1}\sum_{b=1}^BT_\mathrm{L}(\bx,\xi_b)$.
\end{algorithmic}
\end{algorithm}

The generalized version considers an extra tuning parameter mtry, which has been also introduced in Algorithm \ref{alg:balance_GRF}, and performs local polynomial regressions within the terminal leaves as in Algorithm \ref{alg:local RF}. It is worth noting that Algorithms \ref{alg:balance_RF}-\ref{alg:local RF} are all special cases of the most general version Algorithm \ref{alg:local GRF}.

\section{Simulations for the ATE estimation}\label{sec:sim-ATE}

In this section, we evaluate the behavior of the forest-based ATE estimator proposed in Section \ref{sec:ATE} through simulation studies.

We focus on the estimation of $\theta=\E[Y(1)-Y(0)]$ and describe the considered data generating processes below. Generate i.i.d. covariates $\bX_i\sim\mathrm{Uniform}[0,1]^d$ and and noise $\varepsilon_i\sim N(0,1)$ for each $i\leq N$. Let $A_{i}\mid \bX_i \sim \mathrm{Bernoulli}(\pi^*(\bX_i))$ for each $i\leq N$. The outcome variables are generated as $Y_i=A_iY_i(1)+(1-A_i)Y_i(0)$. Consider the following models for the propensity score and outcomes:

\begin{itemize}
	\item[(a)] Consider $\pi^*(\bX_i)=((\sum_{j=1}^d\bX_{ij})/d+1.1)/((\sum_{j=1}^d\bX_{ij})/d+2)$, $Y_i(1)=(\sum_{j=1}^{d-1}\bX_{ij}\bX_{i(j+1)}+\bX_{id}\bX_{i1})/d+\varepsilon_i$, and $Y_i(0)=-(\sum_{j=1}^{d-1}\bX_{ij}\bX_{i(j+1)}+\bX_{id}\bX_{i1})/d+\varepsilon_i$.
	\item[(b)] Consider $\pi^*(\bX_i)=(\sum_{j=1}^d\bX_{ij})/(\sum_{j=1}^d\bX_{ij}+d)$, $Y_i(1)=2(\sum_{j=1}^{d-1}\bX_{ij}\bX_{i(j+1)}+\bX_{id}\bX_{i1})+\varepsilon_i$, and $Y_i(0)=-2(\sum_{j=1}^{d-1}\bX_{ij}\bX_{i(j+1)}+\bX_{id}\bX_{i1})+\varepsilon_i$.
\end{itemize}

\begin{table}[h!]
	\centering
	\caption{Simulations for the forest-based ATE estimation. Bias: empirical bias; RMSE: root mean square error; Length: average length of the $95\%$ confidence intervals; Coverage: average coverage of the $95\%$ confidence intervals. All the reported values (except Coverage) are based on robust (median) estimates.} \label{table:ATE}
	\begin{tabular}{lccccccccc}
		\toprule
		Method&Bias&RMSE&Length&Coverage&&Bias&RMSE&Length&Coverage\\
		\hline
		\multicolumn{1}{c}{} & \multicolumn{4}{c}{ \cellcolor{gray!50} Setting (a): $N=1000,d=5$}&\multicolumn{1}{c}{}& \multicolumn{4}{c}{ \cellcolor{gray!50} Setting (b): $N=500,d=5$}\\
		\cline{2-5}
		BOF1&0.020&0.060&0.298&0.950&&-0.018&0.108&0.698&0.955\\
		BOF&0.020&0.062&0.303&0.970&&-0.012&0.106&0.714&0.945\\
		HRF1&0.022&0.052&0.270&0.960&&0.020&0.100&0.631&0.955\\
		HRF&0.018&0.053&0.270&0.950&&0.015&0.100&0.634&0.955\\
		ASBF&\bf{0.016}&\bf{0.051}&0.263&0.955&&\bf{0.011}&\bf{0.098}&0.627&0.950\\
		\hdashline
		BART&0.022&0.055&0.269&0.960&&-0.015&0.104&0.632&0.940 \\
		LLF1&0.016&0.052&0.271&0.960&&-0.011&0.099&0.619&0.935\\
		LLF&0.015&0.051&0.271&0.950&&-0.016&0.102&0.622&0.930\\
		LL-ASBF&\bf{0.013}&\bf{0.049}&0.267&0.960&&\bf{-0.010}&\bf{0.097}&0.616&0.945\\
		LQ-ASBF&\bf{0.009}&\bf{0.047}&0.261&0.950&&\bf{-0.005}&\bf{0.092}&0.613&0.940\\
		\bottomrule
	\end{tabular}
\end{table}

In settings (a) and (b), we designate sample sizes as 1000 and 500, respectively, with covariate dimensions fixed at $d=5$. Each setting is replicated 200 times. The results, presented in Table \ref{table:ATE}, show that in both settings, all considered methods exhibit coverages close to the desired $95\%$. In terms of estimation, the ATE estimator based on our proposed ASBF method outperforms other local averaging methods (BOF, BOF1, HRF, HRF1), as well as BART. Across both settings, we observe smaller biases (in absolute values) and RMSEs, highlighting the superior importance of the balanced technique. Furthermore, the proposed LL-ASBF consistently outperforms existing local linear methods LLF and LLF1. Notably, LQ-ASBF exhibits the best performance among the considered settings.

\section{Auxiliary Lemmas}\label{sec:lemma}

\begin{lemma}[Theorem 7 of \cite{wager2015adaptive}]\label{lem:R cardinality}
Let $\mathcal{D}=\{1,2,\dots,d\}$ and $\omega, \epsilon\in (0,1)$. Then, there exists a set of rectangles $\mathcal{R}_{\mathcal{D},\omega,\epsilon}$ such that the following properties hold. Any rectangle $L$ of volume $\lambda(L) \geq \omega$ can be well approximated by elements in $\mathcal{R}_{\mathcal{D},\omega,\epsilon}$ from both above and below in terms of Lebesgue measure. Specifically, there exist rectangles $R_{-},R_{+}\in\mathcal{R}_{\mathcal{D},\omega,\epsilon}$ such that
\begin{align*}
R_{-}\subseteq L \subseteq R_{+}\;\;\text{and}\;\; \exp\{-\epsilon\}\lambda(R_{+})\leq \lambda(L) \leq \exp\{\epsilon\}\lambda(R_{-}).
\end{align*}
Moreover, the set $\mathcal{R}_{\mathcal{D},\omega,\epsilon}$ has cardinality bounded by
\begin{align*}
\# \mathcal{R}_{\mathcal{D},\omega,\epsilon}=\frac{1}{\omega}\left( \frac{8d^2}{\epsilon^2}\left(1+\log_{2}\left \lfloor \frac{1}{\omega} \right \rfloor \right)\right)^d\cdot(1+O(\epsilon)).
\end{align*}
\end{lemma}

\begin{lemma}[Theorem 10 of \cite{wager2015adaptive}]\label{lem:R lowerbound}
Suppose that $w\in(0,1]$, and $\alpha\in(0,0.5]$ are constants. Choose any $k\leq n = \lfloor wN\rfloor$ satisfying $k\gg\log(N)$. 
Let $\mathcal{L}$ be the collection of all possible leaves of partitions satisfying $(\alpha,k)$-regular. Let $\mathcal{R}_{\mathcal{D},\omega,\epsilon}$ be as defined in Lemma \ref{lem:R cardinality}, with $\omega$ and $\epsilon$ choosing as
\begin{align}\label{par: R}
\omega=\frac{k}{2n}\;\;\text{and}\;\;\epsilon=\frac{1}{\sqrt k}.
\end{align}
Then, there exists an $n_0 \in\mathbb{N}$ such that, for every $n \geq n_0$, the following statement holds with probability at at least $1-n^{-1/2}$: for each leaf $L\in\mathcal{L}$, we can select a rectangle
$\Rbar\in \mathcal{R}_{\mathcal{D},\omega,\epsilon}$ such that $\Rbar\subseteq L$, $\lambda(L) \leq \exp\{\epsilon\}\lambda(\Rbar)$, and
\begin{align*}
\# L-\# \Rbar\leq 3\epsilon \# L +2\sqrt{3\log(\#\mathcal{R}_{\mathcal{D},\omega,\epsilon})\# L}+O\left(\log(\#\mathcal{R}_{\mathcal{D},\omega,\epsilon})\right).
\end{align*}
\end{lemma}

\begin{lemma}[Lemma 12 of \cite{wager2015adaptive}]\label{lem:event A}
Fix a sequence $\delta(n)>0$, and define the event
\begin{align*}
\Asc:=\left\{\sup\left\{ \frac{|\#R-n\mu(R)|}{\sqrt{n\mu(R)}}: R\in \mathcal{R}, \mu(R)\geq\mu_{\min} \right\}\leq\sqrt{3\log\left(\frac{\# \mathcal{R} }{\delta}\right)}\right\}
\end{align*}
for any set of rectangles $\mathcal{R}$ and threshold $\mu_{\min}$, where $\# R:= \# \{i\in\Isc:\bX_i\in R \}$ and $\# \mathcal{R}$ is the number of rectangles of the set $\mathcal{R}$. Then, for any sequence of
problems indexed by $n$ with
\begin{align}
\lim_{n\to\infty}\frac{\log(\# \mathcal{R})}{n\mu_{\min}}=0\;\;\text{and}\;\; \lim_{n\to\infty}\frac{\delta^{-1}}{\# \mathcal{R}}=0,\label{par:event A}
\end{align}
there is a threshold $n_0\in\mathbb{N}$ such that, for all $n\geq n_0$, we have $\P(\Asc)\geq 1-\delta$.
Note that, above, $\Asc$, $\mathcal{R}$, $\mu_{\min}$ and $\delta$ are all implicitly changing with $n$.
\end{lemma}

\begin{lemma}\label{lem:sample eigen}
Suppose that $w\in(0,1]$, and $\alpha\in(0,0.5]$ are constants. Choose any $k\leq n = \lfloor wN\rfloor$ satisfying $k\gg\log(N)$. Then, there exists a positive constant $\Lambda_0>0$ such that the event
\begin{align}\label{event B}
\Bsc:=\left\{	\inf_{\bx\in[0,1]^d,\xi\in\Xi}\Lambda_{\min}\left(\bS_L-\bd_L \bd_L^{\top}\right)\geq\Lambda_0\right\}
\end{align}
satisfies $\lim_{N\to \infty}\P_{\S_\Isc}\left(\Bsc\right)=1$, where $\bd_L$ and $\bS_L$ are defined as in \eqref{def:d_S}. In addition, on the event $\Bsc$, the matrices $\bS_L-\bd_L \bd_L^{\top}$ and $\bS_L$ are both positive-definite, and we also have
\begin{align}\label{bound: leaf distibution}
\sup_{\bx\in[0,1]^d,\xi\in\Xi}\bd_L^{\top}(\bS_L-\bd_L \bd_L^{\top})^{-1}\bd_L\leq\frac{\dbar}{\Lambda_0}.
\end{align}
\end{lemma}

\begin{lemma}\label{lem:p.s.d}
Let the assumptions in Lemma \ref{lem:sample eigen} hold. Define the event 
\begin{align}\label{event C}
\Csc:=\left\{\mathrm{diam}_j(L(\bx,\xi)) \neq 0,\;\;\text{for all}\;\;1\leq j\leq d,\;\;\bx\in[0,1]^d,\;\;\xi\in\Xi\right\}.
\end{align}
Then, we have $\P_{\S_\Isc}\left(\Csc\right)=1$. Moreover, on the event $\Bsc \cap \Csc $, we have $\bS$, $\sum_{i\in\Isc}\omega_i(\bx,\xi)\bDelta_i\bDelta_i^\top$ and $\sum_{i\in\Isc}\omega_i(\bx,\xi)\bG(\bX_i)\bG(\bX_i)^{\top}$ are both positive-definite, where $\bDelta_i=\bG(\bX_i-\bx)$.
\end{lemma}

In the following, we consider the average treatment effect (ATE) estimation problem and that the proposed forests provide stable propensity score estimates that are away from zero and one with high probability.

\begin{lemma}\label{lem:pihat}
Let Assumptions \ref{identification}(c) hold and $\pi^*\in \mathcal{H}^{q_2,\beta_2}$, where $q_2\in \mathbb{N}$ and $\beta_2\in(0,1]$. Let $M>0$, $w_2\in(0,1]$, and $\alpha_2\in(0,0.5]$ be constants. Choose any $B\geq1$ and $k_2\gg\log^3(N)$. Then, as $N\to\infty$,
\begin{equation}\label{bound:pihat}
\P_{\bX}(c_1<\pihat^{-k}(\bX)\leq 1-c_1)=1,\;\;\mbox{for each}\;\;k\leq K,
\end{equation}
with probability approaching one and some constant $c_1\in(0,1/2)$. Note that the left-hand-side of \eqref{bound:pihat} is a random quantity as the probability is only taken with respect to a new observation $\bX$.
\end{lemma}

Lemma \ref{lem:pihat} demonstrates the stability of the inverse PS estimates, a requirement often assumed in the context of non-parametric nuisance estimates, as discussed in \cite{chernozhukov2017double}. The above results suggest that, under the assumption of overlap, there is typically no necessity to employ any form of trimming or truncation techniques on the estimated propensities, provided the chosen tuning parameter $k_2$ is not too small. In fact, the parameter $k$ can also be viewed as a truncation parameter as it avoids the occurrence of propensity score estimates close to zero or one with high probability.

\section{Proofs of the results for the adaptive split balancing forests}

\begin{proof}[Proof of Lemma \ref{lem:balance diam}]
For any $\bx\in[0,1]^d$ and $\xi\in\Xi$, let $c(\bx,\xi)$ be the number of splits leading to the leaf $L(\bx,\xi)$, and let $c_j(\bx,\xi)$ be the number of such splits along the $j$-th coordinate. Define $t=\min_{1\leq j\leq d}c_j(\bx,\xi)$. By the balanced splitting rule, we know that the number of splits along different coordinates differs by at most one. That is, $c_j(\bx,\xi)\in\{t,t+1\}$ for all $1\leq j\leq d$. Since $c(\bx,\xi)=\sum_{j=1}^d c_j(\bx,\xi) $, $c(\bx,\xi)$ can be express as $c(\bx,\xi)=td+l$, with $l=c(\bx,\xi)-td\in\{0,1,\dots,d-1\}$ denoting the number of splits in last round if $l\neq0$. 
Let $L_0(\bx,\xi)\supseteq L_1(\bx,\xi)\supseteq ...\supseteq L_{c(\bx,\xi)}(\bx,\xi)$ be the successive nodes leading to $L(\bx,\xi)$, where $L_0(\bx,\xi)=[0,1]^d$ and $L_{c(\bx,\xi)}(\bx,\xi)=L(\bx,\xi)$. Let $n_0, n_1,\dots,n_{c(\bx,\xi)}$ be random variables denoting the number of points in $\S_\Isc$ located within the the successive nodes $L_0(\bx,\xi), L_1(\bx,\xi),..., L_{c(\bx,\xi)}(\bx,\xi)$, where $n_0=\lfloor wN\rfloor$. Since the tree is $(\alpha,k)$-regular, we know that $\alpha n_{i-1} \leq n_i\leq (1-\alpha) n_{i-1}$ for each $1\leq i\leq td+l$, and hence the following deterministic inequalities hold:
\begin{align}
&\alpha^{i}n_0\leq n_i\leq (1-\alpha)^i n_0,\label{balance diam_eq1}\\
&\alpha^{td+l-i}n_i\leq n_{td+l}\leq (1-\alpha)^{td+l-i} n_i.\label{balance diam_eq2}
\end{align}
It follows that $\alpha^{td+l} \lfloor wN\rfloor \leq n_{td+l} \leq (1-\alpha)^{td+l} \lfloor wN\rfloor$. Moreover, note that $n_{td+l}\in[k,2k-1]$. Hence, we have $k\leq (1-\alpha)^{td+l} \lfloor wN\rfloor$ and $ \alpha^{td+l} \lfloor wN\rfloor\leq 2k-1$, which implies that
\begin{align}
T:=\left\lceil\frac{\log((2k-1)/\lfloor wN\rfloor)}{d\log(\alpha)}-1\right\rceil\leq\frac{\log((2k-1)/\lfloor wN\rfloor)}{d\log(\alpha)}-\frac{l}{d}\leq t\leq \frac{\log(k/\lfloor wN\rfloor)}{d\log(1-\alpha)}-\frac{l}{d}.\label{diam_eq7}
\end{align}
Note that although $c(\bx,\xi)$, $c_j(\bx,\xi)$, and $t$ are random variables, the balanced approach leads to a non-random lower bound $T$ for the random quantity $t$.

For any $1 \leq j \leq d$ and leaf $L \subset [0,1]^d$, let $\mathrm{diam}_j(L)$ be the length of the longest segment parallel to the $j$-th axis that is a subset of $L$, and let $c_j(L)$ be the number of times the leaf $L$ has been split along the $j$-th coordinate for any $1 \leq m \leq c(\bx,\xi)=td+l$. Define
\[
k_{i,j}:=\min\Bigl\{m: c_j(L_m(\bx,\xi))=i \Bigl\},
\]
i.e., $k_{i,j}$ represents the total number of splits the leaf has undergone after its $j$-th coordinate has been divided $i$ times. Note that $c_j(L_m(\bx,\xi))$ is non-decreasing as $m$ grows.
 
Based on the balanced splitting rule, for any $1\leq i\leq t$ and $1\leq j \leq d$, we have $c_j(L_{(i-1)d}(\bx,\xi))=i-1$, $c_j(L_{(i-1)d+j}(\bx,\xi))\in\{i-1,i\}$, and hence
$$(i-1)d+1\leq k_{i,j}\leq id.$$
As shown in the proof of Lemma 6.1 in \cite{duroux2018impact}, $\mathrm{diam}_j(L(\bx,\xi))$ has the same distribution as the product of independent Beta random variables:
$$\mathrm{diam}_j(L(\bx,\xi))\overset{d}{=}\prod_{m=1}^{c(\bx,\xi)}[\mathrm{Beta}(n_m+1,n_{m-1}-n_m)]^{\delta_{j,m}(\bx,\xi)},$$
where $\mathrm{Beta}(\alpha,\beta)$ represents a Beta random variable with parameters $\alpha$ and $\beta$, and the indicator $\delta_{j,m}(\bx,\xi)$ is defined such that it equals $1$ if the last split performed to form the leaf $L_m(\bx,\xi)$ was along the $j$-th coordinate (and $0$ otherwise). For each $j\leq d$, by the definition of $k_{i,j}$, we have $\delta_{j,m}(\bx,\xi)=1$ if $m\in\{k_{i,j}:i\leq c_j(\bx,\xi)\}$, and $\delta_{j,m}(\bx,\xi)=0$ otherwise. Therefore,
\begin{align}\label{C.3}
\mathrm{diam}_j(L(\bx,\xi))\overset{d}{=}\prod_{i=1}^{c_j(\bx,\xi)}B_{i,j},
\end{align}
with $B_{i,j}\sim\mathrm{Beta}(n_{k_{i,j}}+1,n_{k_{i,j}-1}-n_{k_{i,j}})$. Note that $c_j(\bx,\xi)\geq T$ and $B_{i,j}\in[0,1]$, we have
\begin{align*}
	\E_{\S_\Isc}\left[\mathrm{diam}_j^{r}(L(\bx,\xi))\right]\leq\E_{\S_\Isc}\left[\prod_{i=1}^{T} B_{i,j}^r\right]=\prod_{i=1}^{T} \E_{\S_\Isc}\left[B_{i,j}^r\right],
\end{align*}
since the Beta random variables are independent.
For any $1\leq j\leq d$, $1\leq i \leq T$ and $r\geq 1$, note that
\begin{align*}
	\E_{\S_\Isc}\left[B_{i,j}^r\right]=\frac{\mathrm{Beta}(n_{k_{i,j}}+1+r,n_{k_{i,j}-1}-n_{k_{i,j}})}{\mathrm{Beta}(n_{k_{i,j}}+1,n_{k_{i,j}-1}-n_{k_{i,j}})}=\frac{\Gamma(n_{k_{i,j}}+1+r)\Gamma(n_{k_{i,j}-1}+1)}{\Gamma(n_{k_{i,j}-1}+1+r)\Gamma(n_{k_{i,j}}+1)},
\end{align*}
since $\mathrm{Beta}(\alpha,\beta)=\Gamma(\alpha)\Gamma(\beta)/\Gamma(\alpha+\beta)$, where $\Gamma(z)$ represents a Gamma random variable with parameter $z$. Let $s_r=r-\lfloor r \rfloor$. By $\Gamma(z+1)=z\Gamma(z)$, we have
\begin{align*}
	\E_{\S_\Isc}\left[B_{i,j}^r\right]=\frac{(n_{k_{i,j}}+r)\cdots(n_{k_{i,j}}+s_r+1)\Gamma(n_{k_{i,j}}+s_r+1)\Gamma(n_{k_{i,j}-1}+1)}{(n_{k_{i,j}-1}+r)\cdots(n_{k_{i,j}-1}+s_r+1)\Gamma(n_{k_{i,j}-1}+s_r+1)\Gamma(n_{k_{i,j}}+1)}.
\end{align*}
By inequality (7) of \cite{gautschi1959some}, we have
\begin{align*}
\Gamma(n_{k_{i,j}}+s_r+1)&\leq (n_{k_{i,j}}+1)^{s_r}\Gamma(n_{k_{i,j}}+1),\\
\Gamma(n_{k_{i,j}-1}+s_r+1)&\geq \frac{n_{k_{i,j}-1}+1}{(n_{k_{i,j}-1}+s_r+1)^{1-s_r}}\Gamma(n_{k_{i,j}-1}+1),
\end{align*}
which implies
\begin{align*}
	\E_{\S_\Isc}\left[B_{i,j}^r\right]&\leq\frac{(n_{k_{i,j}}+r)\cdots(n_{k_{i,j}}+s_r+1) (n_{k_{i,j}}+1)^{s_r}(n_{k_{i,j}-1}+s_r+1)}{(n_{k_{i,j}-1}+r)\cdots(n_{k_{i,j}-1}+s_r+1)(n_{k_{i,j}-1}+s_r+1)^{s_r}(n_{k_{i,j}-1}+1)}\\
	&\leq\frac{(n_{k_{i,j}}+r)\cdots(n_{k_{i,j}}+s_r+1) (n_{k_{i,j}}+1)^{s_r}(n_{k_{i,j}-1}+s_r+1)}{(n_{k_{i,j}-1}+r)\cdots(n_{k_{i,j}-1}+s_r+1)(n_{k_{i,j}-1}+1)^{s_r}(n_{k_{i,j}-1}+1)},
\end{align*}
since $n_{k_{i,j}}+1\leq n_{k_{i,j}}+s_r+1 $.
By the $(\alpha,k)$-regular property, we have $n_{k_{i,j}}\leq (1-\alpha)n_{k_{i,j}-1}$. Then, for any  $q\leq \lceil r \rceil$,
\begin{align*}
     \frac{n_{k_{i,j}}+q}{n_{k_{i,j}-1}+q}\leq  \frac{(1-\alpha)n_{k_{i,j}-1}+q}{n_{k_{i,j}-1}+q}=(1-\alpha)\left(1+\frac{q\alpha/(1-\alpha)}{n_{k_{i,j}-1}+q}\right)\leq (1-\alpha)\left(1+\frac{\lceil r \rceil \alpha/(1-\alpha)}{n_{k_{i,j}-1}}\right).
\end{align*}
Hence,
\begin{align*}
	&\E_{\S_\Isc}\left[B_{i,j}^r\right]\leq (1-\alpha)^r\left(1+\frac{\lceil r \rceil\alpha/(1-\alpha)}{n_{k_{i,j}-1}}\right)^r\left(1+\frac{s_r}{n_{k_{i,j}-1}}\right).
\end{align*}
Since $k_{i,j} \leq id$ for any $1\leq j\leq d$ and $1\leq i \leq T$, we have $n_{k_{i,j}-1}\geq n_{id-1}$ and hence
\begin{align*}
	&\E_{\S_\Isc}\left[B_{i,j}^r\right]\leq (1-\alpha)^r\left(1+\frac{\lceil r \rceil\alpha/(1-\alpha)}{n_{id-1}}\right)^r\left(1+\frac{s_r}{n_{id-1}}\right).
\end{align*}
It follows that for any $\bx\in[0,1]^d$, $\xi\in\Xi$, and $1\leq j\leq d$,
\begin{align}	
	&\E_{\S_\Isc}\left[\mathrm{diam}_j^{r}(L(\bx,\xi))\right]\leq (1-\alpha)^{Tr}\prod_{i=1}^{T} \left(1+\frac{\lceil r \rceil\alpha/(1-\alpha)}{n_{id-1}}\right)^r\left(1+\frac{s_r}{n_{id-1}}\right).\label{diam_eq2}
\end{align}
Since $\log(1+x)< x$ for all $x>0$, we have
\begin{align*}
	&\log\left(\prod_{i=1}^{T} \left(1+\frac{\lceil r \rceil\alpha/(1-\alpha)}{n_{id-1}}\right)^r\left(1+\frac{s_r}{n_{id-1}}\right)\right)\\
	&\qquad=r \sum_{i=1}^T\log\left(1+\frac{\lceil r \rceil\alpha/(1-\alpha)}{n_{id-1}}\right)+\sum_{i=1}^T\log\left(1+\frac{s_r}{n_{id-1}}\right)\\
	&\qquad< r \lceil r \rceil \sum_{i=1}^T \frac{\alpha/(1-\alpha)}{n_{id-1}}+ \sum_{i=1}^T\frac{s_r}{n_{id-1}}.
\end{align*}
By  \eqref{balance diam_eq2}, we have $1/n_{id-1}\leq (1-\alpha)^{(T-i)d+1}/n_{Td}$ for each $i\leq T$. Then, by $\alpha/(1-\alpha)\leq 1$,
\begin{align*}
	&\log\left(\prod_{i=1}^{T} \left(1+\frac{\lceil r \rceil\alpha/(1-\alpha)}{n_{id-1}}\right)^r\left(1+\frac{s_r}{n_{id-1}}\right)\right)< \frac{r\lceil r \rceil\alpha+s_r(1-\alpha)}{n_{Td}}\sum_{i=1}^T (1-\alpha)^{(T-i)d}\\
	&\qquad=\frac{r\lceil r \rceil\alpha+s_r(1-\alpha)}{n_{Td}}\cdot\frac{1-(1-\alpha)^{Td}}{1-(1-\alpha)^{d}}.
\end{align*}
By $t,d>0$ and $\alpha\in(0,0.5]$, we have $\alpha/(1-(1-\alpha)^d)\leq\alpha/(1-(1-\alpha))=1$ and $(1-\alpha)/(1-(1-\alpha)^d)\leq(1-\alpha)/(1-(1-\alpha))\leq1$. Hence,
\begin{align*}
	\log\left(\prod_{i=1}^{T} \left(1+\frac{\lceil r \rceil\alpha/(1-\alpha)}{n_{id-1}+\lceil r \rceil}\right)^{r}\right)<\frac{r\lceil r \rceil+s_r}{n_{Td}}\leq r^2+r+1.
\end{align*}
Together with \eqref{diam_eq2}, for any $\bx\in[0,1]^d$, $\xi\in\Xi$, and $1\leq j\leq d$,
\begin{align}
\E_{\S_\Isc}\left[\mathrm{diam}_j^{r}(L(\bx,\xi))\right]
< (1-\alpha)^{Tr}\exp(r^2+r+1).\label{diam_eq8}
\end{align}
By definition, $T\geq\frac{\log((2k-1)/\lfloor wN\rfloor)}{d\log(\alpha)}-1$. Hence,
\begin{align*}
(1-\alpha)^{Tr}\leq (1-\alpha)^{-r}\left(\frac{\lfloor wN\rfloor}{2k-1}\right)^{-\frac{r\log (1-\alpha)}{d \log (\alpha)}}\leq 2^r \left(\frac{\lfloor wN\rfloor}{2k-1}\right)^{-\frac{r\log (1-\alpha)}{d \log (\alpha)}}.
\end{align*}
By \eqref{diam_eq8}, for any $\bx\in[0,1]^d$, $\xi\in\Xi$, and $1\leq j\leq d$,
\begin{align}
\E_{\S_\Isc}\left[\mathrm{diam}_j^{r}(L(\bx,\xi))\right]
&<  2^r \left(\frac{\lfloor wN\rfloor}{2k-1}\right)^{-\frac{r\log (1-\alpha)}{d \log (\alpha)}}\exp(r^2+r+1).\label{diam_eq5}
\end{align}
By the finite form of Jensen's inequality, when $r\geq2$,
\begin{align*}
\left(\frac{\sum_{j=1}^d\mathrm{diam}_j^2(L(\bx,\xi))}{d}\right)^{r/2}&\leq\frac{\sum_{j=1}^d\mathrm{diam}_j^r(L(\bx,\xi))}{d},
\end{align*}
which implies that
\begin{align*}
\E_{\S_\Isc}\left[\mathrm{diam}^r(L(\bx,\xi))\right]&= \E_{\S_\Isc}\left[\sum_{j=1}^d\mathrm{diam}_j^2(L(\bx,\xi))\right]^{r/2}\leq d^{(r-2)/2}\E_{\S_\Isc}\left[\sum_{j=1}^d\mathrm{diam}_j^r(L(\bx,\xi))\right].
\end{align*}
Together with \eqref{diam_eq5}, for any $r\geq2$,
\begin{align*}
\E_{\S_\Isc}\left[\mathrm{diam}^r(L(\bx,\xi))\right]&<2^r d^{r/2}\exp(r^2+r+1)\left(\frac{\lfloor wN\rfloor}{2k-1}\right)^{-\frac{r\log (1-\alpha)}{d \log (\alpha)}}.
\end{align*}
On the other hand, when $r<2$, we also have
\begin{align*}
&\E_{\S_\Isc}\left[\mathrm{diam}^r(L(\bx,\xi))\right]=\E_{\S_\Isc}\left[\sum_{j=1}^d\mathrm{diam}_j^2(L(\bx,\xi))\right]^{r/2}\overset{(i)}{\leq}\E_{\S_\Isc}\left[\sum_{j=1}^d\mathrm{diam}_j^r(L(\bx,\xi))\right]\\
&\qquad\overset{(ii)}{<}2^r d\exp(r^2+r+1)\left(\frac{\lfloor wN\rfloor}{2k-1}\right)^{-\frac{r\log (1-\alpha)}{d \log (\alpha)}},
\end{align*}
where (i) holds since $\|\bx\|_2\leq\|\bx\|_r$ for any vector $\bx\in\R^d$ and $r<2$; (ii) holds by \eqref{diam_eq5}. To sum up, for any $r\geq 1$, $\bx\in[0,1]^d$, and $\xi\in\Xi$,
$$\E_{\S_\Isc}\left[\mathrm{diam}^r(L(\bx,\xi))\right]< 2^r d^{\max\{r/2,1\}}\exp(r^2+r+1)\left(\frac{\lfloor wN\rfloor}{2k-1}\right)^{-\frac{r\log (1-\alpha)}{d \log (\alpha)}}.$$
\end{proof}

\begin{proof}[Proof of Theorem \ref{thm:balance_consistency}]
By Jensen's inequality and the fact that $(a-b)^2\leq2a^2+2b^2$ for any $a,b\in\R$,
\begin{align}
	&\E_{\bx}\left[\mhat(\bx)-m(\bx)\right]^2= \E_{\bx}\left[ \E_{\xi}\left[\sum_{i\in\Isc}\omega_i(\bx,\xi)(Y_i-m(\bx))\right]\right]^2\nonumber\\
	&\qquad\leq \E_{\bx}\left[ \E_{\xi}\left[\sum_{i\in\Isc}\omega_i(\bx,\xi)(Y_i-m(\bx))\right]^2\right]
 \leq 2\E_{\bx}\left[T_1(\bx) \right]+ 2\E_{\bx}\left[T_2(\bx) \right],\label{thm:balance_eq1}
\end{align} 
where for any $\bx\in[0,1]^d$, 
\begin{align}
	T_1(\bx)&:=\E_{\xi}\left[\sum_{i\in\Isc}\omega_i(\bx,\xi)\varepsilon_i\right]^2, \label{T_1}\\
	T_2(\bx)&:= \E_{\xi}\left[\sum_{i\in\Isc}\omega_i(\bx,\xi)\left(m(\bX_i)-m(\bx)\right)\right]^2,\label{T_2}
\end{align}
with $\varepsilon_i=Y_i-m(\bX_i)$. By Fubini's theorem,
\begin{align*}
	\E_{\S_\Isc}\left[\E_{\bx}\left[T_1(\bx) \right]\right]&=\E_{\xi}\left[\E_{\bx}\left[\E_{\S_\Isc}\left[\sum_{i\in\Isc}\omega_i(\bx,\xi)\varepsilon_i\right]^2\right]\right].
\end{align*}
Note that 
\begin{align*}
	\E_{\S_\Isc}\left[\sum_{i\in\Isc}\omega_i(\bx,\xi)\varepsilon_i\right]^2&=\E_{\S_\Isc}\left[\sum_{i\in\Isc}\left[\omega_i(\bx,\xi)\right]^2 \varepsilon_i^2\right]+\E_{\S_\Isc}\left[\sum_{i,j\in\Isc,i\neq j}\omega_i(\bx,\xi)\omega_j(\bx,\xi)\varepsilon_i\varepsilon_j \right].
\end{align*}
For any $i,j\in\Isc$ with $i\neq j$,
\begin{align*}
	&\E_{\S_\Isc}\left[\omega_i(\bx,\xi)\omega_j(\bx,\xi)\varepsilon_i\varepsilon_j \right]
	\overset{(i)}{=}\E_{\S_\Isc}\left[\omega_i(\bx,\xi)\omega_j(\bx,\xi)\E_{\S_\Isc}\left[\varepsilon_i\varepsilon_j \mid \{\bX_l\}_{l=1}^N, \{Y_l\}_{l\in\Jsc}\right]\right]\\
	&\qquad\overset{(ii)}{=}\E_{\S_\Isc}\left[\omega_i(\bx,\xi)\omega_j(\bx,\xi)\E_{\S_\Isc}\left[\varepsilon_i \mid \bX_i \right]\E_{\S_\Isc}\left[\varepsilon_j \mid \bX_j \right]\right]\overset{(iii)}{=}0,
\end{align*}
where (i) holds by the tower rule and ``honesty" of the forests; (ii) holds by the independence of the samples; (iii) holds since $\E[\varepsilon\mid\bX]=0$. Therefore, we have
\begin{align*}
		\E_{\S_\Isc}\left[\E_{\bx}\left[T_1(\bx) \right]\right]=\E_{\xi}\left[\E_{\bx}\left[\E_{\S_\Isc}\left[\sum_{i\in\Isc}\left[\omega_i(\bx,\xi)\right]^2 \varepsilon_i^2\right]\right]\right].
\end{align*}
By the tower rule,
\begin{align*}
	&\E_{\xi}\left[\E_{\bx}\left[\E_{\S_\Isc}\left[\sum_{i\in\Isc}\left[\omega_i(\bx,\xi)\right]^2 \varepsilon_i^2\right]\right]\right]\\
	&\qquad\overset{(i)}{=}\E_{\xi}\left[\E_{\bx}\left[\E_{\S_\Isc}\left[\sum_{i\in\Isc}\left[\omega_i(\bx,\xi)\right]^2 \E_{\S_\Isc}\left[\varepsilon_i^2\mid \{\bX_l\}_{l=1}^N, \{Y_l\}_{l\in\Jsc} \right]\right]\right]\right]\\
	&\qquad\overset{(ii)}{=}\E_{\xi}\left[\E_{\bx}\left[\E_{\S_\Isc}\left[\sum_{i\in\Isc}\left[\omega_i(\bx,\xi)\right]^2 \E_{\S_\Isc}\left[\varepsilon_i^2\mid \bX_i \right]\right]\right]\right]\\
	&\qquad\overset{(iii)}{\leq} M \E_{\xi}\left[\E_{\bx}\left[\E_{\S_\Isc}\left[\sum_{i\in\Isc}\left[\omega_i(\bx,\xi)\right]^2 \right]\right]\right],
\end{align*}
where (i) holds by ``honesty" of the forests; (ii) holds by the independence of the samples; (iii) holds by Assumption \ref{cond:noise}. Therefore, we have
\begin{align*}
	\E_{\S_\Isc}\left[\E_{\bx}\left[T_1(\bx) \right]\right]\leq M \E_{\xi}\left[\E_{\bx}\left[\E_{\S_\Isc}\left[\sum_{i\in\Isc}\left[\omega_i(\bx,\xi)\right]^2 \right]\right]\right].
\end{align*}
Since $\left(\mathbbm{1}_{\left \{\bX_i\in L(\bx,\xi)\right\}}\right)^2=\mathbbm{1}_{\left \{\bX_i\in L(\bx,\xi)\right\}}$, we have 
\begin{align}
	\omega_i^2(\bx,\xi)=\frac{\omega_i(\bx,\xi)}{\# \left \{ l:\bX_l\in L(\bx,\xi) \right \}}\overset{(i)}{\leq}\frac{\omega_i(\bx,\xi)}{k},\label{consistency:eq3}
\end{align}
where (i) holds by $(\alpha,k)$-regular.
By $\sum_{i\in\Isc}\omega_i(\bx,\xi)=1$, we have $\E_{\S_\Isc}\left[\E_{\bx}\left[T_1(\bx) \right]\right]\leq M/k.$ By Markov's inequality, as $N\to\infty$, we have 
\begin{align}
	\E_{\bx}\left[T_1(\bx) \right]=O_p\left( \frac{1}{k}\right).\label{thm:balance_eq2}
\end{align}
Additionally, note that $\E_{\bx}\left[T_2(\bx) \right]=\E_{\bx}\left[\E_{\xi}\left[ \sum_{i\in\Isc}\omega_i(\bx,\xi)\left(m(\bX_i)-m(\bx)\right) \right]^2\right]$. By Cauchy-Schwarz inequality and the fact that $\sum_{i\in\Isc}\omega_i(\bx,\xi)=1$, 
\begin{align*}
		&\left[ \sum_{i\in\Isc}\omega_i(\bx,\xi)\left(m(\bX_i)-m(\bx)\right) \right]^2=\left[ \sum_{i\in\Isc}\sqrt{\omega_i(\bx,\xi)} \left[\sqrt{\omega_i(\bx,\xi)}\left(m(\bX_i)-m(\bx)\right)\right] \right]^2\\
		&\qquad\leq \left[ \sum_{i\in\Isc}\left(\sqrt{\omega_i(\bx,\xi)}\right)^2 \right] \left[ \sum_{i\in\Isc}\left(\sqrt{\omega_i(\bx,\xi)}\left(m(\bX_i)-m(\bx)\right)\right)^2 \right]\\
		&\qquad\leq\left[ \sum_{i\in\Isc}\omega_i(\bx,\xi) \right] \left[ \sum_{i\in\Isc}\omega_i(\bx,\xi)\left(m(\bX_i)-m(\bx)\right)^2 \right]\\
	&\qquad= \sum_{i\in\Isc}\omega_i(\bx,\xi)\left(m(\bX_i)-m(\bx)\right)^2.
\end{align*} 
Then, we have
\begin{align*}
	\E_{\bx}\left[T_2(\bx) \right]\leq \E_{\bx}\left[\E_{\xi}\left[ \sum_{i\in\Isc}\omega_i(\bx,\xi)\left(m(\bX_i)-m(\bx)\right)^2\right] \right].
\end{align*}
By the Lipschitz continuity of $m(\cdot)$, we have
\begin{align*}
 \E_{\bx}\left[\E_{\xi}\left[ \sum_{i\in\Isc}\omega_i(\bx,\xi)\left(m(\bX_i)-m(\bx)\right)^2\right] \right]
	&\leq \E_{\bx}\left[\E_{\xi}\left[ \sum_{i\in\Isc}\omega_i(\bx,\xi)\left( L_0\|\bX_i-\bx\|\right)^2\right] \right]\\
	&\leq L_0^2 \E_{\bx}\left[\E_{\xi}\left[ \sum_{i\in\Isc}\omega_i(\bx,\xi)\mathrm{diam}^2(L(\bx,\xi))\right] \right],
\end{align*}
where $L_0$ is the Lipschitz constant. 
Then, we have
\begin{align*}
	\E_{\bx}\left[T_2(\bx) \right]\leq L_0^2\E_{\bx}\left[\E_{\xi}\left[\sum_{i\in\Isc}\omega_i(\bx,\xi)\mathrm{diam}^2(L(\bx,\xi))\right] \right]
	\overset{(i)}{=}L_0^2\E_{\bx}\left[ \E_{\xi}\left[\mathrm{diam}^2(L(\bx,\xi))\right] \right],
\end{align*} 
where (i) holds by $\sum_{i\in\Isc}\omega_i(\bx,\xi)=1$.
By Fubini's theorem,
\begin{align*}
	&\E_{\S_\Isc}\left[\E_{\bx}\left[T_2(\bx) \right]\right]\leq L_0^2\E_{\S_\Isc}\left[\E_{\bx}\left[ \E_{\xi}\left[\mathrm{diam}^2(L(\bx,\xi))\right] \right]\right]
	=L_0^2\E_{\xi}\left[\E_{\bx}\left[\E_{\S_\Isc}\left[ \mathrm{diam}^2(L(\bx,\xi))\right]\right]\right].
\end{align*} 
By Lemma \ref{lem:balance diam},
\begin{align*}
	\E_{\xi}\left[\E_{\bx}\left[\E_{\S_\Isc}\left[ \mathrm{diam}^2(L(\bx,\xi))\right]\right]\right]&\leq \sup_{\bx\in [0,1]^d,\xi\in\Xi}\E_{\S_\Isc}\left[ \mathrm{diam}^2(L(\bx,\xi))\right]\\
	&< 4d\exp(7) \left(\frac{\lfloor wN\rfloor}{2k-1}\right)^{-\frac{2\log (1-\alpha)}{d \log (\alpha)}}.
\end{align*}
Therefore, we have
\begin{align*}
	\E_{\S_\Isc}\left[\E_{\bx}\left[T_2(\bx) \right]\right]< L_0^2d\exp(4) \left(\frac{\lfloor wN\rfloor}{2k-1}\right)^{-\frac{2\log (1-\alpha)}{d \log (\alpha)}}.
\end{align*}
By Markov's inequality, as $N\to\infty$, we have 
\begin{align}
	\E_{\bx}\left[T_2(\bx) \right]=O_p\left(\left(\frac{N}{k}\right)^{-\frac{2\log (1-\alpha)}{d \log (\alpha)}}\right).\label{thm:balance_eq3}
\end{align}
Combining \eqref{thm:balance_eq1}, \eqref{thm:balance_eq2}, and \eqref{thm:balance_eq3}, we conclude that
$$\E_{\bx}\left[\mhat(\bx)-m(\bx)\right]^2=O_p\left(\frac{1}{k}+ \left(\frac{k}{N}\right)^{\frac{2\log(1-\alpha)}{d \log(\alpha)}} \right).$$
\end{proof}

\section{Proofs of the results for the localized forests}
\begin{proof}[Proof of Theorem \ref{thm:local_consistency}]
Recall the definition of $\mhat_\mathrm{LCF}(\bx)$,
\begin{align}
	&\E_{\bx}\left[\mhat_\mathrm{LCF}(\bx)-m(\bx)\right]^2= \E_{\bx}\left[
	\E_{\xi}\left[\bG(\bx)^\top\left(\bbetahat(\bx,\xi)-\bbeta\right)\right]\right]^2\nonumber\\
	&\qquad\overset{(i)}{\leq} \E_{\bx}\left[\E_{\xi}\left[\bG(\bx)^\top\left(\bbetahat(\bx,\xi)-\bbeta\right)\right]^2\right]\overset{(ii)}{=}\E_{\xi}\left[\E_{\bx}\left[\bG(\bx)^\top\left(\bbetahat(\bx,\xi)-\bbeta\right)\right]^2\right],\label{LCF_eq1}
\end{align}
where (i) holds by Jensen's inequality and (ii) holds by Fubini's theorem. 
In the following, we condition on the event $\Bsc\cap\Csc$ defined as \eqref{event B} and \eqref{event C}. By Lemmas \ref{lem:sample eigen} and \ref{lem:p.s.d}, we know that $\bS_L-\bd_L\bd_L^\top$, $\bS_L$, $\bS$, $\sum_{i\in\Isc}\omega_i(\bx,\xi)\bDelta_i\bDelta_i^\top$ and $\sum_{i\in\Isc}\omega_i(\bx,\xi)\bG(\bX_i)\bG(\bX_i)^{\top}$ are all positive-definite, with $\P_{\S_\Isc}\left(\Bsc\cap\Csc\right)=1-o(1)$.
Recall the definition of $\bbetahat(\bx,\xi)$, \eqref{def:gammahat}, 
\begin{align*}
	\bbetahat(\bx,\xi)&=\left(\sum_{i\in\Isc}\omega_i(\bx,\xi)\bG(\bX_i)\bG(\bX_i)^{\top}\right)^{-1}\left(\sum_{i\in\Isc}\omega_i(\bx,\xi)\bG(\bX_i)Y_i\right).
\end{align*}
Let $\bgamma:=(\bgamma_1,\bgamma_2,\dots,\bgamma_d)$ be the multi-index, where each $\bgamma_i$ is a nonnegative integer. Define $r_i=Y_i-\bG(\bX_i)^{\top}\bbeta-\varepsilon_i$ with $\bG(\bX_i)^{\top}\bbeta=\sum_{|\bgamma|=0}^{q}D^{\bgamma}m(\bx)(\bX-\bx)^{\bgamma}/\bgamma!$. Then, we have
\begin{align*}
	\bbetahat(\bx,\xi)-\bbeta=\left(\sum_{i\in\Isc}\omega_i(\bx,\xi)\bG(\bX_i)\bG(\bX_i)^{\top}\right)^{-1}\left(\sum_{i\in\Isc}\omega_i(\bx,\xi)\bG(\bX_i)(\varepsilon_i+r_i)\right).
\end{align*}
Note that there exists some $\dbar\times \dbar$ lower triangular matrix $\bT$ with $1$ on main diagonal such that
\begin{align}\label{def:bT}
	\bG(\bX_i-\bx)=\bT \bG(\bX_i),
\end{align}
which implies
\begin{align*}
	\bbetahat(\bx,\xi)-\bbeta=\bT^\top\left(\sum_{i\in\Isc}\omega_i(\bx,\xi)\bG(\bX_i-\bx)\bG(\bX_i-\bx)^{\top}\right)^{-1}\left(\sum_{i\in\Isc}\omega_i(\bx,\xi)\bG(\bX_i-\bx)(\varepsilon_i+r_i)\right).
\end{align*}
To simplify the exposition, we let $\bDelta_i:= \bG(\bX_i-\bx)$. By $\bT \bG(\bx)=\bG(\bzero)=\be_1$,
\begin{align}
	&\bG(\bx)^\top\left(\bbetahat(\bx,\xi)-\bbeta\right)=\be_1^{\top} \left(\sum_{i\in\Isc}\omega_i(\bx,\xi)\bDelta_i\bDelta_i^{\top}\right)^{-1}\left(\sum_{i\in\Isc}\omega_i(\bx,\xi)\bDelta_i(\varepsilon_i+r_i)\right).\label{eq_2}
\end{align}
By \eqref{LCF_eq1}, we have
\begin{align*}
	&\E_{\bx}\left[\mhat_\mathrm{LCF}(\bx)-m(\bx)\right]^2\leq \E_{\xi}\left[\E_{\bx}\left[\be_1^{\top} \left(\sum_{i\in\Isc}\omega_i(\bx,\xi)\bDelta_i\bDelta_i^{\top}\right)^{-1}\left(\sum_{i\in\Isc}\omega_i(\bx,\xi)\bDelta_i(\varepsilon_i+r_i)\right) \right]^2\right].
\end{align*}
Define $\bU_i:=\left(\bg_1(\bX_i-\bx)^\top,\bg_2(\bX_i-\bx)^\top,\dots,\bg_q(\bX_i-\bx)^\top\right)^\top =(Z_{i1},\dots,Z_{id},Z_{i1}^2,Z_{i1}Z_{i2},\dots,$ $Z_{id}^2,\dots,Z_{id}^q)^\top\in\R^{\dbar}$ with $Z_{ij}:=\bX_{ij}-\bx_j$ for any $i\in\mathcal I$ and $j \leq d$, and $\dbar=\sum_{i=1}^q d^i$. Since $\bDelta_i=(1,\bU_i^{\top})^{\top}$, we have
\begin{align*}
	\sum_{i\in\Isc}\omega_i(\bx,\xi)\bDelta_i\bDelta_i^{\top}=
	\begin{pmatrix}
		1 & \bd^{\top}\\
		\bd & \bS
	\end{pmatrix},
\end{align*}
where $\bd:=\sum_{i \in \Isc}\omega_i(\bx,\xi)\bU_i$ and $\bS:=\sum_{i \in \Isc}\omega_i(\bx,\xi)\bU_i\bU_i^{\top}$.
By Schur decomposition,
\begin{align}
	\be_1^{\top} \left(\sum_{i\in\Isc}\omega_i(\bx,\xi)\bDelta_i\bDelta_i^{\top}\right)^{-1}=
	\begin{pmatrix}
		(1-\bd^{\top}\bS^{-1}\bd)^{-1} & (1-\bd^{\top}\bS^{-1}\bd)^{-1}\bd^{\top}\bS^{-1}
	\end{pmatrix},\label{eq_3}
\end{align}
Since $\bDelta_i=(1,\bU_i^{\top})^{\top}$, we also have
\begin{align}
	\sum_{i\in\Isc}\omega_i(\bx,\xi)\bDelta_i(\varepsilon_i+r_i)=
	\begin{pmatrix}
		\sum_{i\in\Isc}\omega_i(\bx,\xi)(\varepsilon_i+r_i) & \sum_{i\in\Isc}\omega_i(\bx,\xi)\bU_i^{\top}(\varepsilon_i+r_i)
	\end{pmatrix}^{\top}.\label{eq_4}
\end{align}
It follows that
\begin{align}
	&\E_{\bx}\left[\mhat_\mathrm{LCF}(\bx)-m(\bx)\right]^2\leq \E_{\xi}\biggl[\E_{\bx}\biggl[(1-\bd^{\top}\bS^{-1}\bd)^{-1}\sum_{i\in\Isc}\omega_i(\bx,\xi)(\varepsilon_i+r_i)\nonumber\\
	&\qquad+(1-\bd^{\top}\bS^{-1}\bd)^{-1}\bd^{\top}\bS^{-1}\sum_{i\in\Isc}\omega_i(\bx,\xi)\bU_i(\varepsilon_i+r_i) \biggl]^2\biggl].\label{eq_1}
\end{align}
Define $\bU_i^L:=\left(Z_{i1}^L,\dots,Z_{id}^L,(Z_{i1}^L)^2,Z_{i1}^LZ_{i2}^L,\dots,(Z_{id}^L)^2,\dots,(Z_{i1}^L)^q\right)^\top\in\R^{\dbar}$ with $Z_{ij}^L:=(\bX_{ij}-\bx_j)/\mathrm{diam}_j(L(\bx,\xi))$ for any $i\in\mathcal I$ and $j \leq d$. Define a $\dbar\times \dbar$ diagonal matrix $\bD_L:=\mathrm{diag}(\mathrm{diam}_1(L(\bx,\xi))$, $\dots$, $\mathrm{diam}_d(L(\bx,\xi))$, $\mathrm{diam}_1^2(L(\bx,\xi))$, $\mathrm{diam}_1(L(\bx,\xi))\mathrm{diam}_2(L(\bx,\xi))$, $\dots$, $\mathrm{diam}_1^2(L(\bx,\xi))$, $\dots$, $\mathrm{diam}_d^q(L(\bx,\xi)))$. Then,
\begin{align}
	&\bU_i=\bD_L\bU_i^L,\;\;\bd=\bD_L\bd_L,\;\;\text{and}\;\;\bS=\bD_L\bS_L\bD_L,\;\;\text{where}\label{def:U_L}\\
	&\bd_L:=\sum_{i \in \Isc}\omega_i(\bx,\xi)\bU_{i}^L\;\;\text{and}\;\;\bS_L:=\sum_{i \in \Isc}\omega_i(\bx,\xi)\bU_{i}^L(\bU_{i}^L)^{\top}.\label{def:d_S}
\end{align}
Plugging \eqref{def:U_L} into \eqref{eq_1}, we have
\begin{align}
	&\E_{\bx}\left[\mhat_\mathrm{LCF}(\bx)-m(\bx)\right]^2\leq \E_{\xi}\biggl[\E_{\bx}\biggl[(1-\bd_L^{\top}\bS_L^{-1}\bd_L)^{-1}\sum_{i\in\Isc}\omega_i(\bx,\xi)(\varepsilon_i+r_i)\nonumber\\
	&\qquad+(1-\bd_L^{\top}\bS_L^{-1}\bd_L)^{-1}\bd_L^{\top}\bS_L^{-1}\sum_{i\in\Isc}\omega_i(\bx,\xi)\bU_i^L (\varepsilon_i+r_i) \biggl]^2\biggl].
\end{align}
Let $\bc_L:=\bS_L-\bd_L \bd_L^{\top}$. On the event $\Bsc$, the matrix $\bc_L$ is invertible. Since $\bd_L\left(\bd_L^{\top} \bc_L^{-1} \bd_L+1\right)=\left(\bd_L \bd_L^{\top}+\bc_L\right)\bc_L^{-1} \bd_L$ and $\bd_L^{\top} \bc_L^{-1} \bd_L\geq0$, we have $\bd_L=\left(\bd_L^{\top} \bc_L^{-1} \bd_L+1\right)^{-1}\left(\bd_L \bd_L^{\top}+\bc_L\right)\bc_L^{-1} \bd_L$. It follows that
\begin{align}
	\bd_L^{\top}\bS_L^{-1}\bd_L
	&=\bd_L^{\top}\left(\bd_L \bd_L^{\top}+\bc_L\right)^{-1}\bd_L\nonumber\\
	&=\bd_L^{\top}\left(\bd_L \bd_L^{\top}+\bc_L\right)^{-1}\left(\bd_L^{\top} \bc_L^{-1} \bd_L+1\right)^{-1}\left(\bd_L \bd_L^{\top}+\bc_L\right)\bc_L^{-1} \bd_L\nonumber\\
	&=\frac{\bd_L^{\top}\bc_L^{-1} \bd_L}{\bd_L^{\top} \bc_L^{-1} \bd_L+1}\leq1. \label{eq:dSd}
\end{align}
Then, we have $\left(1-\bd_L^{\top}\bS_L^{-1}\bd_L\right)^{-1}=1+\bd_L^{\top} \left(\bS_L-\bd_L \bd_L^{\top}\right)^{-1} \bd_L$.
Therefore,
\begin{align*}
	\E_{\bx}\left[\mhat_\mathrm{LCF}(\bx)-m(\bx)\right]^2\leq \E_{\xi}\left[\E_{\bx}\left[ \sum_{i=1}^{4}\Delta_i(\bx,\xi) \right]^2\sup_{\bx\in[0,1]^d}\left\{1+\bd_L^{\top} \left(\bS_L-\bd_L \bd_L^{\top}\right)^{-1} \bd_L\right\}\right],
\end{align*}
where for any $\bx\in[0,1]^d$ and $\xi\in\Xi$, 
\begin{align}
	\Delta_1(\bx,\xi)&:=\bd_L^{\top}\bS_L^{-1}\sum_{i\in\Isc}\omega_i(\bx,\xi)\bU_i^L r_i,\;\;\Delta_2(\bx,\xi):=\sum_{i\in\Isc}\omega_i(\bx,\xi)r_i,\label{eq_5}\\
	\Delta_3(\bx,\xi)&:=\bd_L^{\top}\bS_L^{-1}\sum_{i\in\Isc}\omega_i(\bx,\xi)\bU_i^L \varepsilon_i,\;\;\Delta_4(\bx,\xi):=\sum_{i\in\Isc}\omega_i(\bx,\xi)\varepsilon_i.\label{eq_6}
\end{align}
By the finite form of Jensen's inequality, we have
\begin{align*}
	\left[\frac{1}{4} \sum_{i=1}^{4}\Delta_i(\bx,\xi) \right]^2\leq \frac{1}{4}\sum_{i=1}^{4}\left[\Delta_i(\bx,\xi)\right]^2,
\end{align*}
which implies that
\begin{align}
	\E_{\bx}\left[\mhat_\mathrm{LCF}(\bx)-m(\bx)\right]^2\leq \E_{\xi}\left[4\sum_{i=1}^{4}\E_{\bx}\left[ \Delta_i(\bx,\xi) \right]^2\sup_{\bx\in[0,1]^d}\left\{1+\bd_L^{\top} \left(\bS_L-\bd_L \bd_L^{\top}\right)^{-1} \bd_L\right\}\right].\label{consistency:R1-R4}
\end{align}
By Lemma \ref{lem:sample eigen}, we have
\begin{align}\label{eq:cov}
	\sup_{\bx\in[0,1]^d,\xi\in\Xi}\left\{1+\bd_L^{\top} \left(\bS_L-\bd_L \bd_L^{\top}\right)^{-1} \bd_L\right\}=O_p(1).
\end{align}
Since $m\in \mathcal{H}^{q,\beta}$, by the Taylor's theorem, we have
\begin{align*}
	m(\bX_i)=P_{q-1}(\bX_i)+R_{q-1}(\bX_i),\;\;\text{where}\;\; P_{q-1}(\bX_i):= \sum_{|\bgamma|=0}^{q-1}\frac{D^{\bgamma}m(\bx)}{\bgamma!}(\bX_i-\bx)^{\bgamma}
\end{align*}
and $R_{q-1}(\bX_i) := \sum_{|\bgamma|=q}D^{\bgamma}m(\bxi)(\bX_i-\bx)^{\bgamma}/\bgamma!$ for some $\bxi_i$ between $\bx$ and $\bX_i$. By definition, $r_i=m(\bX_i)-\bG(\bX_i)^{\top}\bbeta=R_{q-1}(\bX_i)-(\bG(\bX_i)^{\top}\bbeta-P_{q-1}(\bX_i))=\sum_{|\bgamma|=q}(D^{\bgamma}m(\bxi_i)-D^{\bgamma}m(\bx))(\bX_i-\bx)^{\bgamma}/\bgamma!$ since $\bG(\bX_i)^{\top}\bbeta=\sum_{|\bgamma|=0}^{q}D^{\bgamma}m(\bx)(\bX-\bx)^{\bgamma}/\bgamma!$. By Assumption \ref{cond:holder}, we have
\begin{align}\label{consistency:eq1}
	r_i\leq\sum_{|\bgamma|=q}\frac{L_0}{\bgamma!}\|\bxi_i-\bx\|^\beta\|\bX_i-\bx\|^q\leq\sum_{|\bgamma|=q}\frac{L_0}{\bgamma!}\|\bX_i-\bx\|^{q+\beta}.
\end{align}
It follows that, for any $\xi\in\Xi$,
\begin{align*}
	\E_{\bx}\left[ \Delta_1(\bx,\xi) \right]^2\leq\left[\sum_{|\bgamma|=q}\frac{L_0}{\bgamma!}\right]^2\E_{\bx}\left[ \sum_{i\in\Isc}\omega_i(\bx,\xi)\bd_L^{\top}\bS_L^{-1}\bU_i^L \|\bX_i-\bx\|^{q+\beta} \right]^2.
\end{align*}
By Cauchy-Schwarz inequality, 
\begin{align}
	&\left(\sum_{i\in\Isc}\omega_i(\bx,\xi)\bd_L^{\top}\bS_L^{-1}\bU_i^L \|\bX_i-\bx\|^{q+\beta} \right)^2\nonumber\\
	&\qquad\leq \left(\sum_{i\in\Isc}\omega_i(\bx,\xi)\bd_L^{\top}\bS_L^{-1}\bU_i^L(\bU_i^L)^{\top} \bS_L^{-1}\bd_L\right)\left(\sum_{i\in\Isc}\omega_i(\bx,\xi)\|\bX_i-\bx\|^{2(q+\beta)} \right)\nonumber\\
	&\qquad\overset{(i)}{=} \bd_L^{\top}\bS_L^{-1}\bd_L\sum_{i\in\Isc}\omega_i(\bx,\xi)\|\bX_i-\bx\|^{2(q+\beta)}\overset{(ii)}{\leq} \sum_{i\in\Isc}\omega_i(\bx,\xi)\|\bX_i-\bx\|^{2(q+\beta)},\label{eq_7}
\end{align}
where (i) holds by the fact that $\bS_L=\sum_{i\in\Isc}\omega_i(\bx,\xi)\bU_i^L (\bU_i^L)^{\top}$; (ii) holds by \eqref{eq:dSd}. Then, we have 
\begin{align*}
	\E_{\bx}\left[ \Delta_1(\bx,\xi) \right]^2\leq\left[\sum_{|\bgamma|=q}\frac{L_0}{\bgamma!}\right]^2\E_{\bx}\left[\sum_{i\in\Isc}\omega_i(\bx,\xi)\|\bX_i-\bx\|^{2(q+\beta)} \right].
\end{align*}
By construction, we have
\begin{align}
	\sum_{i\in\Isc}\omega_i(\bx,\xi)\|\bX_i-\bx\|^{2(q+\beta)}
	&\leq \sum_{i\in\Isc} \omega_i(\bx,\xi) \mathrm{diam}^{2(q+\beta)}(L(\bx,\xi))\overset{(i)}{=} \mathrm{diam}^{2(q+\beta)}(L(\bx,\xi)),\label{consistency:eq4}
\end{align}
where (i) holds by $\sum_{i\in\Isc} \omega_i(\bx,\xi)=1$. By Lemma \ref{lem:balance diam}, for any $\xi\in\Xi$,
\begin{align*}
	&\E_{\S_\Isc}\left[\E_{\bx}\left[\mathrm{diam}^{2(q+\beta)}(L(\bx,\xi)) \right]\right]\leq\sup_{\bx\in[0,1]^d}\E_{\S_\Isc}\left[\mathrm{diam}^{2(q+\beta)}(L(\bx,\xi))\right]\\
	&\qquad\leq 4^{q+\beta}d^{q+\beta}\exp \left(4(q+\beta)^2+2(q+\beta)+1\right)\left(\frac{\lfloor wN\rfloor}{2k-1}\right)^{-\frac{2(q+\beta)\log (1-\alpha)}{d \log (\alpha)}}.
\end{align*}
By Markov's inequality, as $N\to\infty$, we have
\begin{align}
	\E_{\bx}\left[\mathrm{diam}^{2(q+\beta)}(L(\bx,\xi)) \right]=O_p\left(\left(\frac{N}{k}\right)^{-\frac{2(q+\beta)\log (1-\alpha)}{d \log (\alpha)}}\right).\label{bound:diam}
\end{align}
Therefore, for any $\xi\in\Xi$,
\begin{align}
	\E_{\bx}\left[ \Delta_1(\bx,\xi) \right]^2&\leq\left[\sum_{|\bgamma|=q}\frac{L_0}{\bgamma!}\right]^2\E_{\bx}\left[\mathrm{diam}^{2(q+\beta)}(L(\bx,\xi)) \right]\nonumber\\
	&=O_p\left(\left(\frac{N}{k}\right)^{-\frac{2(q+\beta)\log (1-\alpha)}{d \log (\alpha)}}\right).\label{consistency:R1}
\end{align}
In addition, by \eqref{consistency:eq1}, for any $\xi\in\Xi$,
\begin{align*}
	\E_{\bx}\left[ \Delta_2(\bx,\xi) \right]^2\leq\left[\sum_{|\bgamma|=q}\frac{L_0}{\bgamma!}\right]^2\E_{\bx}\left[ \sum_{i\in\Isc}\omega_i(\bx,\xi)\|\bX_i-\bx\|^{q+\beta} \right]^2.
\end{align*}
By Cauchy-Schwarz inequality, 
\begin{align}
	&\left[ \sum_{i\in\Isc}\omega_i(\bx,\xi)\|\bX_i-\bx\|^{q+\beta} \right]^2\leq \sum_{i\in\Isc}\omega_i(\bx,\xi)\sum_{i\in\Isc}\omega_i(\bx,\xi)\|\bX_i-\bx\|^{2(q+\beta)}\nonumber\\
	&\qquad\overset{(i)}{=}\sum_{i\in\Isc}\omega_i(\bx,\xi)\|\bX_i-\bx\|^{2(q+\beta)}\overset{(ii)}{\leq} \mathrm{diam}^{2(q+\beta)}(L(\bx,\xi)),\label{eq_8}
\end{align}
where (i) holds by $\sum_{i\in\Isc} \omega_i(\bx,\xi)=1$; (ii) holds by \eqref{consistency:eq4}. Therefore, we have
\begin{align*}
	\E_{\bx}\left[ \Delta_2(\bx,\xi) \right]^2\leq\left[\sum_{|\bgamma|=q}\frac{L_0}{\bgamma!}\right]^2\E_{\bx}\left[\mathrm{diam}^{2(q+\beta)}(L(\bx,\xi)) \right].
\end{align*}
Together with \eqref{bound:diam}, for any $\xi\in\Xi$, we have
\begin{align}
	\E_{\bx}\left[ \Delta_2(\bx,\xi) \right]^2=O_p\left(\left(\frac{N}{k}\right)^{-\frac{2(q+\beta)\log (1-\alpha)}{d \log (\alpha)}}\right).\label{consistency:R2}
\end{align}
As for the term $\Delta_3(\bx,\xi)$, for any $\xi\in\Xi$,
\begin{align*}
	&\E_{\S_\Isc}\left[\E_{\bx}\left[ \Delta_3(\bx,\xi) \right]^2\right]=\E_{\bx}\left[\E_{\S_\Isc}\left[ \Delta_3(\bx,\xi) \right]^2\right]\\
	&\qquad=\E_{\bx}\left[\E_{\S_\Isc}\left[\bd_L^{\top}\bS_L^{-1}\sum_{i\in\Isc}\omega_i^2(\bx,\xi)\bU_i^L (\bU_i^L)^{\top}\varepsilon_i^2 \bS_L^{-1}\bd_L\right]\right]\\
	&\qquad\qquad+\E_{\bx}\left[\E_{\S_\Isc}\left[\sum_{i,j\in\Isc,i\neq j}\bd_L^{\top}\bS_L^{-1}\bU_i^L \bU_{L,j}^{\top}\bS_L^{-1}\bd_L\omega_i(\bx,\xi)\omega_j(\bx)\varepsilon_i\varepsilon_j\right]\right].
\end{align*}
By the tower rule, for any $i,j\in\Isc$ with $i\neq j$, we have
\begin{align*}
	&\E_{\S_\Isc}\left[\bd_L^{\top}\bS_L^{-1}\bU_i^L \bU_{L,j}^{\top}\bS_L^{-1}\bd_L\omega_i(\bx,\xi)\omega_j(\bx)\varepsilon_i\varepsilon_j\right]\\
	&\qquad\overset{(i)}{=}\E_{\S_\Isc}\left[\bd_L^{\top}\bS_L^{-1}\bU_i^L \bU_{L,j}^{\top}\bS_L^{-1}\bd_L\omega_i(\bx,\xi)\omega_j(\bx)\E_{\S_\Isc}\left[\varepsilon_i\varepsilon_j \mid \{\bX_l\}_{l=1}^N,\{Y_l\}_{l\in\Jsc}\right]\right]\\
	&\qquad\overset{(ii)}{=}\E_{\S_\Isc}\left[\bd_L^{\top}\bS_L^{-1}\bU_i^L \bU_{L,j}^{\top}\bS_L^{-1}\bd_L\omega_i(\bx,\xi)\omega_j(\bx)\E_{\S_\Isc}\left[\varepsilon_i \mid \bX_i \right]\E_{\S_\Isc}\left[\varepsilon_j \mid \bX_j \right]\right]\overset{(iii)}{=}0,
\end{align*}
where (i) holds by ``honesty" of the forests; (ii) holds by the independency of the samples; (iii) holds since $\E[\varepsilon\mid\bX]=0$. Therefore, we have 
\begin{align*}
	\E_{\S_\Isc}\left[\E_{\bx}\left[ \Delta_3(\bx,\xi) \right]^2\right]=\E_{\bx}\left[\E_{\S_\Isc}\left[\bd_L^{\top}\bS_L^{-1}\sum_{i\in\Isc}\omega_i^2(\bx,\xi)\bU_i^L (\bU_i^L)^{\top}\varepsilon_i^2 \bS_L^{-1}\bd_L\right]\right].
\end{align*}
By the tower rule, we have
\begin{align*}
	&\E_{\S_\Isc}\left[\E_{\bx}\left[ \Delta_3(\bx,\xi) \right]^2\right]\\
	&\qquad\overset{(i)}{=}\E_{\bx}\left[\E_{\S_\Isc}\left[\E_{\S_\Isc}\left[\bd_L^{\top}\bS_L^{-1}\sum_{i\in\Isc}\omega_i^2(\bx,\xi)\bU_i^L (\bU_i^L)^{\top}\varepsilon_i^2 \bS_L^{-1}\bd_L\mid \{\bX_l\}_{l=1}^N, \{Y_l\}_{l\in\Jsc}\right]\right] \right]\\
	&\qquad\overset{(ii)}{=}\E_{\bx}\left[\E_{\S_\Isc}\left[\E_{\S_\Isc}\left[\bd_L^{\top}\bS_L^{-1}\sum_{i\in\Isc}\omega_i^2(\bx,\xi)\bU_i^L (\bU_i^L)^{\top}\bS_L^{-1}\bd_L\E[\varepsilon_i^2\mid\bX_i]\right]\right] \right]\\
	&\qquad\overset{(iii)}{\leq} M\E_{\bx}\left[\E_{\S_\Isc}\left[\bd_L^{\top}\bS_L^{-1}\sum_{i\in\Isc}\omega_i^2(\bx,\xi)\bU_i^L (\bU_i^L)^{\top} \bS_L^{-1}\bd_L\right] \right],
\end{align*}
where (i) holds by ``honesty" of the forests; (ii) holds by the independency of the samples; (iii) holds by Assumption \ref{cond:noise}.
By \eqref{consistency:eq3}, we have
\begin{align*}
	\E_{\S_\Isc}\left[\E_{\bx}\left[ \Delta_3(\bx,\xi) \right]^2\right]
	\leq \frac{M}{k}\E_{\bx}\left[\E_{\S_\Isc}\left[\bd_L^{\top}\bS_L^{-1}\sum_{i\in\Isc}\omega_i(\bx,\xi)\bU_i^L (\bU_i^L)^{\top} \bS_L^{-1}\bd_L\right] \right].
\end{align*}
Since $\bS_L=\sum_{i\in\Isc}\omega_i(\bx,\xi)\bU_i^L (\bU_i^L)^{\top}$ and \eqref{eq:dSd} holds, we have $\bd_L^{\top}\bS_L^{-1}\sum_{i\in\Isc}\omega_i(\bx,\xi)\bU_i^L (\bU_i^L)^{\top} $ $\bS_L^{-1}\bd_L=\bd_L^{\top}\bS_L^{-1}\bd_L\leq 1$, and hence
\begin{align*}
	\E_{\S_\Isc}\left[\E_{\bx}\left[ \Delta_3(\bx,\xi) \right]^2\right]\leq \frac{M}{k}.
\end{align*}
By Markov's inequality, for any $\xi\in\Xi$, we have
\begin{align}
	\E_{\bx}\left[ \Delta_3(\bx,\xi) \right]^2=O_p\left(\frac{1}{k}\right).\label{consistency:R3}
\end{align}
Lastly, for the term $\Delta_4(\bx,\xi)$, with any $\xi\in\Xi$,
\begin{align*}
	&\E_{\S_\Isc}\left[\E_{\bx}\left[ \Delta_4(\bx,\xi) \right]^2\right]=\E_{\bx}\left[\E_{\S_\Isc}\left[ \Delta_4(\bx,\xi) \right]^2\right]\\
	&\qquad=\E_{\bx}\left[\E_{\S_\Isc}\left[\sum_{i\in\Isc}\omega_i^2(\bx,\xi)\varepsilon_i^2\right]\right]+\E_{\bx}\left[\E_{\S_\Isc}\left[\sum_{i,j\in\Isc,i\neq j}\omega_i(\bx,\xi)\omega_j(\bx)\varepsilon_i\varepsilon_j\right]\right].
\end{align*}
Using the tower rule, we also have
\begin{align*}
	&\E_{\S_\Isc}\left[\omega_i(\bx,\xi)\omega_j(\bx,\xi)\varepsilon_i\varepsilon_j\right]\overset{(i)}{=}\E_{\S_\Isc}\left[\omega_i(\bx,\xi)\omega_j(\bx,\xi)\E_{\S_\Isc}\left[\varepsilon_i\varepsilon_j \mid \{\bX_l\}_{l=1}^N, \{Y_l\}_{l\in\Jsc} \right]\right]\\
	&\qquad\overset{(ii)}{=}\E_{\S_\Isc}\left[\omega_i(\bx,\xi)\omega_j(\bx,\xi)\E_{\S_\Isc}\left[\varepsilon_i \mid \bX_i \right]\E_{\S_\Isc}\left[\varepsilon_j \mid \bX_j \right]\right]\overset{(iii)}{=}0,
\end{align*}
where (i) holds by ``honesty" of the forests; (ii) holds by the independency of the samples; (iii) holds since $\E[\varepsilon\mid\bX]=0$. Therefore, we have 
\begin{align*}
	&\E_{\S_\Isc}\left[\E_{\bx}\left[ \Delta_4(\bx,\xi) \right]^2\right]=\E_{\bx}\left[\E_{\S_\Isc}\left[\sum_{i\in\Isc}\omega_i^2(\bx,\xi)\varepsilon_i^2\right]\right]\\
	&\quad\overset{(i)}{=}\E_{\bx}\left[\E_{\S_\Isc}\left[\E_{\S_\Isc}\left[\sum_{i\in\Isc}\omega_i^2(\bx,\xi) \varepsilon_i^2\mid \{\bX_l\}_{l=1}^N, \{Y_l\}_{l\in\Jsc} \right]\right]\right]\\
	&\quad\overset{(iii)}{=}\E_{\bx}\left[\E_{\S_\Isc}\left[\sum_{i\in\Isc}\omega_i^2(\bx,\xi)\E_{\S_\Isc}\left[ \varepsilon_i^2\mid \bX_i\right]\right]\right]\\
	&\quad\overset{(iv)}{\leq} M\E_{\bx}\left[\E_{\S_\Isc}\left[\sum_{i\in\Isc}\omega_i^2(\bx,\xi)\right]\right],
\end{align*}
where (i) holds by the tower rule and ``honesty" of the forests; (ii) holds by the independency of the samples; (iii) holds by Assumption \ref{cond:noise}. By \eqref{consistency:eq3} and $\sum_{i\in\Isc}\omega_i(\bx,\xi)=1$, we have
\begin{align*}
	\E_{\S_\Isc}\left[\E_{\bx}\left[ \Delta_4(\bx,\xi) \right]^2\right]\leq\frac{M}{k}.
\end{align*}
By Markov's inequality, for any $\xi\in\Xi$, we have 
\begin{align}
	\E_{\bx}\left[ \Delta_4(\bx,\xi) \right]^2=O_p\left(\frac{1}{k}\right). \label{consistency:R4}
\end{align}
Combining \eqref{consistency:R1}, \eqref{consistency:R2}, \eqref{consistency:R3} and \eqref{consistency:R4} with \eqref{eq:cov}, we have
\begin{align*}
	&\E_{\xi}\left[4\sum_{i=1}^{4}\E_{\bx}\left[ \Delta_i(\bx,\xi) \right]^2\sup_{\bx\in[0,1]^d}\left\{1+\bd_L^{\top} \left(\bS_L-\bd_L \bd_L^{\top}\right)^{-1} \bd_L\right\}\right]\\
	&\qquad=O_p\left(\frac{1}{k}+ \left(\frac{N}{k}\right)^{-\frac{2(q+\beta)\log (1-\alpha)}{d \log (\alpha)}} \right).
\end{align*}
Together with \eqref{consistency:R1-R4}, we conclude that
$$\E_{\bx}\left[\mhat_\mathrm{L}(\bx)-m(\bx)\right]^2=O_p\left(\frac{1}{k}+ \left(\frac{k}{N}\right)^{\frac{2(q+\beta)\log (1-\alpha)}{d \log (\alpha)}} \right). $$
\end{proof}

\section{Proofs of the uniform convergence results}

\begin{proof}[Proof of Lemma \ref{lem:balance diam_point-wise}]
In this proof, we use the same notation as Lemma \ref{lem:balance diam}. Define $T$ as in \eqref{diam_eq7}. For any $1 \leq j \leq d$, we have
\begin{align}
	\mathrm{diam}_j(L(\bx,\xi))\leq\mathrm{diam}_j(L_{k_{T,j}}(\bx,\xi))\overset{(i)}{=}\prod_{i=1}^{T} \frac{\mathrm{diam}_j(L_{k_{i,j}}(\bx,\xi))}{\mathrm{diam}_j(L_{k_{i,j}-1}(\bx,\xi))},\label{u_1}
\end{align}
where (i) holds by $\mathrm{diam}_j(L_{k_{i,j}-1}(\bx,\xi))=\mathrm{diam}_j(L_{k_{i-1,j}}(\bx,\xi))$ for any $2\leq i\leq T$ and the fact that $\mathrm{diam}_j(L_{k_{1,j}}(\bx,\xi))=1$. 

For any $m\in\{1,2,\dots,Td\}$, choose $\omega_m=k(2n(1-\alpha)^{Td-m})^{-1}$, $\epsilon_m=1/ \sqrt{n\omega_m}$ and $\delta=1/\sqrt{n}$ throughout this proof, where $n=\lfloor wN\rfloor$. By \eqref{balance diam_eq1}, $\omega_m\leq1/2$. Additionally, we also have $(n\omega_m)^{-1}\leq2/k=o(1)$ and $\epsilon_m\leq\sqrt{2/k}=o(1)$ as $N\to\infty$. Define $\mathcal{R}_m:=\mathcal{R}_{\mathcal{D},\omega_m,\epsilon_m}$ as in Lemma \ref{lem:R cardinality} and the event
\begin{align}\label{Asc_i}
\Asc_m:=\left\{\sup\left\{ \frac{|\# R-n\mu(R)|}{\sqrt{n\mu(R)}}: R\in \mathcal{R}_m, \mu(R)\geq\exp(-\epsilon_m)\omega_m \right\}\leq\sqrt{3\log\left(\frac{\# \mathcal{R}_m }{\delta}\right)}\right\},
\end{align}
for any $m\leq Td$. By Lemma \ref{lem:R cardinality}, as $N\to\infty$, we have $\log(\#\mathcal{R}_m /\delta)\asymp\log(n)$ and $1/\#\mathcal{R}_m=O(n^{-1}(n\omega_m)^{1-d})=O(n^{-1})$ since $d\geq1$. Hence,
\begin{align*}
\frac{\log(\#\mathcal{R}_m)}{n\exp(-\epsilon_m)\omega_m}=O\left(\frac{\log(n)}{k}\right)=o(1)\;\;\text{and}\;\;\frac{\sqrt {n} }{\#\mathcal{R}_m}=o(1),
\end{align*}
as long as $k\gg\log(N)\asymp\log(n)$. By Lemma \ref{lem:event A}, when $n$ is large enough, we have $\P_{\S_\Isc}\left(\Asc_m^c\right)\leq 1/\sqrt n.$ By the union bound, we have
\begin{align*}
	\P_{\S_\Isc}\left(\bigcup_{m=1}^{Td}\Asc_m^c\right)\leq \frac{Td}{\sqrt n} \leq \frac{\log(n/k)}{\sqrt n\log\left((1-\alpha)^{-1}\right)}=o(1),
\end{align*}
since $T$, defined as \eqref{diam_eq7}, satisfies $Td\leq\log(n/k)/\log((1-\alpha)^{-1})$. Condition on the event $\cap_{m=1}^{Td}\Asc_m$. Note that $\omega_m\geq\exp(-\epsilon_m)\omega_m$. By Lemma 13 of \cite{wager2015adaptive},
\begin{align*}
&\sup\{\# R:\mu(R)=\omega_m\}\leq \exp(\epsilon_m)n\omega_m+\exp(\epsilon_m/2)\sqrt{3n\omega_m\log(\#\mathcal{R}_m/\delta)}\\
&\qquad= \exp(1/\sqrt{n\omega_m})n\omega_m+\exp\{1/(2\sqrt{n\omega_m})\}\sqrt{3n\omega_m\log(\#\mathcal{R}_m/\delta)}< 2n\omega_m,
\end{align*}
when $n\omega_m\geq k/2$ is large enough, since $\log(\#\mathcal{R}_m/\delta)=O(\log(n))=o(n\omega_m)$. Therefore, any rectangle containing at least $2n\omega_m$ samples must have size greater than $\omega_m$. Meanwhile, by \eqref{balance diam_eq2}, we have $\# L_m(\bx,\xi)\geq k/(1-\alpha)^{Td-m}=2n\omega_m$. Hence, $\mu(L_m(\bx,\xi))\geq \omega_m$ for all $\bx\in [0,1]^d$, $\xi\in\Xi$, and $m\leq Td$. By Lemma \ref{lem:R cardinality}, we can choose some $\Rbar_m:=\Rbar_m(\bx,\xi)\in \mathcal{R}_m$ as an inner approximation of $L_m(\bx,\xi)$ satisfying $\Rbar_m\subseteq L_m(\bx,\xi)$ with $\mu(L_m(\bx,\xi))\leq\exp(\epsilon_m)\mu(\Rbar_m)$ and hence $\mu(\Rbar_m)\geq\exp(-\epsilon_m)\mu(L_m(\bx,\xi))\geq\exp(-\epsilon_m)\omega_m$ when $k$ is large enough. Since $\Asc_m$ occurs, for any $\bx\in [0,1]^d$, $\xi\in\Xi$, and $m\leq Td$,
\begin{align*}
\#\Rbar_m\geq n\mu(\Rbar_m)-\sqrt{3n\mu(\Rbar_m)\log(\#\mathcal{R}_m /\delta)},
\end{align*}
which implies that
\begin{align*}
\sqrt{\mu(\Rbar_m)}\leq\sqrt\frac{3\log(\#\mathcal{R}_m /\delta)}{n}+\sqrt\frac{\#\Rbar_m}{n}.
\end{align*}
Hence,
\begin{align}
&\mu(L_m(\bx,\xi))\leq\exp(\epsilon_m)\mu(\Rbar_m)\leq\exp(\epsilon_m)\left(\sqrt\frac{3\log(\#\mathcal{R}_m /\delta)}{n}+\sqrt\frac{\#\Rbar_m}{n}\right)^2\nonumber\\
&\qquad\overset{(i)}{\leq}\exp(\sqrt{2/k})\left(\sqrt\frac{3\log(\#\mathcal{R}_m /\delta)}{n}+\sqrt\frac{n_m}{n}\right)^2\nonumber\\
&\qquad\overset{(ii)}{\leq}\exp(\sqrt{2/k})\left(C\frac{\sqrt{n_m\log(n)}}{n}+\frac{n_m}{n}\right),\label{bound:mu-}
\end{align}
where (i) holds since $\exp(\epsilon_m)=\exp(1/\sqrt{n\omega_m})\leq\exp(\sqrt{2/k})$, $\Rbar_m\subseteq L_m(\bx,\xi)$, and hence $\#\Rbar_m\leq\#L_m(\bx,\xi)=:n_m$; (ii) holds with some constant $C>0$ as $\log(\#\mathcal{R}_m /\delta)=O(\log(n))=O(k)=O(n_m)$. 

By Lemma \ref{lem:R cardinality}, we can also choose some $\Rtil_m:=\Rbar_m(\bx,\xi)\in \mathcal{R}_m$ as an outer approximation of $L_m(\bx,\xi)$ satisfying $\Rtil_m\supseteq L_m(\bx,\xi)$ with $\mu(\Rtil_m)\leq\exp(\epsilon_m)\mu(L_m(\bx,\xi))$ when $k$ is large enough. In addition, since $\mu(\Rtil_m)\geq\mu(L_m(\bx,\xi))\geq \omega_m\geq\exp(-\epsilon_m)\omega_m$, on the event $\Asc_m$,
$$\#\Rtil_m\leq n\mu(\Rtil_m)+\sqrt{3n\mu(\Rtil_m)\log(\#\mathcal{R}_m /\delta)},$$
which implies that
\begin{align*}
\sqrt{\mu(\Rtil_m)}&\geq\frac{\sqrt{3n\log(\#\mathcal{R}_m /\delta)+4n\#\Rtil_m}-\sqrt{3n\log(\#\mathcal{R}_m /\delta)}}{2n}\\
&\geq\frac{\sqrt{3\log(\#\mathcal{R}_m /\delta)+4n_m}-\sqrt{3\log(\#\mathcal{R}_m /\delta)}}{2\sqrt n}\\
&=\frac{2n_m/\sqrt n}{\sqrt{3\log(\#\mathcal{R}_m /\delta)+4n_m}+\sqrt{3\log(\#\mathcal{R}_m /\delta)}},
\end{align*}
since $\#\mathcal{R}_m\geq\#L_m(\bx,\xi)=n_m$. It follows that
\begin{align}
&1/\mu(L_{m-1}(\bx,\xi))\leq1/\mu(\Rtil_{m-1})\nonumber\\
&\qquad\leq\frac{6\log(\#\mathcal{R}_{m-1} /\delta)+4n_{m-1}+\sqrt{(3\log(\#\mathcal{R}_{m-1} /\delta)+4n_{m-1})3\log(\#\mathcal{R}_{m-1} /\delta)}}{4n_{m-1}^2/n}\nonumber\\
&\qquad\leq\frac{n}{n_{m-1}}+\frac{cn\sqrt{\log(n)}}{n_{m-1}^{3/2}},\label{bound:mu+}
\end{align}
with some constant $c>0$, since $\log(\#\mathcal{R}_{m-1} /\delta)\asymp\log(n)=O(k)=O(n_{m-1})$.

By definition, the leaves $L_{k_{i,j}}(\bx,\xi)$ and $L_{k_{i,j}-1}(\bx,\xi))$ only differ along the $j$th coordinate. Hence, for any $j\leq d$ and $i\leq T$, we have
\begin{align*}
&\frac{\mathrm{diam}_j(L_{k_{i,j}}(\bx,\xi))}{\mathrm{diam}_j(L_{k_{i,j}-1}(\bx,\xi))}=\frac{\mu(L_{k_{i,j}}(\bx,\xi))}{\mu(L_{k_{i,j}-1}(\bx,\xi))}\\
&\qquad\overset{(i)}{\leq}\exp(\sqrt{2/k})\left(C\frac{\sqrt{n_{k_{i,j}}\log(n)}}{n}+\frac{n_{k_{i,j}}}{n}\right)\left(\frac{n}{n_{k_{i,j}-1}}+\frac{cn\sqrt{\log(n)}}{n_{k_{i,j}-1}^{3/2}}\right)\\
&\qquad\overset{(ii)}{\leq}\exp(\sqrt{2/k})\left(1-\alpha+C'\sqrt\frac{\log(n)}{k}\right),
\end{align*}
where (i) holds by \eqref{bound:mu-} and \eqref{bound:mu+}, (ii) holds with some constant $C'>0$ since $n_{k_{i,j}}\leq (1-\alpha)n_{k_{i,j}-1}$, $n_{k_{i,j}-1}\geq k$, and $k\gg\log(n)$. Together with \eqref{u_1}, we have for any $\bx\in [0,1]^d$ and $\xi\in\Xi$,
\begin{align*}
&\mathrm{diam}_j(L(\bx,\xi))\leq\prod_{i=1}^{T}\exp(\sqrt{2/k})\left(1-\alpha+C'\sqrt\frac{\log(n)}{k}\right)\\
&\qquad\leq(1-\alpha)^T\exp(T\sqrt{2/k})\left(1+\frac{C'}{1-\alpha}\sqrt\frac{\log(n)}{k}\right)^T.
\end{align*}
Since $T\leq\frac{\log(n/(2k-1))}{d\log(\alpha^{-1})}$ and $k\gg\log^3(n)$, we have $T\sqrt{2/k}=o(1)$ and hence $\exp(T\sqrt{2/k})=1+o(1)$. Additionally, when $k\gg\log^3(n)$,
$$\left(1+\frac{C'}{1-\alpha}\sqrt\frac{\log(n)}{k}\right)^T\leq\left(1+\frac{C'}{1-\alpha}\sqrt\frac{\log(n)}{k}\right)^{\sqrt\frac{k}{\log(n)}\cdot\sqrt\frac{\log^3(n)}{k}}=1+o(1).$$
Therefore, on the event $\cap_{m=1}^{Td}\Asc_m$, when $n$ is large enough, for all $\bx\in[0,1]^d$ and $\xi\in\Xi$,
\begin{align*}
&\mathrm{diam}_j(L(\bx,\xi))\leq2(1-\alpha)^\frac{\log(n/(2k-1))}{d\log(\alpha^{-1})}=2\left(\frac{n}{2k-1}\right)^{-\frac{\log(1-\alpha)}{d\log(\alpha)}}.
\end{align*}
Since $N\asymp n$ and $\mathrm{diam}^r(L(\bx,\xi))= \left[\sum_{j=1}^d\mathrm{diam}_j^2(L(\bx,\xi))\right]^{r/2}$ for any $r\geq 1$, as $N\to\infty$, we have
\begin{align*}
	\sup_{\bx\in[0,1]^d,\xi\in\Xi}\left\{\mathrm{diam}^r(L(\bx,\xi))\right\} =O\left(\left(\frac{N}{k}\right)^{-\frac{r\log (1-\alpha)}{d \log (\alpha)}} \right).
\end{align*}	
\end{proof}

\begin{proof}[Proof of Theorem \ref{thm:local_consistency_uniform}]
In this proof, we condition on the event $\Asc\cap\Bsc\cap\Csc\cap \bar{\Asc}$, where $ \bar{\Asc}:=\cap_{m=1}^{Td}\Asc_m$. Let $n=\lfloor wN\rfloor$. The event $\Asc$ is defined in Lemma \ref{lem:event A}, with $\mathcal{R}=\mathcal{R}_{\mathcal{D},\omega,\epsilon}$, $\mu_{min}=\omega$, and $\delta=1/\sqrt n$, where $\omega$ and $\epsilon$ are chosen as in \eqref{par: R}. The events $\Bsc$, $\Csc$, and $\Asc_m$ are defined as \eqref{event B}, \eqref{event C}, and \eqref{Asc_i}, respectively. By Lemmas \ref{lem:sample eigen} and \ref{lem:p.s.d}, we know that $\bS_L-\bd_L\bd_L^\top$, $\bS_L$, $\bS$, $\sum_{i\in\Isc}\omega_i(\bx,\xi)\bDelta_i\bDelta_i^\top$ and $\sum_{i\in\Isc}\omega_i(\bx,\xi)\bG(\bX_i)\bG(\bX_i)^{\top}$ are all positive-definite, with $\P_{\S_\Isc}\left(\Bsc\cap\Csc\right)=1-o(1)$. Together with Lemma \ref{lem:event A}, we have $\P_{\S_\Isc}(\Asc\cap\Bsc\cap\Csc\cap\bar{\Asc})=1-o(1)$.
Recall the definition of $\mhat_\mathrm{LCF}(\bx)$, we have
\begin{align*}
	&\sup_{\bx\in[0,1]^d}\left|\mhat_\mathrm{LCF}(\bx)-m(\bx)\right|= \sup_{\bx\in[0,1]^d}\left|\E_{\xi}\left[\bG(\bx)^\top\left(\bbetahat(\bx,\xi)-\bbeta\right)\right]\right|\\
	&\qquad\leq
	\sup_{\bx\in[0,1]^d}\left[\E_{\xi}\left|\bG(\bx)^\top\left(\bbetahat(\bx,\xi)-\bbeta\right)\right|\right].
\end{align*}
By \eqref{eq_2}, we have
\begin{align}
	&\sup_{\bx\in[0,1]^d}\left|\mhat_\mathrm{LCF}(\bx)-m(\bx)\right|\nonumber\\
	&\qquad\leq\sup_{\bx\in[0,1]^d}\left[ \E_{\xi}\left|\be_1^{\top} \left(\sum_{i\in\Isc}\omega_i(\bx,\xi)\bDelta_i\bDelta_i^{\top}\right)^{-1}\left(\sum_{i\in\Isc}\omega_i(\bx,\xi)\bDelta_i(\varepsilon_i+r_i)\right)\right|\right]\nonumber\\
	&\qquad\overset{(i)}{\leq} \sup_{\bx\in[0,1]^d} \biggl[\E_{\xi}\biggl|(1-\bd^{\top}\bS^{-1}\bd)^{-1}\sum_{i\in\Isc}\omega_i(\bx,\xi)(\varepsilon_i+r_i)\nonumber\\
	&\qquad\qquad+(1-\bd^{\top}\bS^{-1}\bd)^{-1}\bd^{\top}\bS^{-1}\sum_{i\in\Isc}\omega_i(\bx,\xi)\bU_i(\varepsilon_i+r_i) \biggl|\biggl].\label{p1}
\end{align}
where (i) hold by \eqref{eq_3} and \eqref{eq_4}.
Plugging \eqref{def:U_L} into \eqref{p1}, we have
\begin{align*}
	\sup_{\bx\in[0,1]^d}\left|\mhat_\mathrm{LCF}(\bx)-m(\bx)\right|&\leq \sup_{\bx\in[0,1]^d} \biggl[\E_{\xi}\biggl|(1-\bd_L^{\top}\bS_L^{-1}\bd_L)^{-1}\sum_{i\in\Isc}\omega_i(\bx,\xi)(\varepsilon_i+r_i)\nonumber\\
	&\qquad+(1-\bd_L^{\top}\bS_L^{-1}\bd_L)^{-1}\bd_L^{\top}\bS_L^{-1}\sum_{i\in\Isc}\omega_i(\bx,\xi)\bU_i^L (\varepsilon_i+r_i)\biggl|\biggl].
\end{align*}
By \eqref{eq:dSd} and the triangle inequality,
\begin{align}
	&\sup_{\bx\in[0,1]^d}\left|\mhat_\mathrm{LCF}(\bx)-m(\bx)\right|\nonumber\\
	&\qquad\leq 
	\sum_{i=1}^{4}\left[\sup_{\bx\in[0,1]^d}\E_{\xi}\left| \Delta_i(\bx,\xi) \right|\right]\sup_{\bx\in[0,1]^d,\xi\in\Xi}\left\{1+\bd_L^{\top} \left(\bS_L-\bd_L \bd_L^{\top}\right)^{-1} \bd_L\right\},\label{delta}
\end{align}
where $	\Delta_i(\bx,\xi)$ ($i\in\{1,2,3,4\}$) are defined as \eqref{eq_5}-\eqref{eq_6}.
Since $m\in \mathcal{H}^{q,\beta}$, by \eqref{consistency:eq1}, we have
\begin{align*}
	\sup_{\bx\in[0,1]^d}\E_{\xi}\left| \Delta_1(\bx,\xi) \right|&\leq\left[\sum_{|\bgamma|=q}\frac{L_0}{\bgamma!}\right]\sup_{\bx\in[0,1]^d}\E_{\xi}\left| \sum_{i\in\Isc}\omega_i(\bx,\xi)\bd_L^{\top}\bS_L^{-1}\bU_i^L \|\bX_i-\bx\|^{q+\beta} \right|,\\
	\sup_{\bx\in[0,1]^d}\E_{\xi}\left| \Delta_2(\bx,\xi) \right|&\leq\left[\sum_{|\bgamma|=q}\frac{L_0}{\bgamma!}\right]\sup_{\bx\in[0,1]^d}\E_{\xi}\left| \sum_{i\in\Isc}\omega_i(\bx,\xi) \|\bX_i-\bx\|^{q+\beta} \right|.
\end{align*}
By \eqref{eq_7}, \eqref{consistency:eq4} and \eqref{eq_8}, we have
\begin{align*}
	\left| \sum_{i\in\Isc}\omega_i(\bx,\xi)\bd_L^{\top}\bS_L^{-1}\bU_i^L \|\bX_i-\bx\|^{q+\beta} \right|&\leq\mathrm{diam}^{(q+\beta)}(L(\bx,\xi)),\\
	\left| \sum_{i\in\Isc}\omega_i(\bx,\xi) \|\bX_i-\bx\|^{q+\beta} \right|&\leq\mathrm{diam}^{(q+\beta)}(L(\bx,\xi)).
\end{align*}
Then, we have
\begin{align*}
	\sup_{\bx\in[0,1]^d}\E_{\xi}\left| \Delta_1(\bx,\xi) \right|&\leq\left[\sum_{|\bgamma|=q}\frac{L_0}{\bgamma!}\right]\sup_{\bx\in[0,1]^d}\E_{\xi}\left[\mathrm{diam}^{(q+\beta)}(L(\bx,\xi))\right],\\
	\sup_{\bx\in[0,1]^d}\E_{\xi}\left| \Delta_2(\bx,\xi) \right|&\leq\left[\sum_{|\bgamma|=q}\frac{L_0}{\bgamma!}\right]\sup_{\bx\in[0,1]^d}\E_{\xi}\left[\mathrm{diam}^{(q+\beta)}(L(\bx,\xi))\right].
\end{align*}
Hence, by Lemma \ref{lem:balance diam_point-wise} with $r=q+\beta$, conditional on the event $\bar{\Asc}$, as $N\to\infty$, we have
\begin{align}
	\sup_{\bx\in[0,1]^d}\E_{\xi}\left| \Delta_1(\bx,\xi) \right|&=O_p\left( \left(\frac{N}{k}\right)^{-\frac{(q+\beta)\log (1-\alpha)}{d\log (\alpha)}}\right),\label{delta_1}\\
	\sup_{\bx\in[0,1]^d}\E_{\xi}\left| \Delta_2(\bx,\xi) \right|&=O_p\left( \left(\frac{N}{k}\right)^{-\frac{(q+\beta)\log (1-\alpha)}{d\log (\alpha)}}\right).\label{delta_2}
\end{align}
Conditional on the event $\Asc$ above, we follow the proof of Lemma \ref{lem:sample eigen} to choose some $\Rbar:=\Rbar(\bx,\xi)\in \mathcal{R}_{\mathcal{D},\omega,\epsilon}$ as an inner approximation of $L(\bx,\xi)$ satisfying $\Rbar\subseteq L(\bx,\xi)$ with \eqref{eigen_eq4} and \eqref{eigen_eq2}. Recall the definition $\omega_{i}^L:=\omega_i(\bx,\xi)=\mathbbm{1}_{\left \{\bX_i\in L(\bx,\xi)\right\}}/\#L$ and $\omega_{i}^{\Rbar}:=\mathbbm{1}_{\left \{\bX_i\in \Rbar\right\}}/\# \Rbar$, where $\#L:=\#L(\bx,\xi)$ and $\#\Rbar:=\#\Rbar(\bx,\xi)$. By the triangle inequality, 
\begin{align}
	\sup_{\bx\in[0,1]^d}\E_{\xi}\left| \Delta_3(\bx,\xi) \right| \leq \sup_{\bx\in[0,1]^d,\xi\in\Xi}\left| \bd_L^{\top}\bS_L^{-1}\sum_{i\in\Isc}\omega_i(\bx,\xi)\bU_i^L \varepsilon_i\right|\leq \sum_{j=1}^2\sup_{\bx\in[0,1]^d,\xi\in\Xi} |\Delta_{3,j}|\label{p8}
\end{align}
where for any $\bx\in[0,1]^d$ and $\xi\in\Xi$,
\begin{align*}
	\Delta_{3,1}&:=\Delta_{3,1}(\bx, \xi)= \sum_{i\in\Isc}\omega_{i}^L\bd_L^{\top}\bS_L^{-1}\bU_i^L \varepsilon_i-\sum_{i\in\Isc}\omega_{i}^{\Rbar} \bd_L^{\top}\bS_L^{-1}\bU_i^L \varepsilon_i,\\
	\Delta_{3,2}&:=\Delta_{3,2}(\bx, \xi)=\sum_{i\in\Isc}\omega_{i}^{\Rbar} \bd_L^{\top}\bS_L^{-1}\bU_i^L \varepsilon_i.
\end{align*}
Note that
\begin{align}
	\left|\Delta_{3,1}\right|
	&\leq\left|\frac{1}{\# L}\sum_{i\in\left\{i\in\Isc:\bX_i\in \Rbar\right\}}\bd_L^{\top}\bS_L^{-1}\bU_i^L \varepsilon_i-\frac{1}{\# \Rbar}\sum_{i\in\left\{i\in\Isc:\bX_i\in \Rbar\right\}}\bd_L^{\top}\bS_L^{-1}\bU_i^L \varepsilon_i\right|\nonumber\\
	&\;\;+\left|\frac{1}{\# L}\sum_{i\in\left\{i\in\Isc:\bX_i\in L(\bx, \xi)\setminus \Rbar\right\}}\bd_L^{\top}\bS_L^{-1}\bU_i^L \varepsilon_i\right|
	\leq\frac{2(\# L-\# \Rbar)}{\# L}\sup_{i\in\left\{i\in\Isc:\bX_i\in L(\bx, \xi)\right\}}\left| \bd_L^{\top}\bS_L^{-1}\bU_i^L \varepsilon_i\right|\nonumber\\
	&\overset{(i)}{\leq} 4M \frac{(\# L-\# \Rbar)}{\# L}\sup_{i\in\left\{i\in\Isc:\bX_i\in L(\bx, \xi)\right\}}\left|\bd_L^{\top}\bS_L^{-1}\bU_i^L \right|,\label{p4}
\end{align}
where (i) holds by $|\varepsilon|=|Y-\E[Y\mid\bX]|\leq2M$ since $\E[\varepsilon\mid \bX]=0$ and $Y\in[-M,M]$. 
By the triangle inequality, we have 
\begin{align*}
	\|\bd_L\|_2\leq\sum_{i\in\Isc}\omega_i(\bx,\xi)\left\|\bU_i^L \right\|_2\overset{(i)}{\leq}\sqrt{\dbar},
\end{align*}
where (i) holds by \eqref{bound_U} and $\sum_{i\in\Isc}\omega_i(\bx,\xi)=1$.
By Cauchy-Schwarz inequality and the sub-multiplicative property of matrix norm, we have for any $\bx\in[0,1]^d$ and $\xi\in\Xi$,
\begin{align}
	\sup_{i\in\left\{i\in\Isc:\bX_i\in L(\bx, \xi)\right\}}\left|\bd_L^{\top}\bS_L^{-1}\bU_i^L \right| \leq\sup_{i\in\left\{i\in\Isc:\bX_i\in L(\bx, \xi)\right\}}\|\bd_L\|_2\|\bS_L\|_2\|\bU_i^L\|_2\overset{(i)}{\leq}\frac{\dbar}{\Lambda_0}.\label{p2}
\end{align}
where (i) holds by Lemma \ref{lem:sample eigen} and \eqref{bound_U}.
By \eqref{eigen_eq1} and \eqref{eigen_eq2}, conditional on the event $\Asc$, as $N\to \infty$, we have
\begin{align}
	\sup_{\bx\in[0,1]^d,\xi\in\Xi}\left\{ \frac{\# L-\# \Rbar}{\# L}\right\}=O\left(\sqrt{\frac{\log(N)}{k}}\right).\label{p3}
\end{align}
Combining \eqref{p2} and \eqref{p3} with \eqref{p4}, conditional on the event $\Asc$, as $N\to \infty$ we have
\begin{align}
	\sup_{\bx\in[0,1]^d,\xi\in\Xi} |\Delta_{3,1}|=O\left(\sqrt{\frac{\log(N)}{k}}\right).\label{p7}
\end{align}
By the tower rule and $\E[\varepsilon\mid\bX]=0$, we have $\E\left[\sum_{i\in\Isc}\omega_{i}^{R} \bd_L^{\top}\bS_L^{-1}\bU_i^L \varepsilon_i \mid\{\bX_l\}_{l=1}^N,\{Y_l\}_{l\in\Jsc}\right]=0$. Since $\varepsilon_i\in[-2M, 2M]$ for all $i\in\{i\in\mathcal I:\bX_i\in R\}$ and \eqref{p2}, by Theorem 2 of \cite{hoeffding1963probability}, for any $\zeta>0$,
\begin{align*}
	\P_{\S_\Isc}\left(\left|\frac{1}{\# R}\sum_{i\in\left\{i\in\Isc:\bX_i\in R\right\}}\bd_L^{\top}\bS_L^{-1}\bU_i^L \varepsilon_i\right|\geq\zeta\mid \{\bX_l\}_{l=1}^N,\{Y_l\}_{l\in\Jsc}\right)\leq 2 \exp\left\{-\frac{\Lambda_0^2\# R\zeta^2}{8M^2\dbar^2}\right\}.
\end{align*}
For any $R\in \mathcal{R}_{\mathcal{D},\omega,\epsilon}$, when $\# R=\sum_{i\in\mathcal I}\mathbbm{1}_{\{\bX_i \in R\}}\geq k/2$,
\begin{align*}
	\P_{\S_\Isc}\left(\left|\frac{1}{\# R}\sum_{i\in\left\{i\in\Isc:\bX_i\in R\right\}}\bd_L^{\top}\bS_L^{-1}\bU_i^L \varepsilon_i\right|\geq\zeta\mid \{\bX_l\}_{l=1}^N,\{Y_l\}_{l\in\Jsc}\right)\leq 2 \exp\left\{-\frac{\Lambda_0^2k\zeta^2}{16M^2\dbar^2}\right\}.
\end{align*}
By the tower rule, for any $R\in \mathcal{R}_{\mathcal{D},\omega,\epsilon}$,
\begin{align*}
	\P_{\S_\Isc}\left(\left|\frac{1}{\# R}\sum_{i\in\left\{i\in\Isc:\bX_i\in R\right\}}\bd_L^{\top}\bS_L^{-1}\bU_i^L \varepsilon_i\right|\geq\zeta\mid \# R\geq k/2\right)\leq 2 \exp\left\{-\frac{\Lambda_0^2k\zeta^2}{16M^2\dbar^2}\right\},
\end{align*}
and hence
\begin{align*}
	\P_{\S_\Isc}\left(\left|\frac{1}{\# R}\sum_{i\in\left\{i\in\Isc:\bX_i\in R\right\}}\bd_L^{\top}\bS_L^{-1}\bU_i^L \varepsilon_i\right|\geq\zeta\;\mbox{and}\; \# R\geq k/2\right)\leq 2 \exp\left\{-\frac{\Lambda_0^2k\zeta^2}{16M^2\dbar^2}\right\}.
\end{align*}
By the union bound, we have
\begin{align*}
	&\P_{\S_\Isc}\left( \bigcup_{R\in \mathcal{R}_{\mathcal{D},\omega,\epsilon}}\left\{\left|\frac{1}{\# R}\sum_{i\in\left\{i\in\Isc:\bX_i\in R\right\}}\bd_L^{\top}\bS_L^{-1}\bU_i^L \varepsilon_i\right|\geq\zeta\;\mbox{and}\;\#R\geq k/2\right\}\right)\\
	&\qquad\leq 2 \# \mathcal{R}_{\mathcal{D},\omega,\epsilon} \exp\left\{-\frac{\Lambda_0^2k\zeta^2}{16M^2\dbar^2}\right\}.
\end{align*}
By $k\gg\log(N)$ and \eqref{p3}, conditional on the event $\Asc$, as $N\to \infty$, we have $(\# L-\# \Rbar)/\# L=o(1)$ uniformly for any $\bx\in [0,1]^d$ and $\xi\in\Xi$.
Hence, there exists $n_1\in \mathbb N$ such that whenever $n\geq n_1$,
\begin{equation}\label{bound:Rbar}
\Rbar=\Rbar(\bx,\xi)\;\mbox{satisfies}\;\# \Rbar\geq k/2\;\mbox{for all}\;\bx\in[0,1]^d\;\mbox{and}\;\xi\in\Xi\;\mbox{when}\;\Asc\;\mbox{occurs}.
\end{equation}
Note that $\Rbar\in\mathcal{R}_{\mathcal{D},\omega,\epsilon}$ for all $\bx\in[0,1]^d$ and $\xi\in\Xi$. It follows that
\begin{align*}
	&\P_{\S_\Isc}\left(\sup_{\bx\in[0,1]^d,\xi\in\Xi}\left|\Delta_{3,2}\right|\geq\zeta\cap\Asc\right)\\
	&\quad\leq\P_{\S_\Isc}\left(\sup_{\bx\in[0,1]^d,\xi\in\Xi}\left|\Delta_{3,2}\right|\geq\zeta\;\mbox{and}\;\bigcap_{\bx\in[0,1]^d,\xi\in\Xi}\#\Rbar(\bx,\xi)\geq k/2\right)\\
	&\quad\leq\P_{\S_\Isc}\left( \bigcup_{R\in \mathcal{R}_{\mathcal{D},\omega,\epsilon}}\left\{\left|\frac{1}{\# R}\sum_{i\in\left\{i\in\Isc:\bX_i\in R\right\}}\bd_L^{\top}\bS_L^{-1}\bU_i^L \varepsilon_i\right|\geq\zeta\;\mbox{and}\;\#R\geq k/2\right\}\right)\\
	&\quad\leq 2 \# \mathcal{R}_{\mathcal{D},\omega,\epsilon} \exp\left\{-\frac{\Lambda_0^2k\zeta^2}{16M^2\dbar^2}\right\}.
\end{align*}
Therefore,
\begin{align*}
	&\P_{\S_\Isc}\left(\sup_{\bx\in[0,1]^d,\xi\in\Xi}\left|\Delta_{3,2}\right|\geq\zeta\right)
	\leq\P_{\S_\Isc}\left(\sup_{\bx\in[0,1]^d,\xi\in\Xi}\left|\Delta_{3,2}\right|\geq\zeta\cap\Asc\right)+\P_{\S_\Isc}(\Asc^c)\\
	&\quad\leq 2 \# \mathcal{R}_{\mathcal{D},\omega,\epsilon} \exp\left\{-\frac{\Lambda_0^2k\zeta^2}{16M^2\dbar^2}\right\}+\P_{\S_\Isc}(\Asc^c).
\end{align*}
Let $\zeta=\sqrt{\frac{32M^2\dbar^2\log(\# \mathcal{R}_{\mathcal{D},\omega,\epsilon})}{\Lambda_0^2k}}$. By \eqref{eigen_eq1}, there exists $n_2\in \mathbb N$ such that $\# \mathcal{R}_{\mathcal{D},\omega,\epsilon}\geq2\sqrt n$ whenever $n\geq n_2$. 
By \eqref{eigen_eq3}, provided that $n\geq\max\{n_0, n_1,n_2\}$, we have
\begin{align}
	\P_{\S_\Isc}\left(\sup_{\bx\in[0,1]^d,\xi\in\Xi}\left|\Delta_{3,2}\right|\geq\sqrt{\frac{32M^2\dbar^2\log(\# \mathcal{R}_{\mathcal{D},\omega,\epsilon})}{\Lambda_0^2k}}\right)\leq \frac{2}{\# \mathcal{R}_{\mathcal{D},\omega,\epsilon}}+\frac{1}{\sqrt n}\leq\frac{2}{\sqrt n}.\label{p6}
\end{align}
Combining \eqref{p7} and \eqref{p6} with \eqref{p8}, as $N\to \infty$, 
\begin{align}
	\sup_{\bx\in[0,1]^d}\E_{\xi}\left| \Delta_3(\bx,\xi) \right| =O_p\left(\sqrt{\frac{\log(N)}{k}}\right).\label{delta_3}
\end{align}
Note that
\begin{align*}
	\sup_{\bx\in[0,1]^d} \E_{\xi}\left| \Delta_4(\bx,\xi) \right| =\sup_{\bx\in[0,1]^d}\left[ \E_{\xi}\left|\sum_{i\in\Isc}\omega_i(\bx,\xi)\varepsilon_i\right|\right]\leq \sup_{\bx\in[0,1]^d,\xi\in\Xi} \left|\sum_{i\in\Isc}\omega_i(\bx,\xi)\varepsilon_i\right|.
\end{align*}
Repeating the same procedure as \eqref{delta_3} except replacing $\bd_L^{\top}\bS_L^{-1}\bU_i^L \varepsilon_i $ with $\varepsilon_i $, conditional on the event $\Asc$, as $N\to \infty$, we have
\begin{align}
	\sup_{\bx\in[0,1]^d}\E_{\xi}\left| \Delta_4(\bx,\xi) \right| =O_p\left(\sqrt{\frac{\log(N)}{k}}\right).\label{delta_4}
\end{align}
Combining \eqref{delta_1}, \eqref{delta_2}, \eqref{delta_3}, \eqref{delta_4}, \eqref{eq:cov} with \eqref{delta}, we have
\begin{align*}
	\sup_{\bx\in[0,1]^d}\left|\mhat_\mathrm{LCF}(\bx)-m(\bx)\right|=O_p\left(\sqrt\frac{\log(N)}{k}+ \left(\frac{N}{k}\right)^{-\frac{(q+\beta)\log (1-\alpha)}{d\log (\alpha)}} \right).
\end{align*}
\end{proof}

\section{Proofs of the average treatment effect (ATE) estimation results}

\begin{proof}[Proof of Theorem \ref{thm:ATE}]
	For this proof, it sufficient to check the conditions of Assumptions 2.1 from Theorem 2.1 of \cite{chernozhukov2017double}.
Let $V:=A-\pi^*(\bX)$ and $U:=U_1+U_{0}$ with $U_a:=\mathbbm{1}_{\{A=a\}} (Y(a)-\mu_a^*(\bX))$ for $a\in\{0,1\}$. By the definition of $\pi^*(\bX)$ and $\mu_a^*(\bX)$, we have $\E[V\mid \bX]=0$ and $\E[U_a\mid \bX, A=a]=0$ for $a\in\{0,1\}$. By the law of total probability, we have $\E[U\mid \bX, A]=\E[U_1\mid \bX, A=1] \P(A=1\mid \bX)+\E[U_0\mid \bX, A=0] \P(A=0\mid \bX)=0$. Hence, the condition (i) of Assumptions 2.1 is satisfied. Let $r>4$ be any fixed positive constant. Since $|Y|\leq M$, we have $|\mu_a^*(\bX)|\leq M$, which implies $\{\E[\mu_a^*(\bX)]^r\}^{1/r}\leq M$. By $|Y|\leq M$, we also get $\{\E[Y]^r\}^{1/r}\leq M$. By the triangle inequality and $\mathbbm{1}_{\{A=a\}}\leq1$, we have $|U|\leq2|U_a|\leq2|Y-\mu_a^*(\bX)|\leq2 |Y|+2|\mu_a^*(\bX)|\leq 4M$, which implies $\P(\E[U^2\mid \bX]\leq4M)=1$. Since $\E[\mathbbm{1}_{\{A=a\}}(Y(a)-\mu_a^*)]^2\geq C_0$, we have $\{\E[U]^2\}^{1/2}=\{\E[U_0]^2+\E[U_1]^2\}^{1/2}\geq\sqrt{2C_0}$. 
By overlap condition under Assumption \ref{identification}, we have $\P(|A-\pi^*(\bX)|\geq c_0)=1$, which implies $\{\E[V^2]\}^{1/2}\geq c_0$. Hence, the condition (ii) of Assumptions 2.1 is satisfied.
By Theorem \ref{thm:local_consistency}, we have
\begin{align*}
	\left\{\E_{\bX}\left[\muhat_a^{-k}(\bX)-\mu_a^*(\bX)\right]^2\right\}^{1/2}&=O_p\left(N^{-\frac{(q_a+\beta_a)\log \left((1-\alpha_a)^{-1}\right)}{d\log \left(\alpha_a^{-1}\right)+2(q_a+\beta_a)\log \left((1-\alpha_a)^{-1}\right)}} \right)=o_p(1),\\
	\left\{\E_{\bX}\left[\pihat^{-k}(\bX)-\pi^*(\bX)\right]^2\right\}^{1/2}&=O_p\left(N^{-\frac{(q_2+\beta_2)\log \left((1-\alpha_2)^{-1}\right)}{d\log \left(\alpha_2^{-1}\right)+2(q_2+\beta_2)\log \left((1-\alpha_2)^{-1}\right)}} \right)=o_p(1).
\end{align*}
By $d<2\sqrt{\frac{(q_a+\beta_a)(q_2+\beta_2)\log \left((1-\alpha_a)^{-1}\right)\log \left((1-\alpha_2)^{-1}\right)}{\log \left(\alpha_a^{-1}\right)\log \left(\alpha_2^{-1}\right)}}$ for $a=\{0,1\}$, we have
\begin{align*}
	\frac{(q_a+\beta_a)\log \left((1-\alpha_a)^{-1}\right)}{d\log \left(\alpha_a^{-1}\right)+2(q_a+\beta_a)\log \left((1-\alpha_a)^{-1}\right)}+\frac{(q_2+\beta_2)\log \left((1-\alpha_2)^{-1}\right)}{d\log \left(\alpha_2^{-1}\right)+2(q_2+\beta_2)\log \left((1-\alpha_2)^{-1}\right)}> \frac{1}{2},
\end{align*}
which implies
\begin{align*}
	\left\{\E_{\bX}\left[\muhat_a^{-k}(\bX)-\mu_a^*(\bX)\right]^2\right\}^{1/2}\left\{\E_{\bX}\left[\pihat^{-k}(\bX)-\pi^*(\bX)\right]^2\right\}^{1/2}=o_p(N^{-1/2}).
\end{align*}
By Lemma \ref{lem:pihat}, we have $\P(c_1<\pihat^{-k}(\bX)\leq 1-c_1)=1$ with probability approaching one and some constant $c_1\in(0,1/2)$. Hence, the condition (iii) of Assumptions 2.1 is satisfied.
\end{proof}

\section{Proofs of the auxiliary Lemmas}\label{sec:proof-lemma}

\begin{proof}[Proof of Lemma \ref{lem:sample eigen}]
Choose $\omega$ and $\epsilon$ as in \eqref{par: R}. By Lemma \ref{lem:R cardinality}, there exists some $\widetilde{\mathcal{R}}_{\mathcal{D},\omega,\epsilon}$ satisfying the approximation property as in Lemma \ref{lem:R cardinality} with $\log(\#\widetilde{\mathcal{R}}_{\mathcal{D},\omega,\epsilon})=O(\log(N))$. Therefore, we can choose some $\mathcal{R}_{\mathcal{D},\omega,\epsilon}\supseteq\widetilde{\mathcal{R}}_{\mathcal{D},\omega,\epsilon}$ satisfying $\log(\#\mathcal{R}_{\mathcal{D},\omega,\epsilon})=O(\log(N))$ and $\sqrt n=o(\#\mathcal{R}_{\mathcal{D},\omega,\epsilon})$. Condition on the event $\Asc$ defined in Lemma \ref{lem:event A}, with $\mathcal{R}=\mathcal{R}_{\mathcal{D},\omega,\epsilon}$, $\mu_{min}=\omega$, and $\delta=1/\sqrt n$. By $k\gg\log(N)$, as $N\to \infty$, we have
\begin{align}
\frac{\log(\#\mathcal{R}_{\mathcal{D},\omega,\epsilon})}{k}=O\left(\frac{\log(N)}{k}\right)=o(1)\;\;\text{and}\;\;\frac{\sqrt n}{\#\mathcal{R}_{\mathcal{D},\omega,\epsilon}}=o(1).\label{eigen_eq1}
\end{align}
Thus, the condition \eqref{par:event A} is satisfied. By Lemma \ref{lem:event A}, there exists $n_0\in\mathbb{N}$ such that 
\begin{align}
\P_{\S_\Isc}\left(\Asc\right)\geq 1-\frac{1}{\sqrt n}\;\;\text{for any}\;\;n\geq n_0.\label{eigen_eq3}
\end{align}
Condition on the event $\Asc$ above. For any $\bx\in[0,1]^d$ and $\xi\in\Xi$, under $(\alpha,k)$-regular, and by Corollary 14 of \cite{wager2015adaptive}, we have 
\begin{align}\label{eigen_eq5}
\mu(L(\bx,\xi))\geq \omega.
\end{align}
By Lemmas \ref{lem:R cardinality} and \ref{lem:R lowerbound},
we can choose some $\Rbar:=\Rbar(\bx,\xi)\in \mathcal{R}_{\mathcal{D},\omega,\epsilon}$ as an inner approximation of $L(\bx,\xi)$ satisfying $\Rbar\subseteq L(\bx,\xi)$,
\begin{align}\label{eigen_eq4}
\lambda(L(\bx,\xi))&\overset{(i)}{=}\mu(L(\bx,\xi))\leq \exp\{\epsilon\}\lambda(\Rbar)\overset{(i)}{=}\exp\{\epsilon\}\mu(\Rbar),\;\;\text{and}\;\\\
\frac{\# L-\# \Rbar}{\# L}&\leq \frac{3}{\sqrt k}+2\sqrt{\frac{3\log(\#\mathcal{R}_{\mathcal{D},\omega,\epsilon})}{\# L}}+O\left(\frac{\log(\#\mathcal{R}_{\mathcal{D},\omega,\epsilon})}{\# L}\right),\label{eigen_eq2}
\end{align}
where we denote $\#L:=\#L(\bx,\xi)$ for the sake of simplicity and (i) holds since $\bX_i\sim\mathrm{Uniform}[0,1]^d$. Define $\omega_{i}^L:=\omega_i(\bx,\xi)=\mathbbm{1}_{\left \{\bX_i\in L(\bx,\xi)\right\}}/\#L$ and $\omega_{i}^{R}:=\mathbbm{1}_{\left \{\bX_i\in R\right\}}/\# R$ for any $R \in \mathcal{R}_{\mathcal{D},\omega,\epsilon}$.
Note that
\begin{align*}
\bS_L-\bd_L \bd_L^{\top}=	\sum_{i\in\Isc}\omega_i(\bx,\xi)\bU_i^L (\bU_i^L)^{\top}-\sum_{i\in\Isc}\omega_i(\bx,\xi)\bU_i^L \sum_{i\in\Isc}\omega_i(\bx,\xi)(\bU_i^L)^\top=\sum_{i=1}^4 \bQ_i,
\end{align*}
where
\begin{align*}
\bQ_1&:=\sum_{i\in\Isc}\omega_{i}^L\bU_i^L (\bU_i^L)^{\top}-\sum_{i\in\Isc}\omega_{i}^L\bU_i^L \sum_{i\in\Isc}\omega_{i}^L(\bU_i^L)^\top\\
&\qquad-\sum_{i\in\Isc}\omega_{i}^{\Rbar}\bU_i^L (\bU_i^L)^{\top}+\sum_{i\in\Isc}\omega_{i}^{\Rbar}\bU_i^L \sum_{i\in\Isc}\omega_{i}^{\Rbar}(\bU_i^L)^\top,\\
\bQ_2&:=\sum_{i\in\Isc}\omega_{i}^{\Rbar}\bU_i^L (\bU_i^L)^{\top}-\sum_{i\in\Isc}\omega_{i}^{\Rbar}\bU_i^L \sum_{i\in\Isc}\omega_{i}^{\Rbar}(\bU_i^L)^\top-\Var\left(\bU \mid \bX \in \Rbar\right),\\	
\bQ_3&:=\Var\left(\bU^L \mid \bX \in \Rbar\right)-\Var\left(\bU^L \mid \bX \in L(\bx, \xi)\right),\\
\bQ_4&:=\Var\left(\bU^L \mid \bX \in L(\bx, \xi)\right),
\end{align*}
where $\bU$ and $\bU^L$ are independent copies of $\bU_i$ and $\bU_i^L$, respectively. By the triangle inequality,
\begin{align*}
&\inf_{\bx\in[0,1]^d,\xi\in\Xi}\Lambda_{\min}\left(\bS_L-\bd_L \bd_L^{\top}\right)=\inf_{\bx\in[0,1]^d,\xi\in\Xi,\|\ba\|_2=1}\sum_{i=1}^4 \ba^\top\bQ_i\ba\\
&\qquad\geq\inf_{\bx\in[0,1]^d,\xi\in\Xi,\|\ba\|_2=1}\ba^\top\bQ_4\ba-\sum_{i=1}^3 \sup_{\bx\in[0,1]^d,\xi\in\Xi,\|\ba\|_2=1}\left|\ba^\top\bQ_i\ba\right|.
\end{align*}
In the following, we show that there exists some constant $\Lambda_0>0$ such that
\begin{align*}
\inf_{\bx\in[0,1]^d,\xi\in\Xi}\Lambda_{\min}\left(\bS_L-\bd_L \bd_L^{\top}\right)\geq\Lambda_0,
\end{align*}
with probability approaching one as $N\to \infty$.

\textbf{Step 1.} We first demonstrate that on the event $\mathcal A$, as $N\to\infty$,
\begin{align}
\sup_{\bx\in[0,1]^d,\xi\in\Xi,\|\ba\|_2=1}\left|\ba^\top\bQ_1\ba\right|=o(1).\label{rate_Q1}
\end{align}
By the triangle inequality, we have 
\begin{align}\label{Q1_eq1}
\sup_{\bx\in[0,1]^d,\xi\in\Xi,\|\ba\|_2=1}\left|\ba^\top\bQ_1\ba\right|\leq \sum_{j=1}^2\sup_{\bx\in[0,1]^d,\xi\in\Xi,\|\ba\|_2=1}\left|\ba^\top\bQ_{1,j}\ba\right|,
\end{align}
where for any $\bx\in[0,1]^d$ and $\xi\in\Xi$,
\begin{align*}
\bQ_{1,1}&:=\bQ_{1,1}(\bx,\xi)=\sum_{i\in\Isc}\omega_{i}^L\bU_i^L (\bU_i^L)^{\top}-\sum_{i\in\Isc}\omega_{i}^{\Rbar}\bU_i^L (\bU_i^L)^{\top},\\	
\bQ_{1,2}&:=\bQ_{1,1}(\bx,\xi)=\sum_{i\in\Isc}\omega_{i}^L\bU_i^L \sum_{i\in\Isc}\omega_{i}^L(\bU_i^L)^\top-\sum_{i\in\Isc}\omega_{i}^{\Rbar}\bU_i^L \sum_{i\in\Isc}\omega_{i}^{\Rbar}(\bU_i^L)^\top.
\end{align*}
Note that $\Rbar\subseteq L(\bx,\xi)$, for any $\ba\in\R^{\dbar}$, we have
\begin{align*}
\left|\ba^\top\bQ_{1,1}\ba\right|
&\leq\left|\frac{1}{\# L}\sum_{i\in\left\{i\in\Isc:\bX_i\in \Rbar\right\}}\left(\ba^\top\bU_i^L \right)^2-\frac{1}{\# \Rbar}\sum_{i\in\left\{i\in\Isc:\bX_i\in \Rbar\right\}}\left(\ba^\top\bU_i^L \right)^2\right|\\
&\qquad+\left|\frac{1}{\# L}\sum_{i\in\left\{i\in\Isc:\bX_i\in L(\bx, \xi)\setminus \Rbar\right\}}\left(\ba^\top\bU_i^L \right)^2\right|\\
&\leq\frac{2(\# L-\# \Rbar)}{\# L}\sup_{i\in\left\{i\in\Isc:\bX_i\in L(\bx, \xi)\right\}}\left(\ba^\top\bU_i^L \right)^2.
\end{align*}
For any $\bx\in[0,1]^d$ and $\xi\in\Xi$, if $\omega_i(\bx,\xi)\neq0$, i.e., $\bX_i \in L(\bx, \xi)$, we have $(\bX_{ij}-\bx_j)/\mathrm{diam}_j(L(\bx,\xi))\in [-1,1]$. By the construction of $\bU_i^L$,
\begin{align}\label{bound_U}
\|\bU_i^L\|_2\leq\sqrt{\dbar}\|\bU_i^L\|_\infty\leq\sqrt{\dbar},\;\;\forall i\in\{i\in\mathcal I:\omega_i(\bx,\xi)\neq0\},
\end{align}
where $\dbar= \sum_{i=1}^q d^i$. Hence, it follows that
\begin{align}\label{Q1_eq3}
&\sup_{i\in\left\{i\in\Isc:\bX_i\in L(\bx, \xi)\right\},\|\ba\|_2=1}\left\{\left(\ba^\top\bU_i^L \right)^2\right\}=\sup_{i\in\left\{i\in\Isc:\bX_i\in L(\bx, \xi)\right\}} \left\|\bU_i^L \right\|_2^2\leq\dbar.
\end{align}
Therefore,
\begin{align}
\sup_{\bx\in[0,1]^d,\xi\in\Xi,\|\ba\|_2=1}\left|\ba^\top\bQ_{1,1}\ba\right|	&\leq 2\dbar\sup_{\bx\in[0,1]^d,\xi\in\Xi}\left\{ \frac{\# L-\# \Rbar}{\# L}\right\}.\label{Q1_eq2}
\end{align}
By \eqref{eigen_eq1} and \eqref{eigen_eq2}, as $N\to \infty$, we have $(\# L-\# \Rbar)/\# L=o(1)$ for any $\bx\in [0,1]^d$ and $\xi\in\Xi$, which implies that
\begin{align}\label{Q1_eq5}
\sup_{\bx\in[0,1]^d,\xi\in\Xi}\left\{ \frac{\# L-\# \Rbar}{\# L}\right\}=o(1).
\end{align}
By \eqref{Q1_eq2}, we have
\begin{align}\label{rate_Q11}
\sup_{\bx\in[0,1]^d,\xi\in\Xi,\|\ba\|_2=1}\left|\ba^\top\bQ_{1,1}\ba\right|=o(1).
\end{align}
In addition, note that for any $\ba\in\R^{\dbar}$,
\begin{align*}
\left|\ba^\top\bQ_{1,2}\ba\right|&=\left|\frac{1}{\# L}\sum_{i\in\left\{i\in\Isc:\bX_i\in L(\bx, \xi)\right\}}\ba^\top\bU_i^L +\frac{1}{\# \Rbar}\sum_{i\in\left\{i\in\Isc:\bX_i\in \Rbar\right\}}\ba^\top\bU_i^L \right|\\
&\qquad\cdot\left|\frac{1}{\# L}\sum_{i\in\left\{i\in\Isc:\bX_i\in L(\bx, \xi)\right\}}\ba^\top\bU_i^L -\frac{1}{\# \Rbar}\sum_{i\in\left\{i\in\Isc:\bX_i\in \Rbar\right\}}\ba^\top\bU_i^L \right|.
\end{align*}
By the triangle inequality, we have
\begin{align*}
&\left|\frac{1}{\# L}\sum_{i\in\left\{i\in\Isc:\bX_i\in L(\bx, \xi)\right\}}\ba^\top\bU_i^L -\frac{1}{\# \Rbar}\sum_{i\in\left\{i\in\Isc:\bX_i\in \Rbar\right\}}\ba^\top\bU_i^L \right|\\
&\qquad\leq \left|\frac{1}{\# L}\sum_{i\in\left\{i\in\Isc:\bX_i\in \Rbar\right\}}\ba^\top\bU_i^L -\frac{1}{\# \Rbar}\sum_{i\in\left\{i\in\Isc:\bX_i\in \Rbar\right\}}\ba^\top\bU_i^L \right|
+\left|\frac{1}{\# L}\sum_{i\in\left\{i\in\Isc:\bX_i\in L(\bx, \xi)\setminus \Rbar\right\}}\ba^\top\bU_i^L \right|\\
&\qquad\leq\frac{2(\# L-\# \Rbar)}{\# L}\sup_{i\in\left\{i\in\Isc:\bX_i\in L(\bx, \xi)\right\} } \left|\ba^\top\bU_i^L \right|.
\end{align*}
Besides, we also have 
\begin{align*}
\left|\frac{1}{\# L}\sum_{i\in\left\{i\in\Isc:\bX_i\in L(\bx, \xi)\right\}}\ba^\top\bU_i^L +\frac{1}{\# \Rbar}\sum_{i\in\left\{i\in\Isc:\bX_i\in \Rbar\right\}}\ba^\top\bU_i^L \right|\leq 2\sup_{i\in\left\{i\in\Isc:\bX_i\in L(\bx, \xi)\right\}}\left|\ba^\top\bU_i^L \right|.
\end{align*}
Therefore,
\begin{align*}
\left|\ba^\top\bQ_{1,2}\ba\right|\leq\frac{4(\# L-\# \Rbar)}{\# L}\sup_{i\in\left\{i\in\Isc:\bX_i\in L(\bx, \xi)\right\} } \left|\ba^\top\bU_i^L \right|^2.
\end{align*}
By \eqref{Q1_eq3} and \eqref{Q1_eq5}, we have
\begin{align}\label{rate_Q12}
\sup_{\bx\in[0,1]^d,\xi\in\Xi,\|\ba\|_2=1}\left|\ba^\top\bQ_{1,2}\ba\right|=o(1).
\end{align}
Combining \eqref{rate_Q11} and \eqref{rate_Q12} with \eqref{Q1_eq1}, we conclude that \eqref{rate_Q1} holds.

\textbf{Step 2.} We now demonstrate that on the event $\Asc$, as $N\to\infty$,
\begin{align}
\sup_{\bx\in[0,1]^d,\xi\in\Xi,\|\ba\|_2=1}\left|\ba^\top\bQ_3\ba\right|=o(1).\label{rate_Q3}
\end{align}
By the triangle inequality, we have 
\begin{align}\label{Q3_eq1}
\sup_{\bx\in[0,1]^d,\xi\in\Xi,\|\ba\|_2=1}\left|\ba^\top\bQ_3\ba\right|\leq\sum_{j=1}^2\sup_{\bx\in[0,1]^d,\xi\in\Xi,\|\ba\|_2=1}\left|\ba^\top\bQ_{3,j}\ba\right|,
\end{align}
where for any $\bx\in[0,1]^d$ and $\xi\in\Xi$,
\begin{align*}
\bQ_{3,1}:=\bQ_{3,1}(\bx,\xi)=&\E\left(\bU^L (\bU^L)^\top\mid \bX \in \Rbar\right)-\E\left(\bU^L (\bU^L)^\top \mid \bX \in L(\bx, \xi)\right),\\	
\bQ_{3,2}:=\bQ_{3,1}(\bx,\xi)=&\E\left(\bU^L \mid \bX \in \Rbar\right)\E\left((\bU^L)^\top \mid \bX \in \Rbar\right)-\\
&\qquad\E\left(\bU^L \mid \bX \in L(\bx, \xi)\right)\E\left((\bU^L)^\top \mid \bX \in L(\bx, \xi)\right).
\end{align*}
Let $\mu_L=\E\left[\mathbbm{1}_{\left\{ \bX \in L(\bx, \xi)\right\}}\right]$ and $\mu_{R}=\E\left[\mathbbm{1}_{\left\{ \bX \in R\right\}}\right]$ for any $R \in \mathcal{R}_{\mathcal{D},\omega,\epsilon}$. By the triangle inequality,
\begin{align*}
\left|\ba^\top\bQ_{3,1}\ba\right|
&\leq\left|\frac{1}{\mu_L}\E\left[\left(\ba^\top\bU^L\right)^2 \mathbbm{1}_{\left\{ \bX \in \Rbar\right\}}\right] -\frac{1}{\mu_{\Rbar}}\E\left[\left(\ba^\top\bU^L\right)^2 \mathbbm{1}_{\left\{ \bX \in \Rbar\right\}}\right]\right|\\
&\qquad+\left|\frac{1}{\mu_L}\E\left[\left(\ba^\top\bU^L\right)^2 \mathbbm{1}_{\left\{ \bX \in L(\bx, \xi)\setminus \Rbar\right\}}\right] \right|\\
&\leq\frac{2(\mu_L-\mu_{\Rbar})}{\mu_L}\sup_{\bX\in[0,1]^d} \left(\ba^\top\bU^L \right)^2\mathbbm{1}_{\{\bX\in L(\bx, \xi)\}}.
\end{align*}
Similarly as in \eqref{Q1_eq3}, we also have 
\begin{align*}
\sup_{\bX\in[0,1]^d,\|\ba\|_2=1} \left(\ba^\top\bU^L \right)^2\mathbbm{1}_{\{\bX\in L(\bx, \xi)\}}\leq\dbar.
\end{align*}
Hence,
\begin{align}\label{Q3_eq4}
\sup_{\bx\in[0,1]^d,\xi\in\Xi,\|\ba\|_2=1}\left|\ba^\top\bQ_{3,1}\ba\right|
\leq 2\dbar\sup_{\bx\in[0,1]^d,\xi\in\Xi}\left\{\frac{\mu_L-\mu_{\Rbar}}{\mu_L}\right\}.
\end{align}
By \eqref{eigen_eq4} since $\bX_i\sim\mathrm{Uniform}[0,1]^d$, we have $\mu_{\Rbar}\geq \exp\{-\epsilon\}\mu_L$ for any $\bx\in[0,1]^d$ and $\xi\in\Xi$, which implies that
\begin{align}\label{Q3_eq3}
\sup_{\bx\in[0,1]^d,\xi\in\Xi}\left\{\frac{\mu_L-\mu_{\Rbar}}{\mu_L}\right\}\leq 1-\exp\{-\epsilon\}= 1-\exp\{-1/\sqrt k\}=o(1)
\end{align}
as $N\to\infty$. Together with \eqref{Q3_eq4}, we have
\begin{align}
\sup_{\bx\in[0,1]^d,\xi\in\Xi,\|\ba\|_2=1}\left|\ba^\top\bQ_{3,1}\ba\right|=o(1).\label{rate_Q31}
\end{align}
In addition, for any $\ba\in\R^{\dbar}$,
\begin{align*}
\left|\ba^\top\bQ_{3,2}\ba\right|&=\left|\E\left(\ba^\top\bU^L \mid \bX \in L(\bx, \xi)\right)+\E\left(\ba^\top\bU^L \mid \bX \in \Rbar\right) \right|\\
&\qquad\cdot\left|\E\left(\ba^\top\bU^L \mid \bX \in L(\bx, \xi)\right)-\E\left(\ba^\top\bU^L \mid \bX \in \Rbar\right)\right|.
\end{align*}
By the triangle inequality, we have
\begin{align*}
&\left|\E\left(\ba^\top\bU^L \mid \bX \in L(\bx, \xi)\right)-\E\left(\ba^\top\bU^L \mid \bX \in \Rbar\right) \right|\\
&\qquad\leq 
\left|\frac{1}{\mu_L}\E\left[\ba^\top\bU^L \mathbbm{1}_{\left\{ \bX \in \Rbar\right\}}\right] -\frac{1}{\mu_{\Rbar}}\E\left[\ba^\top\bU^L \mathbbm{1}_{\left\{ \bX \in \Rbar\right\}}\right]\right|\\
&\qquad\qquad+\left|\frac{1}{\mu_L}\E\left[\ba^\top\bU^L \mathbbm{1}_{\left\{ \bX \in L(\bx, \xi)\setminus \Rbar\right\}}\right] \right|\\
&\qquad\leq\frac{2(\mu_L-\mu_{\Rbar})}{\mu_L}\sup_{\bX\in[0,1]^d} \left|\ba^\top\bU^L \right|\mathbbm{1}_{\{\bX\in L(\bx, \xi)\}}.
\end{align*}
Besides, we also have
\begin{align*}
\left|\E\left(\ba^\top\bU^L \mid \bX \in L(\bx, \xi)\right)+\E\left(\ba^\top\bU^L \mid \bX \in \Rbar\right)\right|\leq 2\sup_{\bX\in[0,1]^d} \left|\ba^\top\bU^L \right|\mathbbm{1}_{\{\bX\in L(\bx, \xi)\}}.
\end{align*}
Therefore,
\begin{align*}
\left|\ba^\top\bQ_{3,2}\ba\right|\leq\frac{4(\mu_L-\mu_{\Rbar})}{\mu_L}\sup_{\left\{i:\bX_i\in L(\bx, \xi)\right\} } \left|\ba^\top\bU_i^L \right|^2.
\end{align*}
By \eqref{Q1_eq3} and \eqref{Q3_eq3}, we have
\begin{align}
\sup_{\bx\in[0,1]^d,\xi\in\Xi,\|\ba\|_2=1}\left|\ba^\top\bQ_{3,2}\ba\right|
=o(1).\label{rate_Q32}
\end{align}
Combining \eqref{rate_Q31} and \eqref{rate_Q32} with \eqref{Q3_eq1}, we conclude that \eqref{rate_Q3} holds.

\textbf{Step 3.} We next demonstrate that condition on the event $\Asc$, as $N\to\infty$, with probability at least $1-4/\sqrt n$,
\begin{align}
\sup_{\bx\in[0,1]^d,\xi\in\Xi,\|\ba\|_2=1}\left|\ba^\top\bQ_2\ba\right|=o(1).\label{rate_Q2}
\end{align}

For any $i\in\mathcal I$ and $j\leq d$, define $\bU_i^{R}:=\left(Z_{i1}^{R},\dots,Z_{id}^{R},(Z_{i1}^{R})^2,Z_{i1}^{R}Z_{i2}^{R},\dots,(Z_{id}^{R})^2,\dots,(Z_{id}^{R})^q\right)^\top$ with $Z_{ij}^{R}:=(\bX_{ij}-\bx_j)/\mathrm{diam}_j(R)$ for any $R \in \mathcal{R}_{\mathcal{D},\omega,\epsilon}$. Then, we have
\begin{align}\label{def:Ubar}
\bU_i^L=\bD_{\Rbar}\bU_i^{\Rbar},
\end{align}
where $\bD_{R}\;:=\;\mathrm{diag}\biggl(\frac{\mathrm{diam}_1(R)}{\mathrm{diam}_1(L(\bx,\xi))},\;\dots,\;\frac{\mathrm{diam}_d(R)}{\mathrm{diam}_d(L(\bx,\xi))},\;\frac{\mathrm{diam}_1^2(R)}{\mathrm{diam}_1^2(L(\bx,\xi))},\;\frac{\mathrm{diam}_1(R)\mathrm{diam}_2(R)}{\mathrm{diam}_1(L(\bx,\xi))\mathrm{diam}_2(L(\bx,\xi))}
,\;\dots,$ $\frac{\mathrm{diam}_2^2(R)}{\mathrm{diam}_1^2(L(\bx,\xi))},\dots,\frac{\mathrm{diam}_d^q(R)}{\mathrm{diam}_d^q(L(\bx,\xi))}\biggl)$ for any $R \in \mathcal{R}_{\mathcal{D},\omega,\epsilon}$. For any $R \in \mathcal{R}_{\mathcal{D},\omega,\epsilon}$, denote $R_j=[r_j^-,r_j^+]\subseteq[0,1]$ as the interval of the $j$-axis of the rectangle $R$ for each $1\leq j\leq d$.
Define 
$\Vbar_{ij}^{R} =(\bX_{ij}-r_j^-)/\mathrm{diam}_j(R)$ and $\cbar_{R,j}=(r_j^- -x_j)/\mathrm{diam}_j(R)$ for any $1\leq j\leq d$ and $R \in \mathcal{R}_{\mathcal{D},\omega,\epsilon}$. Then, the $\dbar$-dimensional vector $\bU_i^{R}$ for any $R \in \mathcal{R}_{\mathcal{D},\omega,\epsilon}$ can be represented as 
$\bU_i^{R}=(\Vbar_{i1}^{R}+\cbar_{R,1},\dots,\Vbar_{id}^{R}+\cbar_{R,d},(\Vbar_{i1}^{R}+\cbar_{R,1})^2,(\Vbar_{i1}^{R}+\cbar_{R,1})(\Vbar_{i2}^{R}+\cbar_{R,2}),\dots,(\Vbar_{id}^{R}+\cbar_{R,d})^2,\dots,(\Vbar_{id}^{R}+\cbar_{R,d})^q)^\top.$
Note that there exists some $\dbar\times \dbar$ lower triangular matrix $\bP_{\Rbar}$ with $1$ on main diagonal such that 
\begin{align}\label{def:Vbar}
\bU_i^{\Rbar}=\bP_{\Rbar} \bV_i^{\Rbar}+\bC_{\Rbar},
\end{align}
where $\bV_i^{R}\;:=(V_{i1}^{R},\dots,V_{id}^{R},(V_{i1}^{R})^2,V_{i1}^{R}V_{i2}^{R},\dots,(V_{id}^{R})^2,\dots,(V_{id}^{R})^q)^\top$ and $\bC_{R}:=(c_{R,1},\dots,$ $c_{R,d},c_{R,1}^2,c_{R,1}c_{R,2},\dots,c_{R,d}^2,\dots,c_{R,d}^q)^\top$ for any $R \in \mathcal{R}_{\mathcal{D},\omega,\epsilon}$. By \eqref{def:Ubar} and \eqref{def:Vbar}, we have 
$$\bU_i^L=\bD_{\Rbar}\bP_{\Rbar} \bV_i^{\Rbar}+\bD_{\Rbar}\bC_{\Rbar}.$$
 Let $\bV^{R}$ be an independent copy of $\bV_i^{R}$. Then, we can express $\bQ_2$ as
\begin{align*}
\bQ_2=\bD_{\Rbar}\bP_{\Rbar}\bQ_{\Rbar}\bP_{\Rbar}^\top\bD_{\Rbar},
\end{align*}
where $\bQ_{R}:=\sum_{i\in\Isc}\omega_{i}^{R}\bV_{i}^{R}(\bV_{i}^{R})^{\top}-\sum_{i\in\Isc}\omega_{i}^{R}\bV_{i}^{R}\sum_{i\in\Isc}\omega_{i}^{R}(\bV_{i}^{R})^\top-\Var\left(\bV^{R} \mid \bX \in R\right)$ for any $R \in \mathcal{R}_{\mathcal{D},\omega,\epsilon}$. By the sub-multiplicative property of matrix norm, we have
\begin{align*}
\|\bQ_2\|_2\leq\|\bD_{\Rbar}\|_2^2\|\bP_{\Rbar}\|_2^2\|\bQ_{\Rbar}\|_2.
\end{align*}
Since $\bD_{\Rbar}$ is a diagonal matrix and its largest eigenvalue is smaller than $1$, we have
$\|\bD_{\Rbar}\|_2^2\leq1$. In addition, since the eigenvalues of an lower triangular matrix are the diagonal entries of the matrix, we also have
$\|\bP_{\Rbar}\|_2^2=1$. Then, we have
\begin{align*}
&\sup_{\bx\in[0,1]^d,\xi\in\Xi,\|\ba\|_2=1}\left|\ba^\top\bQ_2\ba\right|=\sup_{\bx\in[0,1]^d,\xi\in\Xi}\|\bQ_2\|_2\leq\sup_{\bx\in[0,1]^d,\xi\in\Xi}\|\bQ_{\Rbar}\|_2.
\end{align*}
Let $\mathbf{m}_{R}:=\E\left[\bV^R\mid \bX \in R\right]$ for any $R \in \mathcal{R}_{\mathcal{D},\omega,\epsilon}$.
By the triangle inequality, we have 
\begin{align*}
	&\sup_{\bx\in[0,1]^d,\xi\in\Xi}\left\|\bQ_{R}\right\|_2\leq\sup_{\bx\in[0,1]^d,\xi\in\Xi}\left\|\bQ_{R,1}\right\|_2 + \sup_{\bx\in[0,1]^d,\xi\in\Xi}\left\|\bQ_{R,2}\right\|_2,
\end{align*}
where for any $R \in \mathcal{R}_{\mathcal{D},\omega,\epsilon}$,
\begin{align*}
	\bQ_{R,1}&:=\sum_{i\in\Isc}\omega_{i}^R\left(\bV_{i}^R-\mathbf{m}_{R}\right)\left(\bV_{i}^R-\mathbf{m}_{R}\right)^{\top}- \Var\left(\bV^R\mid \bX \in R\right),\\	
	\bQ_{R,2}&:=\left(\sum_{i\in\Isc}\omega_{i}^R\bV_{i}^R-\mathbf{m}_{R}\right)\left(\sum_{i\in\Isc}\omega_{i}^R\bV_{i}^R-\mathbf{m}_{R}\right)^\top.
\end{align*}
Therefore, we have
\begin{align}\label{Q2_eq1}
	\sup_{\bx\in[0,1]^d,\xi\in\Xi,\|\ba\|_2=1}\left|\ba^\top\bQ_2\ba\right|\leq\sup_{\bx\in[0,1]^d,\xi\in\Xi}\left\|\bQ_{\Rbar,1}\right\|_2 + \sup_{\bx\in[0,1]^d,\xi\in\Xi}\left\|\bQ_{\Rbar,2}\right\|_2.
\end{align}

For all $i\in\{i\in\mathcal I:\bX_i\in R\}$ and $j \leq d$, we have $\bV_{ij}^R\in[0,1]$, where $\bV_{ij}^R$ denotes the $j$-th coordinate of $\bV_i^R$. Note that $(\bV_{i}^R)_{i\in\mathcal I:\bX_i \in R}$ are i.i.d. random vectors condition on the indicators $\left\{\mathbbm{1}_{\left\{\bX_i \in R\right\}}\right\}_{i\in\Isc}$. As shown in Example 2.4 of \cite{wainwright2019high}, condition on $\left\{\mathbbm{1}_{\left\{\bX_i \in R\right\}}\right\}_{i\in\Isc}$, $(\bV_{ij}^R)_{i\in\mathcal I:\bX_i \in R,j\leq d}$ are sub-Gaussian with parameter at most $\sigma=1$. Note that $\E[\bQ_{R,1}\mid\left\{\mathbbm{1}_{\left\{\bX_i \in R\right\}}\right\}_{i\in\Isc}]=\bzero$. By Theorem 6.5 of \cite{wainwright2019high}, for all $\zeta_1\geq0$ and any $R\in \mathcal{R}_{\mathcal{D},\omega,\epsilon}$,
\begin{align*}
&\P_{\S_\Isc}\left(\left\|\bQ_{R,1}\right\|_2\geq C_1\left(\sqrt{\frac{\dbar}{\# R}}+\frac{\dbar}{\# R}\right)+\zeta_1 \mid \left\{\mathbbm{1}_{\left\{\bX_i \in R\right\}}\right\}_{i\in\Isc}\right)\\
&\qquad\leq C_2 \exp\left\{-C_3\# R \min \{\zeta_1,\zeta_1^2\}\right\},
\end{align*}
where $C_1$, $C_2$, and $C_3$ are some positive constants. Then,
\begin{align*}
\P_{\S_\Isc}\left(\left\|\bQ_{R,1}\right\|_2\geq C_1\left(\sqrt{\frac{2\dbar}{k}}+\frac{2\dbar}{k}\right)+\zeta_1 \mid \#R\geq k/2\right)\leq C_2 \exp\left\{-\frac{C_3}{2} k \min \{\zeta_1,\zeta_1^2\}\right\},
\end{align*}
By the union bound, we have
\begin{align*}
&\P_{\S_\Isc}\left( \bigcup_{R\in \mathcal{R}_{\mathcal{D},\omega,\epsilon}}\left\{\left\|\bQ_{R,1}\right\|_2\geq C_1\left(\sqrt{\frac{2\dbar}{k}}+\frac{2\dbar}{k}\right)+\zeta_1\;\mbox{and}\;\#R\geq k/2\right\}\right)\\
&\qquad\leq C_2 \# \mathcal{R}_{\mathcal{D},\omega,\epsilon} \exp\left\{-\frac{C_3}{2} k \min \{\zeta_1,\zeta_1^2\}\right\}.
\end{align*}
As shown in \eqref{bound:Rbar}, when $\Asc$ occurs and $n\geq n_1$, we have $\#\Rbar\geq k/2$ for all $\bx\in[0,1]^d$ and $\xi\in\Xi$. In addition, we note that the chosen $\Rbar$ satisfies $\Rbar=\Rbar(\bx,\xi)\in\mathcal{R}_{\mathcal{D},\omega,\epsilon}$ for all $\bx\in[0,1]^d$ and $\xi\in\Xi$. It follows that
\begin{align}\label{eq_G.23}
&\P_{\S_\Isc}\left(\sup_{\bx\in[0,1]^d,\xi\in\Xi}\left\|\bQ_{\Rbar,1}\right\|_2\geq C_1\left(\sqrt{\frac{2\dbar}{k}}+\frac{2\dbar}{k}\right)+\zeta_1\cap\Asc\right)\\
&\quad\leq\P_{\S_\Isc}\left(\sup_{\bx\in[0,1]^d,\xi\in\Xi}\left\|\bQ_{\Rbar,1}\right\|_2\geq C_1\left(\sqrt{\frac{2\dbar}{k}}+\frac{2\dbar}{k}\right)+\zeta_1\;\mbox{and}\;\bigcap_{\bx\in[0,1]^d,\xi\in\Xi}\#\Rbar(\bx,\xi)\geq k/2\right)\nonumber\\
&\quad\leq\P_{\S_\Isc}\left( \bigcup_{R\in \mathcal{R}_{\mathcal{D},\omega,\epsilon}}\left\{\left\|\bQ_{R,1}\right\|_2\geq C_1\left(\sqrt{\frac{2\dbar}{k}}+\frac{2\dbar}{k}\right)+\zeta_1\;\mbox{and}\;\#R\geq k/2\right\}\right)\nonumber\\
&\quad\leq C_2 \# \mathcal{R}_{\mathcal{D},\omega,\epsilon} \exp\left\{-\frac{C_3}{2} k \min \{\zeta_1,\zeta_1^2\}\right\}.\nonumber
\end{align}
Therefore,
\begin{align*}
&\P_{\S_\Isc}\left(\sup_{\bx\in[0,1]^d,\xi\in\Xi}\left\|\bQ_{\Rbar,1}\right\|_2\geq C_1\left(\sqrt{\frac{2\dbar}{k}}+\frac{2\dbar}{k}\right)+\zeta_1\right)\\
&\quad\leq\P_{\S_\Isc}\left(\sup_{\bx\in[0,1]^d,\xi\in\Xi}\left\|\bQ_{\Rbar,1}\right\|_2\geq C_1\left(\sqrt{\frac{2\dbar}{k}}+\frac{2\dbar}{k}\right)+\zeta_1\cap\Asc\right)+\P_{\S_\Isc}(\Asc^c)\\
&\quad\leq C_2 \# \mathcal{R}_{\mathcal{D},\omega,\epsilon} \exp\left\{-\frac{C_3}{2} k \min \{\zeta_1,\zeta_1^2\}\right\}+\P_{\S_\Isc}(\Asc^c).
\end{align*}
Let $\zeta_1=\sqrt{\frac{4\log(\# \mathcal{R}_{\mathcal{D},\omega,\epsilon})}{C_3 k}}$. By \eqref{eigen_eq1}, there exists $n_2\in \mathbb N$ such that $\zeta_1^2\leq \zeta_1$ and $\# \mathcal{R}_{\mathcal{D},\omega,\epsilon}\geq\max\{C_2,\dbar/2\}\sqrt n$ whenever $n\geq n_2$. 
By \eqref{eigen_eq3}, provided that $n\geq\max\{n_0, n_1,n_2\}$, we have
\begin{align}
&\P_{\S_\Isc}\left(\sup_{\bx\in[0,1]^d,\xi\in\Xi}\left\|\bQ_{\Rbar,1}\right\|_2\geq C_1\left(\sqrt{\frac{2\dbar}{k}}+\frac{2\dbar}{k}\right)+\sqrt{\frac{4\log(\# \mathcal{R}_{\mathcal{D},\omega,\epsilon})}{C_3 k}}\right)\nonumber\\
&\quad\leq C_2/\# \mathcal{R}_{\mathcal{D},\omega,\epsilon}+1/\sqrt n \leq 2/\sqrt n.\label{Q2_eq3}
\end{align}
Additionally, note that for any $R\in \mathcal{R}_{\mathcal{D},\omega,\epsilon}$, 
\begin{align}
\left\|\bQ_{R,2}\right\|_2
&=\left\|\sum_{i\in\Isc}\omega_{i}^R\bV_{i}^R-\mathbf{m}_{R}\right\|_2^2=\sum_{j=1}^{\dbar} \left(\frac{1}{\# R}\sum_{i\in\left\{i\in\Isc:\bX_i\in R\right\}}\bV_{ij}^R-\mathbf{m}_{R,j}\right)^2,\label{Q2_eq2}
\end{align}
where $\mathbf{m}_{R,j}$ is the $j$-th coordinate of $\mathbf{m}_R$.
Since $\bV_{ij}^R\in[0,1]$ for all $i\in\{i\in\mathcal I:\bX_i\in R\}$ and $j\leq d$, by Theorem 2 of \cite{hoeffding1963probability}, for any $j \leq \dbar$ and $\zeta_2>0$,
\begin{align*}
\P_{\S_\Isc}\left(\left|\frac{1}{\# R}\sum_{i\in\left\{i\in\Isc:\bX_i\in R\right\}}\bV_{ij}^R-\mathbf{m}_{R,j}\right|\geq\zeta_2\mid \left\{\mathbbm{1}_{\left\{\bX_i \in R\right\}}\right\}_{i\in\Isc}\right)\leq 2 \exp\left\{-2\# R\zeta_2^2\right\}.
\end{align*}
For any $R\in \mathcal{R}_{\mathcal{D},\omega,\epsilon}$, $j \leq \dbar$, and $\zeta_2>0$,
\begin{align*}
	\P_{\S_\Isc}\left(\left|\frac{1}{\# R}\sum_{i\in\left\{i\in\Isc:\bX_i\in R\right\}}\bV_{ij}^R-\mathbf{m}_{R,j}\right|\geq\zeta_2\mid \# R\geq k/2\right)\leq 2 \exp\left\{-k\zeta_2^2\right\}.
\end{align*}
By the union bound and \eqref{Q2_eq2}, for all $\zeta_2\geq0$,
\begin{align*}
	&\P_{\S_\Isc}\left( \bigcup_{R\in \mathcal{R}_{\mathcal{D},\omega,\epsilon}}\left\{\left\|\bQ_{R,2}\right\|_2\geq\dbar\zeta_2^2\;\mbox{and}\; \# R\geq k/2\right\}\right)\leq 2\dbar \# \mathcal{R}_{\mathcal{D},\omega,\epsilon} \exp\left\{-k\zeta_2^2\right\}.
\end{align*}
Repeating the similar procedure as \eqref{eq_G.23}, we have
\begin{align*}
\P_{\S_\Isc}\left(\sup_{\bx\in[0,1]^d,\xi\in\Xi}\left\|\bQ_{\Rbar,2}\right\|_2\geq \dbar\zeta_2^2\cap\Asc\right)\leq 2\dbar \# \mathcal{R}_{\mathcal{D},\omega,\epsilon} \exp\left\{-k\zeta_2^2\right\}.
\end{align*}
Therefore,
\begin{align*}
	&\P_{\S_\Isc}\left(\sup_{\bx\in[0,1]^d,\xi\in\Xi}\left\|\bQ_{\Rbar,2}\right\|_2\geq \dbar\zeta_2^2\right)\leq\P_{\S_\Isc}\left(\sup_{\bx\in[0,1]^d,\xi\in\Xi}\left\|\bQ_{\Rbar,2}\right\|_2\geq \dbar\zeta_2^2\cap\Asc\right)+\P_{\S_\Isc}(\Asc^c)\\
	&\quad\leq 2\dbar \# \mathcal{R}_{\mathcal{D},\omega,\epsilon} \exp\left\{-k\zeta_2^2\right\}+\P_{\S_\Isc}(\Asc^c).
\end{align*}
Let $\zeta_2=\sqrt{\frac{2\log(\# \mathcal{R}_{\mathcal{D},\omega,\epsilon})}{k}}$.
By \eqref{eigen_eq3}, provided that $n\geq\max\{n_0, n_1,n_2\}$, we have
\begin{align}
&\P_{\S_\Isc}\left(\sup_{\bx\in[0,1]^d,\xi\in\Xi}\left\|\bQ_{\Rbar,2}\right\|_2\geq\frac{2\dbar\log(\# \mathcal{R}_{\mathcal{D},\omega,\epsilon})}{k}\right)\leq 2\dbar/\# \mathcal{R}_{\mathcal{D},\omega,\epsilon}+1/\sqrt n \leq2/\sqrt n,\label{Q2_eq4}
\end{align}
since $\# \mathcal{R}_{\mathcal{D},\omega,\epsilon}\geq2\dbar\sqrt n$ whenever $n\geq n_2$. Combining \eqref{Q2_eq3} and \eqref{Q2_eq4} with \eqref{Q2_eq1}, provided that $n\geq\max\{n_0, n_1,n_2\}$, we have
\begin{align*}
\P_{\S_\Isc}&\biggl(\sup_{\bx\in[0,1]^d,\xi\in\Xi,\|\ba\|_2=1}\left|\ba^\top\bQ_2\ba\right|\geq C_1\left(\sqrt{\frac{2\dbar}{k}}+\frac{2\dbar}{k}\right)+\sqrt{\frac{4\log(\# \mathcal{R}_{\mathcal{D},\omega,\epsilon})}{C_3 k}}\\
&\qquad+\frac{2\log(\# \mathcal{R}_{\mathcal{D},\omega,\epsilon})}{k}\biggl)\leq 4/\sqrt n.
\end{align*}	
Therefore, condition on the event $\Asc$, as $N\to \infty$,\begin{align*}
\sup_{\bx\in[0,1]^d,\xi\in\Xi,\|\ba\|_2=1}\left|\ba^\top\bQ_2\ba\right|=O\left(\sqrt{\frac{1}{k}}+\frac{1}{k}+\sqrt{\frac{\log(\# \mathcal{R}_{\mathcal{D},\omega,\epsilon})}{k}}+\frac{\log(\# \mathcal{R}_{\mathcal{D},\omega,\epsilon})}{k} \right)=o(1),
\end{align*}
with probability at least $1-4/\sqrt n$.

\textbf{Step 4.} We demonstrate that there exists some constant $\Lambda_0>0$ such that
\begin{align}
\inf_{\bx\in[0,1]^d,\xi\in\Xi,\|\ba\|_2=1}\ba^\top\bQ_4\ba\geq 2\Lambda_0.\label{rate_Q4}
\end{align}
Let $L_j(\bx,\xi)=[a_j,b_j]\subseteq[0,1]$ be the interval of the $j$-axis of the leaf $L(\bx,\xi)$ for each $j\leq d$. Define $V_j^L=(\bX_j-a_j)/\mathrm{diam}_j(L(\bx,\xi))$ and $c_{L,j}=(a_j-\bx_j)/\mathrm{diam}_j(L(\bx,\xi))$ for any $j\leq d$. The $\dbar$-dimensional vector $\bU^L$ can be represented as 
$\bU^L=(V_{1}^L+c_{L,1},\dots,V_{d}^L+c_{L,d},(V_{1}^L+c_{L,1})^2,(V_{1}^L+c_{L,1})(V_{2}^L+c_{L,2}),\dots,(V_{d}^L+c_{L,d})^2,\dots,(V_{d}^L+c_{L,d})^q)^\top$. Then, there exists some $\dbar\times \dbar$ lower triangular matrix $\bP_L$ with $1$ on main diagonal such that 
\begin{align}\label{def:Vtil}
\bU^L=\bP_L \bV^L+\bC_L,
\end{align}
where $\bV^L:=(V_{L,1},\dots,V_{L,d},V_{L,1}^2,V_{L,1}V_{L,2},\dots,V_{L,d}^2,\dots,V_{L,d}^q)^\top$ and $\bC_L:=(c_{L,1},\dots,c_{L,d},c_{L,1}^2,$ $c_{L,1}c_{L,2},\dots,c_{L,d}^2,\dots,c_{L,d}^q)^\top$. Here, $\P_L$ and $\bC_L$ are both deterministic given $L(\bx,\xi)$. Plugging $\bU^L=\bP_L \bV^L+\bC_L$ into $\bQ_4$, we have
\begin{align*}
\bQ_4=\Var\left(\bP_L\bV^L \mid \bX \in L(\bx, \xi)\right)=\bP_L\Var\left(\bV^L \mid \bX \in L(\bx, \xi)\right)\bP_L^\top.
\end{align*}
By the sub-multiplicative property of matrix norm, we have	
\begin{align*}
\|\bQ_4^{-1}\|_2\leq\|\bP_L^{-1}\|_2^2/\Lambda_{\min}(\Var\left(\bV^L \mid \bX \in L(\bx, \xi)\right))
\end{align*}
Since $\bP_L$ is a lower triangular matrix with $1$ on main diagonal, we know that $\bP_L^{-1}$ is an upper triangular matrix with $1$ on main diagonal, and it follows that $\left\|\bP_L^{-1}\right\|_2=1$. Therefore, 
\begin{align*}
\inf_{\bx\in[0,1]^d,\xi\in\Xi,\|\ba\|_2=1}\ba^\top\bQ_4\ba
&\qquad \geq \inf_{\bx\in[0,1]^d,\xi\in\Xi,\|\ba\|=1}\Var\left(\ba^\top\bV^L \mid \bX \in L(\bx, \xi)\right).
\end{align*}
Since the coordinates of $\bX$ are i.i.d. uniformly distributed, we know that $(V_j^L)_{j=1}^d$ are also i.i.d. uniformly distributed given $\bX \in L(\bx, \xi)$. Let $(\Vtil_{j})_{j=1}^d$ be a sequence of i.i.d. uniform random variables with support $[0,1]$, and denote $\bVtil:=(\Vtil_{1},\dots,\Vtil_{d},\Vtil_{1}^2,\Vtil_{1}\Vtil_{2},\dots,$ $\Vtil_{d}^2,\dots,\Vtil_{d}^q)^\top$. Then, for any $\bx\in[0,1]^d$ and $\xi\in\Xi$, we have
\begin{align*}
&\inf_{\|\ba\|=1}\Var\left(\ba^\top\bV^L \mid \bX \in L(\bx, \xi)\right)=\inf_{\|\ba\|_2=1}\Var\left(\ba^\top\bVtil \right).
\end{align*}
Let $\Lambda_0=\inf_{\|\ba\|_2=1}\Var\left(\ba^\top\bVtil \right)/2$. Note that the quantity $\Lambda_0$ is deterministic given the dimension $\dbar$ and hence is independent of the sample size $N$. Suppose that $\Var\left(\ba^\top\bVtil \right)=0$ with some $\ba\neq\bzero$. Then, we have $\P(\ba^\top\bVtil=a_0)=1$ with some constant $a_0\in\R$. However, note that $\ba^\top\bVtil-a_0$ is a $q$-th polynomial function of $(\Vtil_{1},\Vtil_{2},\dots,\Vtil_{d})$. As shown in Section 2.6.5 of \cite{federer2014geometric}, $\P(\ba^\top\bVtil=a_0)=1$ occurs only if $a_0=0$ and $\ba=\bzero$; this contradicts with $\ba\neq\bzero$. Therefore, we conclude that $\Lambda_0>0$ and $\inf_{\bx\in[0,1]^d,\xi\in\Xi,\|\ba\|_2=1}\ba^\top\bQ_4\ba\geq 2\Lambda_0$.

Combining the results of Steps 1-4 and note that \eqref{eigen_eq3} holds, we conclude that 
\begin{align*}
\lim_{N\to \infty}\P_{\S_\Isc}\left(\Bsc\right)=1.
\end{align*}
Lastly, condition on the event $\Bsc$, \eqref{event B}. Then, $\bS_L-\bd_L \bd_L^{\top}$ and $\bS_L$ are both positive-definite. In addition, we also have
\begin{align*}
&\sup_{\bx\in[0,1]^d,\xi\in\Xi}\bd_L^{\top}(\bS_L-\bd_L \bd_L^{\top})^{-1}\bd_L\leq \frac{1}{\Lambda_0} 
\sup_{\bx\in[0,1]^d,\xi\in\Xi}\left\|\sum_{i\in\Isc}\omega_i(\bx,\xi)\bU_i^L \right\|_2^2\\
&\qquad\overset{(i)}{\leq}\frac{1}{\Lambda_0}\sup_{\bx\in[0,1]^d,\xi\in\Xi}\left(\sum_{i\in\Isc}\omega_i(\bx,\xi)\left\|\bU_i^L \right\|_2\right)^2\overset{(ii)}{\leq}\frac{\dbar}{\Lambda_0},
\end{align*}
where (i) holds by the triangle inequality; (ii) holds by \eqref{bound_U} and $\sum_{i\in\Isc}\omega_i(\bx,\xi)=1$. 
\end{proof}

\begin{proof}[Proof of Lemma \ref{lem:p.s.d}] 
Condition on the event $\Bsc$, \eqref{event B}. Then, the matrix
$\bS_L-\bd_L \bd_L^{\top}$ is positive-definite, which implies that $\bS_L$ is also positive-definite. 
Recall that $\bD_L:=\mathrm{diag}$ $(\mathrm{diam}_1(L(\bx,\xi)),\;\dots,\mathrm{diam}_d(L(\bx,\xi)),\mathrm{diam}_1^2(L(\bx,\xi)),\mathrm{diam}_1(L(\bx,\xi))\mathrm{diam}_2(L(\bx,\xi)),\dots,\mathrm{diam}_1^2$ $(L(\bx,\xi)),\dots,\mathrm{diam}_d^q(L(\bx,\xi)))$. On the event $\Csc$, \eqref{event C}, the diagonal matrix $\bD_L$ is invertible. By \eqref{def:U_L}, we have $\bS=\bD_L\bS_L\bD_L$ and $\bS-\bd \bd^{\top}=\bD_L(\bS_L-\bd_L \bd_L^{\top})\bD_L$. 
Hence, on the event $\Bsc \cap \Csc$, we have $\bS$ and $\bS-\bd \bd^{\top}$ are both positive-definite. 
In addition, note that
\begin{align*}
\sum_{i\in\Isc}\omega_i(\bx,\xi)\bDelta_i\bDelta_i^{\top}=
\begin{pmatrix}
1 & \bd^{\top}\\
\bd & \bS
\end{pmatrix}.
\end{align*}
For any $(a,\bb)\in\R^{\dbar+1}\setminus\{\bzero\}$, we have 
\begin{align*}
	\begin{pmatrix}
		a & \bb^{\top}
	\end{pmatrix}
	\begin{pmatrix}
		1 & \bd^{\top}\\
		\bd & \bS
	\end{pmatrix}
	\begin{pmatrix}
		a\\
		\bb
	\end{pmatrix}&=(a+\bd^\top\bb)^2+\bb^\top(\bS-\bd \bd^{\top})\bb>0,
\end{align*}
since $(a+\bd^\top\bb)^2=\bb^\top(\bS-\bd \bd^{\top})\bb=0$ only when $\bb=\bzero$ and $a=-d^\top\bb=0$. Hence, $\sum_{i\in\Isc}\omega_i(\bx,\xi)\bDelta_i\bDelta_i^{\top}$ is positive-definite on the event $\Bsc \cap \Csc$. Recall that the lower triangular matrix $\bT$ is invertible. By \eqref{def:bT}, we have $\sum_{i\in\Isc}\omega_i(\bx,\xi)\bG(\bX_i)\bG(\bX_i)^{\top}= \bT^{-1}\sum_{i\in\Isc}$ $\omega_i(\bx,\xi)\bDelta_i\bDelta_i^{\top}(\bT^{-1})^{\top}$. Hence, $\sum_{i\in\Isc}\omega_i(\bx,\xi)\bG(\bX_i)\bG(\bX_i)^{\top}$ is also positive-definite on the event $\Bsc \cap \Csc$. 

In the following, we further show that $\P_{\S_\Isc}\left(\Csc\right)=1$. Let $\bX_{ij}$ be the $j$-th coordinate of the vector $\bX_{i}$ and $c\in[0,1]$ be some constant. Then, we have
\begin{align*}
&\P_{\S_\Isc}\left(\Csc^c\right)=\P_{\S_\Isc}\left(\exists\;j\leq d,\;\bx\in[0,1]^d,\;\xi\in\Xi,\;\text{s.t.}\;\mathrm{diam}_j(L(\bx,\xi))=0\right)\\
&\quad=\P_{\S_\Isc}\left(\exists\;j\leq d,\;\bx\in[0,1]^d,\;\xi\in\Xi,\;\text{s.t.}\;\bX_{ij}=\bX_{i'j}\;\forall i,i'\in\{i\in\mathcal I:\bX_i\in L(\bx,\xi)\}\right)\\
&\quad\overset{(i)}{\leq}\P_{\S_\Isc}\left(\exists\;j\leq d, i,i'\in\{i\in\mathcal I:\bX_i\in L(\bx,\xi)\}\;\text{s.t.}\;\;i\neq i'\;\text{and}\;\bX_{ij}=\bX_{i'j}\right)\\
&\quad\overset{(ii)}{\leq}\sum_{1\leq j\leq d, i,i'\in\Isc, i\neq i'} \P_{\S_\Isc}\left(\bX_{ij}=\bX_{i'j}\right)\overset{(iii)}{=}0,
\end{align*}
where (i) holds since the minimum leaf size $\# \left \{ l:\bX_l\in L(\bx,\xi) \right \}\geq k\geq 2$; (ii) holds by the union bound; (iii) holds since $\P_{\S_\Isc}\left(\bX_{ij}=\bX_{i'j}\right)=0$ as $\bX_{ij}$ and $\bX_{i'j}$ are independent uniform random variables for any $i\neq i'$ and $j\leq d$.
Therefore, we conclude that $\P_{\S_\Isc}\left(\Csc\right)=1$ holds.
\end{proof}

\begin{proof}[Proof of Lemma \ref{lem:pihat}]
By Theorem \ref{thm:local_consistency_uniform}, as $N\to \infty$, we have $\sup_{\bx\in[0,1]^{d_1}}|\pihat^{-k}(\bx)-\pi^*(\bx)|=o_p(1)$. Then, we have
\begin{align*}
	\lim_{N\to\infty}\P_{\S_n}\left(\sup_{\bx\in[0,1]^{d_1}}|\pihat^{-k}(\bx)-\pi^*(\bx)|\leq c_0/2\right)=1,
\end{align*}
which implies
\begin{align*}
	\lim_{N\to\infty}\P_{\S_n}\left[\P_{\bX}\left(\pi^*(\bS_1)- c_0/2<\pihat^{-k}(\bX)\leq \pi^*(\bX)+ c_0/2\right)=1\right]=1,
\end{align*}
By the overlap condition $\P_{\bX}(c_0<\pi^*(\bX)<1-c_0)=1$, we have
\begin{align*}
	\lim_{N\to\infty}\P_{\S_n}\left(\P_{\bX}(c_0/2<\pihat^{-k}(\bX)\leq 1-c_0/2)=1\right)=1.
\end{align*}
Let $c_1=c_0/2$. Then, we have 
\begin{align*}
	\lim_{N\to\infty}\P_{\S_n}\left(\P_{\bX}(c_1<\pihat^{-k}(\bX)\leq 1-c_1)=1\right)=1.
\end{align*}
\end{proof}

\end{document}